\documentclass{article} 
\usepackage{iclr2027_conference,times}

\usepackage{amsmath,amsfonts,bm}

\def\eqref#1{equation~\ref{#1}}

\def\1{\bm{1}}

\DeclareMathAlphabet{\mathsfit}{\encodingdefault}{\sfdefault}{m}{sl}
\SetMathAlphabet{\mathsfit}{bold}{\encodingdefault}{\sfdefault}{bx}{n}

\newcommand{\E}{\mathbb{E}}

\usepackage{url,longtable}

\usepackage{natbib}
\usepackage{caption}

\usepackage{algorithm}
\usepackage{algorithmic}
\usepackage{amsmath,amssymb,amsthm,mathtools,booktabs,array,xcolor,colortbl}
\usepackage{times,amsmath,graphicx}

\definecolor{internalrefcolor}{RGB}{160,10,20}
\definecolor{citationrefcolor}{RGB}{0,105,92}
\usepackage{hyperref}
\hypersetup{
  colorlinks=true,
  linkcolor=internalrefcolor,
  citecolor=citationrefcolor,
  urlcolor=citationrefcolor
}

\newtheorem{theorem}{Theorem}
\newtheorem{proposition}{Proposition}
\newtheorem{lemma}{Lemma}
\newtheorem{corollary}{Corollary}

\theoremstyle{definition}

\newcommand{\cvar}{\operatorname{CVaR}}
\newcommand{\var}{\operatorname{Var}}
\newcommand{\Prob}{\mathbb P}
\newcommand{\calZ}{\mathcal Z}
\newcommand{\calB}{\mathcal B}
\newcommand{\calG}{\mathcal G}
\newcommand{\PiC}{\Pi_{\!C}}
\newcommand{\ind}{\mathbf 1}
\newcommand{\tis}{\textnormal{\textsc{TIS}}}
\newcommand{\rev}[1]{{\color{black}#1}}

\definecolor{newrevisioncolor}{RGB}{0,110,60}
\newcommand{\newrev}[1]{#1}

\definecolor{tableheader}{RGB}{218,233,248}
\definecolor{tableprimary}{RGB}{232,243,252}
\definecolor{tablereference}{RGB}{244,248,253}
\newcommand{\best}[1]{\textbf{#1}}

\newif\ifshowrevisions
\showrevisionsfalse
\newcommand{\revcolor}[1]{\ifshowrevisions\color{#1}\fi}
\newcommand{\purplerev}[1]{{\revcolor{purple}#1}}
\newcommand{\redrev}[1]{{\revcolor{red}#1}}
\newcommand{\tableheadrow}{\rowcolor{tableheader}}
\newcommand{\primaryrow}{\rowcolor{tableprimary}}
\newcommand{\referencerow}{\rowcolor{tablereference}}

\newsavebox{\keyresultbox}

\title{Tail-Influence Sampling for CVaR Policy Evaluation}

\author{Pauline Bourigault \\
Imperial College London \thanks{Code available at \href{https://github.com/paulinebourigault/TIS}{https://github.com/paulinebourigault/TIS}}
\And Xiaotong Ji \\
Huawei Noah's Ark Lab
\And Matthieu Zimmer \\
Huawei Noah's Ark Lab\\
\And 
Rasul Tutunov \\
Huawei Noah's Ark Lab
\And Haitham Bou-Ammar 
\\
UCL Centre for AI
}

\iclrfinalcopy 
\begin{document}

\maketitle
\addtocontents{toc}{\protect\setcounter{tocdepth}{-1}}

\begin{abstract}
Policies with similar mean returns can differ sharply in rare failures, yet estimating lower-tail conditional value-at-risk (CVaR) accurately can require many costly rollouts. When different conditional components of a stochastic workflow can be queried separately, we ask how to allocate a fixed evaluation budget to estimate a fixed policy’s CVaR most accurately. We derive a tail influence for each queryable conditional law that aggregates how its uncertainty affects CVaR across every Bellman reuse. Its variance yields the fixed-design efficiency bound and the oracle Neyman allocation. Tail-Influence Sampling (TIS) estimates these influence scales from a pilot model and reallocates fresh queries toward kernels that matter most for the tail; a visitation-anchored variant protects against pilot underallocation. Under fixed dimension and a positive quantile margin, TIS attains oracle asymptotic variance and first-order MSE including pilot cost, while the anchored variant is within a factor two of the oracle. We also characterize an exact-grid regime in which tail- and mean-optimal allocations coincide. On CliffWalking, TIS reduces MSE by 41\% versus learned occupancy and 76\% versus complete rollouts at the same charged transition budget. In frozen language-model review workflows, anchored TIS beats an equally regularized mean-influence blend in 23 of 24 MMLU-Pro settings and reaches $2.4 - 3.4\times$ lower MSE than rollouts on six-call FinQA reviews.\end{abstract}

\section{Introduction}
Average performance can conceal the failures that matter most. A language-model workflow may answer correctly on most runs yet occasionally produce a severe numerical error. Similarly, a robot may complete a task reliably yet rarely enter a state that causes damage. Lower-tail conditional value-at-risk (CVaR) measures average return in a specified worst fraction of runs. Estimating this tail average accurately by Monte Carlo can require many complete runs. \purplerev{Such estimates decide which policy or workflow is safe to deploy. For language-model workflows each run costs several model calls, so the evaluation budget, not the model, often limits how reliably rare severe failures can be measured.} In resettable simulators and some generative workflows, however, we can directly sample what happens next from a chosen state and action. This gives the evaluator control over where to collect information. Given a limited evaluation budget, which state--action pairs should we query to estimate a fixed policy's CVaR most accurately?

Fixing the policy does not tell us the transition and reward probabilities, which must be learned from queries. Sampling every state--action pair equally ignores their different roles in producing failure, while sampling according to visitation concentrates the budget on frequently encountered pairs. Neither strategy necessarily identifies where uncertainty about the tail originates. A frequently visited state may have predictable outcomes, while uncertainty at a rarer state can change the estimated severity of failures. \purplerev{This is the central difficulty in evaluating rare failures of language-model agents: runs seldom reach the states that produce the worst outcomes, such as a step at which a tool returned a wrong value, so complete rollouts spend almost all of their budget elsewhere.} Effective allocation must account for both outcome variability and its effect on the final tail estimate.

Our answer is a computable notion of tail influence: how uncertainty in a conditional law, or kernel, affects uncertainty in the final CVaR estimate. A kernel receives a high score when its possible outcomes vary substantially and those differences meaningfully change the estimated severity of the worst runs. \redrev{Its CVaR effect depends on the continuation model and tail cutoff, and the same kernel can be reused at multiple stages.} For example, the same review prompt may recur in a language-model workflow, or a control policy may revisit a state. One observation of this shared behaviour therefore provides information about several parts of the return calculation simultaneously. We combine these effects before measuring their variability to obtain one allocation score per kernel.

\redrev{Tail-Influence Sampling (\tis{}) uses a uniform pilot to fit the conditional laws, computes influence scales in the fitted model, and allocates fresh queries in proportion to those scales, with a uniform exploration share.}
A short pilot can nevertheless miss a rare consequential outcome, causing \tis{} to underestimate a kernel's importance and give it too few additional queries. Such an error can dominate the final CVaR estimate because the neglected kernel may be precisely where the worst outcomes originate. Anchored \tis{} averages the influence-based allocation with visitation-based shares, mitigating this failure when visitation remains informative. \redrev{This follows the established defensive-mixture principle \citep{hesterberg1995weighted,owen2000safe}: combine a targeted design with a broader reference design.} \redrev{Under the same conditions, the anchor's asymptotic MSE is at most twice the oracle's.}

\purplerev{A tail-specific score is not always needed. \redrev{On an exact categorical grid, if the probability of a submaximal return is below the tail level, the worst runs contain every submaximal return plus some maximal returns;} CVaR then moves exactly with the mean return, and tail influence becomes proportional to mean influence. The divergence between the two scores indicates where tail-specific allocation can pay off: estimated from the pilot, it tells an evaluator whether to allocate by tail influence with the defensive anchor or whether mean-based allocation suffices. The payoff is practical: at matched accuracy, \tis{} needed $2$--$4$ times fewer queries than complete rollouts on our tabular benchmarks, and the anchor $2$--$10$ times fewer than uniform sampling on held-out FinQA.}

Prior distributional-RL inference and efficiency results analyze estimation under specified sampling laws \citep{zhang2025inference,cheng2026quantile}, adaptive stratification learns allocations for fixed within-stratum quantities \citep{etore2010adaptive,carpentier2015adaptive}, and trajectory designs such as ReVar target mean policy evaluation \citep{mukherjee2022revar}. Here the allocation score itself depends on an unknown Bellman continuation model and CVaR cutoff; we derive it and show that learning it together with the allocation recovers oracle first-order MSE, including pilot cost.

In short, our contributions can be stated as follows: 
\redrev{\emph{(i) A CVaR-specific allocation signal.} We derive each shared kernel's influence through its Bellman uses. Its variance gives a fixed-design efficiency bound and the scales for classical Neyman allocation \citep{neyman1934two}.} \redrev{\emph{(ii) Learning the score and allocation.}} Under fixed dimension, a positive quantile margin, and suitable pilot and exploration schedules, \tis{} \redrev{learns the continuation model, cutoff, and influence scales while attaining} oracle asymptotic variance and first-order mean-squared error (MSE), including pilot cost. \redrev{We also characterize the anchor's asymptotic MSE under the same conditions,} \purplerev{within a factor two of the oracle. \redrev{\emph{(iii) When tail specificity matters.} Models with identical visitation, reward moments, and return laws can need different allocations; on exact grids in a rare-failure regime, tail- and mean-optimal allocations coincide, and their divergence serves as a diagnostic.}} \emph{(\purplerev{iv}) Experiments.} CliffWalking \purplerev{and an 18-case inventory disruption family} show gains over learned occupancy, learned mean influence, and rollouts at matched query budgets. On language-model workflows with frozen laws, gains over occupancy and rollouts depend on the generator, budget, and workflow length. \purplerev{Blending controls show that the tail score, not added regularization, drives the anchor's gains; the pilot-estimated divergence selects between tail- and mean-based allocation; and in longer review loops the anchor beats complete rollouts at matched cost.}

\section{Problem and Estimator}
\label{sec:problem}

We evaluate a fixed policy in a finite-horizon Markov decision process (MDP) with stationary dynamics. The state and action spaces $\mathcal S,\mathcal A$ are finite, $H$ is the horizon, and $s_0$ is the initial state. With $h=H-t$ steps remaining, the policy selects $A_t\sim\pi_h(\cdot\mid S_t)$ and the environment draws $(R_t,S_{t+1})\sim P_{S_t,A_t}$, with $R_t\in[0,1]$. One transition takes $(h,s)$ to $(h-1,s^\prime)$; $h=0$ ends the run. The policy may depend on $h$, while $P_{s,a}$ does not. Starting at $S_0=s_0$, total return is $G_H(s_0)=\sum_{t=0}^{H-1}R_t\in[0,H]$.
For a chosen level $\alpha\in(0,1)$, our goal is to estimate $\cvar_\alpha(G_H(s_0))$: the average return in the worst $\alpha$-fraction of runs, taking only the required fraction of probability mass at the cutoff. Smaller $\alpha$ emphasizes rarer outcomes. Appendix~\ref{sec:supp-projection} gives the formal definition.

\textbf{What can be sampled?} A query group $g=(s,a)\in\calG$ is one independently queryable conditional law, or kernel. One query returns $W_g=(R,S')\sim P_g$; reward and next state may be dependent. Queries are i.i.d. within a group and independent across groups. For example, an evaluator can reconstruct a review prompt from a specified answer and confidence, then sample its response without reaching that state by rollout. We retain a fixed set of states and $G=|\calG|$ groups, closed under all possible continuations. Allocating $n_g$ queries to each group costs $N=\sum_g n_g$; a rollout costs one query per transition.
The same fitted law $\widehat P_g$ is used at every remaining horizon where that group occurs. Thus $g=(s,a)$ identifies what we sample, while $(h,s)$ identifies where we evaluate its consequences. The same next state can have different continuation returns with one or three steps left; these differences matter for the allocation score.
\begin{figure}[t]
\centering
\includegraphics[width=0.7\linewidth]{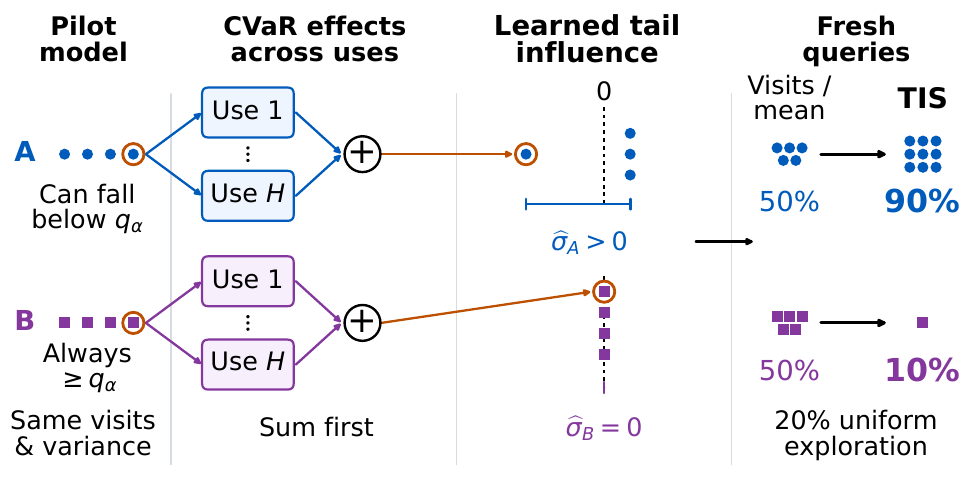}
\caption{\redrev{\textbf{Equal visitation and reward variability can conceal different information about CVaR.} The same kernel can be reused across many stages (Use 1, $\ldots$, Use $H$). A and B are visited equally and have identical reward means and variances, but only A's returns can fall below the CVaR cutoff $q_\alpha$. \tis{} sums each pilot draw's first-order CVaR effects across all uses, then measures their spread $\widehat\sigma_g$ across draws; orange rings track one draw. With 20\% uniform exploration, the illustrated scores give nominal fresh-query shares of 90\%/10\% for A/B, versus 50\%/50\% for visitation or mean influence (before minimum counts and rounding).}}
\label{fig:tis-setup}
\end{figure}

\textbf{How do samples give a CVaR estimate?} We first estimate the return distribution from each state with $h$ steps left. Following categorical distributional RL \citep{bellemare2017distributional}, we store probabilities on a fixed return grid $\calZ=\{0=z_0<\cdots<z_{K-1}=H\}$ with maximum gap $\Delta$. The vector $p^*_{h,s}$ contains these probabilities: $p^*_{h,s,k}$ is the mass at return $z_k$. The indices mean remaining steps ($h$), current state ($s$), and return value ($k$).

A Bellman update combines the immediate reward with the distribution of future returns. The matrix $Q(R)$ adds $R$ to each grid value and projects the result onto the grid, splitting mass between neighboring points and clipping outside the endpoints \citep{bellemare2017distributional,rowland2018analysis}. Averaging over actions and next transitions gives
\begin{equation}
 p^*_{h,s}=\sum_a\pi_h(a\mid s)\,
       \E_{(R,S')\sim P_{s,a}}\big[Q(R)p^*_{h-1,S'}\big],
 \qquad p^*_{0,s}=e_0,
 \label{eq:cat-dp}
\end{equation}
where $e_0$ puts all mass at zero. Replace each expectation by a group sample average and evaluate successively for $h=1,\ldots,H$. CVaR of the estimated root probabilities is $\widehat C_N$; $C_{\alpha,K}$ uses the population probabilities (Appendix~\ref{sec:supp-notation}).
Allocation controls sampling error around $C_{\alpha,K}$. The grid introduces a separate approximation error of at most $H\Delta$ relative to true-return CVaR (Appendix~\ref{sec:supp-representation}). All limits keep the horizon, state and action sets, grid, tail level, and retained groups fixed. Stars denote population quantities, hats estimates, and $(0)$ pilot estimates.

\section{From One Observation to an Optimal Allocation}
\label{sec:shared_kernel_tail_influence}

More queries to a group help only if uncertainty in that group affects the final CVaR estimate. We derive that effect in three steps: write CVaR using an expected shortfall, trace a transition error to that shortfall, then allocate samples according to the variability of the resulting effects.

\textbf{Read CVaR as a shortfall.}
For a grid threshold $q\in\calZ$, the shortfall $(q-Z)_+$ is zero above $q$ and measures the distance below it. Let $U_h^*(s,q)=\E[(q-Z_h(s))_+]$, where $Z_h(s)$ has the grid probabilities $p^*_{h,s}$. At the root $\alpha$-quantile $q_\alpha$, the standard shortfall identity gives
\begin{equation}
 C_{\alpha,K}=q_\alpha-\frac{U_H^*(s_0,q_\alpha)}{\alpha}.
 \label{eq:main-cvar-shortfall}
\end{equation}
For probabilities $(.04,.08,.88)$ at returns $(0,.5,1)$, the worst $10\%$ contains all zero returns and enough $.5$ returns to average $.3$. Here $q_\alpha=.5$ and the mean shortfall is $.04\times.5=.02$; the formula gives $.5-.02/.1=.3$.
We assume a positive quantile margin: the tail cutoff lies strictly inside the probability mass at one grid value. In the example, $.04<.1<.12$, so small probability errors leave $q_\alpha=.5$ unchanged. With that threshold unchanged, CVaR error is shortfall error scaled by $-1/\alpha$. This lets us focus on one expected shortfall. Appendix~\ref{sec:supp-projection} defines the margin and explains why a zero margin can invalidate the Gaussian limit.

\textbf{Trace a transition error to CVaR.}
After observing reward $R$, the remaining shortfall threshold is $q-R$. The sampled target $T_h(W;U)$ therefore uses $U_{h-1}(S',q-R)$, interpolating between grid thresholds and taking zero at nonpositive thresholds; at $h=1$ it is $(q-R)_+$. Categorical projection preserves these shortfalls at grid thresholds (Lemma~\ref{lem:projection}), so this is the same return calculation in more convenient coordinates.

Write all $(h,s,q)$ shortfalls as the vector $U^*$. To learn how errors in them affect the initial state's shortfall, compute sensitivity weights $r$, also called adjoint weights. The two Bellman passes can be written as
\begin{equation}
 U^*=\mathbf b+MU^*,\qquad (I-M)^\top r=e_{x_0},
 \quad x_0=(H,s_0,q_\alpha).
 \label{eq:affine-system}
\end{equation}
Here $\mathbf b$ contains terminal contributions, $M$ carries continuation weights, and $e_{x_0}$ selects the root shortfall. The first equation computes shortfalls; the second assigns a weight to each coordinate according to its effect on the root. Both are computed by passes through the layers, without a dense matrix inverse.

For one observation $W$ from group $g$, let $\mathcal T_g(W;U^*)$ collect its updates wherever that group is used, including policy weights. Subtracting the expected updates gives the one-draw error $\Xi_g(W)=\mathcal T_g(W;U^*)-\E\mathcal T_g(W;U^*)$. Weighting by $r$ translates this error to the root, and $-1/\alpha$ converts it to CVaR error:
\begin{equation}
 \phi_g(W)=-\frac1\alpha r^\top\Xi_g(W),\qquad
 \sigma_g^2=\E \big[\phi_g(W)^2 \big].
 \label{eq:block-sigma}
\end{equation}
Thus $\phi_g$ is one draw's first-order effect on CVaR, and $\sigma_g$ measures how much that effect varies across draws. A shared kernel's observation affects several stages together. We sum those effects before taking their variance, retaining the cross-stage covariance of the same draw (Equation~\ref{eq:shared-influence-explicit}).

\redrev{Figure~\ref{fig:tis-setup} illustrates this calculation; Appendix~\ref{sec:supp-projection} gives a numerical example and its covariance calculation.}

\textbf{Allocate where more samples reduce error.}
Averaging $n_g$ independent observations reduces group $g$'s leading variance contribution to $\sigma_g^2/n_g$. Contributions from independently sampled groups add. The next theorem formalizes the resulting prediction: CVaR MSE is approximately $\sum_g\sigma_g^2/n_g=V(w)/N$, where $w_g$ is the group's budget share.

\begin{theorem}[Allocation-dependent limit]
\label{thm:allocation-clt}
Under the fixed-dimensional model and positive margin above, let the counts be deterministic with $n_g/N\to w_g>0$. Then
\begin{equation}
 \sqrt N \Big(\widehat C_N-C_{\alpha,K}\Big)\Rightarrow\mathcal N \big(0,V(w) \big),
 \qquad N\E \Big[(\widehat C_N-C_{\alpha,K})^2 \Big]\to V(w),
 \quad V(w)=\sum_g\frac{\sigma_g^2}{w_g}.
 \label{eq:allocation-clt}
\end{equation}
\end{theorem}
Increasing a high-influence group's share reduces its contribution to error. The proof also controls the nonlinear Bellman remainder and wrong-quantile probability in normalized MSE (Appendix~\ref{sec:supp-fixed-design}).

\begin{theorem}[Fixed-design efficiency]
\label{thm:oracle}
For each such allocation and population law satisfying the positive quantile margin, $V(w)$ is the local semiparametric efficiency bound for $C_{\alpha,K}$ in the product model of unrestricted group laws on their fixed declared outcome spaces, relative to its differentiable-in-quadratic-mean (DQM) tangent space. The empirical categorical Bellman estimator is regular under these local submodels and attains the bound.
\end{theorem}
For a fixed budget split, this is the smallest leading variance among regular estimators using the stated conditional samples (Appendix~\ref{sec:supp-efficiency}). We can now choose the split itself. When $\sum_g\sigma_g>0$, minimizing $V(w)$ gives the classical Neyman rule
\begin{equation}
 w_g^*=\frac{\sigma_g}{\sum_j\sigma_j},\qquad
 V^*=\left(\sum_g\sigma_g\right)^2.
 \label{eq:oracle-value}
\end{equation}
A group with twice the influence scale receives twice the oracle share. If some scales vanish, the optimum over positive shares is approached by letting their exploration shares tend to zero. The next example shows why ordinary visitation and reward variability cannot replace this tail-specific score.
In a one-step model with equally visited kernels, one returns $0$ with probability $.1$ and $5/9$ otherwise; each other kernel returns $1/3$ or $2/3$ with equal probability. Their first two moments match, but only the first has random shortfall below the CVaR threshold $1/3$. Appendix~\ref{sec:supp-controlled-separation} gives the proof and fixed-root-law family.

\begin{proposition}[Equal visitation and reward variance can hide tail influence]
\label{prop:controlled-separation}
For every fixed $G\ge2$, there is a one-step categorical model with a uniform policy, conditional mean $1/2$ and variance $1/36$ in every kernel, and a positive margin at $\alpha=.1$, for which
\[
 V_{\rm occupancy}=V_{\rm mean}=\frac1G,
 \qquad V^*=\frac1{G^2}.
\]
These are the coefficients of $1/N$ in asymptotic CVaR MSE; $V_{\rm mean}$ uses the allocation optimal for estimating the mean return. There is also a family with these same visitation probabilities, conditional moments, and entire root return law in which the occupancy-to-oracle ratio ranges from $1$ to $G$.
\end{proposition}
\purplerev{With ten kernels, the oracle thus has one tenth of uniform's leading MSE.}

\section{Tail-Influence Sampling}
\label{sec:tis}

The oracle shares require the unknown transition laws and their effects on future returns. \tis{} learns both from a pilot. Algorithm~\ref{alg:tis} \purplerev{(Appendix~\ref{sec:supp-learning})} draws $m=m_N$ samples per group, fits a provisional model, and evaluates it by Equation~\ref{eq:cat-dp}. This gives the pilot root quantile; Equation~\ref{eq:affine-system} then gives the shortfalls $\widehat U^{(0)}$ and sensitivity weights $\widehat r^{(0)}$ needed to score each draw.

\textbf{Score each draw, then measure variability.} For pilot outcome $W_{g,i}$, compute
\begin{equation}
 \widehat d_{g,i}=-\frac1\alpha\widehat r^{(0)\top}
     \mathcal T_g(W_{g,i};\widehat U^{(0)}),\qquad
 \widehat\sigma_g=\sqrt{\frac1m\sum_{i=1}^m
     (\widehat d_{g,i}-\bar d_g)^2},
 \label{eq:pilot-scales-main}
\end{equation}
Here $\bar d_g$ is the group's average score. Each score combines all stage effects of a draw; larger estimated scales receive more queries:
\begin{equation}
 \widehat w_g=(1-\lambda_N)
       \frac{\widehat\sigma_g}{\sum_j\widehat\sigma_j}
       +\frac{\lambda_N}{G},\qquad 0<\lambda_N<1.
 \label{eq:floored-design}
\end{equation}
When all scales vanish, use uniform sampling. For scales $1$ and $3$, the shares before the floor are $1/4$ and $3/4$. The floor reserves samples for groups the pilot may have underestimated.

\purplerev{After two main samples per group, largest-remainder rounding spends the remaining budget exactly.} Only fresh main samples enter the final estimate.
As $N$ grows, a larger pilot can consume a shrinking budget fraction. Under the following schedules, \tis{} attains oracle leading MSE, including pilot cost.

\begin{theorem}[Oracle adaptation]
\label{thm:adaptive}
Under the same model and margin, suppose $\sum_g\sigma_g>0$. If
$m_N\to\infty$, $Gm_N=o(N)$, $\lambda_N\to0$,
$\sqrt N\lambda_N\to\infty$, and 
$\log(1/\lambda_N)=o(m_N)$, then
\begin{equation}
 \sqrt N \Big(\widehat C_N^{\tis}-C_{\alpha,K} \Big)\Rightarrow\mathcal N(0,V^*),
 \qquad N\E \bigg[ \Big(\widehat C_N^{\tis}-C_{\alpha,K} \Big)^2 \bigg]\to V^*.
 \label{eq:adaptive-limit}
\end{equation}
A total pilot of order $N^{2/3}$ and floor $N^{-1/4}$ suffice for fixed $G$. Individual zero-influence groups are allowed.
\end{theorem}
\redrev{The theorem controls learning the score through the unknown continuation model and quantile, as well as learning its variance and allocation. The proof handles rare underallocating pilots, the nonlinear Bellman remainder, and wrong quantile atoms (Appendix~\ref{sec:supp-adaptation}).} The limit is pointwise in the fixed model, not uniform over increasingly rare events or shrinking margins.

\textbf{\purplerev{A defensive} anchor for finite pilots.} A small pilot can miss a consequential outcome and give its group too few main samples. Local stability near the oracle (Proposition~\ref{prop:finite-pilot}) does not control that failure. \redrev{Following the established defensive-mixture principle \citep{hesterberg1995weighted,owen2000safe}, we combine tail targeting with pilot-estimated visitation. This can protect kernels that the pilot recognizes as frequently reached even when their tail influence is underestimated.} Let $\widehat o_{s,a}=\sum_h
\widehat\mu_h(s)\pi_h(a\mid s)$ be the pilot-model expected visit count. Form a floored occupancy design from these scores, using the same pilot and floor as the influence design. Anchored \tis{} averages the two:
\begin{equation}
 w^{\rm anc}=\tfrac12(w^{\rm inf}+w^{\rm occ}).
 \label{eq:anchored-score}
\end{equation}
The influence component keeps its uniform fallback. \redrev{Use these shares in Algorithm~\ref{alg:tis}.}
The mixture keeps at least half of either component's share for every group. Before rounding, its leading variance is therefore at most twice that of the better component (Proposition~\ref{prop:anchored}, Appendix~\ref{sec:supp-anchor}). Both components can still underallocate the same group. \redrev{Under the assumptions and schedules of Theorem~\ref{thm:adaptive}, the anchor's asymptotic MSE constant satisfies $V^*\le V_{\rm anc}\le2V^*$ (Corollary~\ref{cor:anchor-asymptotic}, Appendix~\ref{sec:supp-anchor}). This does not guarantee lower finite-budget MSE.} \purplerev{Plain \tis{} is the efficient limit when pilots estimate the scales reliably; the anchor pays at most a factor two in the constant for protection against pilots that miss rare outcomes.} Pilot underallocation in the initial language-model experiments motivated this anchor. We fixed its design before collecting held-out FinQA numerical-review data (Section~\ref{sec:experiments}).

\begingroup\revcolor{purple}
\textbf{When can a tail-specific score help?} Learning an allocation must repay its pilot. Let $\rho$ be the pilot's fraction of the budget, $A_v=V(v)/V^*$ the oracle's advantage over a fixed design $v$ that spends all $N$ queries, and $D(w^*\Vert w)=\sum_g(w^*_g-w_g)^2/w_g$ the error of the realized main-sample shares $w$ relative to the oracle shares $w^*$ in Equation~\ref{eq:oracle-value}. \redrev{For the leading error term, learning beats $v$ exactly when
\begin{equation}
 1+\E D(w^*\Vert w)<(1-\rho)A_v
 \label{eq:pilot-payoff-main}
\end{equation}
(Appendix~\ref{sec:pilot-payoff}).} Large allocation advantages, small pilots, and accurate shares favor learning. Because $D$ divides by $w_g$, a starved group is especially costly; the anchor guards against exactly this. A further limit is structural.

\begin{proposition}[Rare failures make CVaR a mean]
\label{prop:rare-failure}
Suppose the categorical grid contains every partial return permitted by the fixed declared outcome spaces, so the recursion is exact throughout the product model. Let $X:=G_H(s_0)$ and let $x^\star$ be the maximum return permitted by those outcome spaces. Write $\pi=\Pr(X<x^\star)$. If $\pi<\alpha$, then
\[
 \mathrm{CVaR}_\alpha(X)=\frac{\mathbb E[X]-(1-\alpha)x^\star}{\alpha},
\]
and locally every kernel's categorical-CVaR influence is $1/\alpha$ times its ordinary mean-return influence. Hence the tail- and mean-optimal Neyman allocations coincide whenever the influence scales are not all zero; if they are all zero, every allocation has zero first-order variance.
\end{proposition}

The worst $\alpha$-fraction of runs then contains every submaximal return plus enough maximal returns to fill the tail, so CVaR moves exactly with the mean. A mean score is also easier to learn, since it does not depend on a quantile. The total-variation distance $\tfrac12\sum_g|p_g-m_g|$ between normalized tail and mean influence shares (using the uniform vector for an all-zero score) is zero under the proposition and serves as a diagnostic (Appendix~\ref{sec:supp-rare-failure}).
\endgroup

\section{Experiments}
\label{sec:experiments}

We ask \purplerev{five} questions: \textbf{Q1} Can visitation and mean-return scores miss tail influence? \textbf{Q2} Does learning this signal improve matched-budget accuracy? \textbf{Q3} Why can pilots fail, and does anchoring help? \textbf{Q4} Does anchoring transfer to held-out numerical-review workflows? \purplerev{\textbf{Q5} Is the gain tail-score-specific, and does the divergence diagnostic predict where? Theory predicts three regimes: plain \tis{} should win with concentrated tail influence and reliable pilots (Q1--Q2); the defensive anchor should matter when pilots miss rare outcomes (Q3--Q4); and, on this diagnostic's exact closed grids, no tail-specific gain should appear when submaximal returns are rarer than the tail level (Q5).}

\textbf{Compared methods.} Uniform splits queries evenly. Learned occupancy/mean use pilot-estimated visit counts/mean-return influence, respectively, with \tis{}'s pilot, floor, and CVaR estimator. Complete rollouts use whole trajectories, not conditional-query methods' direct kernel access; comparisons are cost-matched, not equal-access. Oracle+floor: exact tail-influence scales, no pilot. Budgets include discarded pilots and all rollout transitions; MSE uses exact closed-grid targets (Appendix~\ref{sec:supp-experiments}).

\noindent\textbf{Q1. Controlled separation.}
Table~\ref{tab:controlled-sep} varies tail-influence concentration while preserving visitation, conditional moments, and root return law. At $t=1$, plain \tis{} is $.16$--$.22$ of uniform MSE (oracle $.12$--$.13$); at $t=0$ (uniform optimal), it is $4.3$--$5.4\times$ uniform MSE. The intermediate case likewise does not repay the pilot (Appendix~\ref{sec:pilot-payoff}).

\begin{table}[t]
\centering
\small
\caption{\textbf{TIS gains under concentrated tail influence; elsewhere its pilot does not repay.} MSE/uniform MSE ($G=10,\alpha=.1$); columns: queries/kernel incl. pilots. Uniform=known occupancy; occupancy+pilot isolates pilot cost; Oracle+floor=population scores. Bold: sample-only column minima (point estimates). Table~\ref{tab:controlled-sep-uncertainty}: absolute errors/SEs.}
\label{tab:controlled-sep}
\begin{tabular}{lrrrrrr}
\toprule
\tableheadrow
 & \multicolumn{2}{c}{$t=0$ (equal)} & \multicolumn{2}{c}{$t=\frac12$} & \multicolumn{2}{c}{$t=1$ (one active)}\\
\tableheadrow
Method & 100 & 400 & 100 & 400 & 100 & 400\\
\midrule
Uniform & \best{1.00} & 1.00 & 1.00 & \best{1.00} & 1.00 & 1.00\\
Occupancy + pilot & 1.60 & 1.31 & 1.67 & 1.39 & 1.69 & 1.31\\
Learned mean & 1.63 & 1.28 & 1.71 & 1.39 & 1.87 & 1.34\\
Complete rollout & 1.01 & \best{0.95} & \best{0.99} & 1.07 & 1.11 & 1.18\\
\primaryrow \textsc{TIS} & 5.42 & 4.30 & 3.62 & 3.15 & \best{0.22} & \best{0.16}\\
\primaryrow anchored \textsc{TIS} & 2.04 & 1.52 & 1.66 & 1.17 & 0.37 & 0.27\\
\referencerow Oracle + floor & 1.00 & 1.00 & 0.75 & 0.77 & 0.13 & 0.12\\
\bottomrule
\end{tabular}
\end{table}

\noindent\textbf{Q2. Tabular benchmarks.}
Seasonal inventory ($H=8$, 41 blocks, $\alpha=.1$) uses the prespecified pilot/floor; at $1{,}200$ queries/block, MSE/uniform is $.657$ for \tis{}, $.918$ for learned occupancy, $.866$ for learned mean influence, and $2.22$ for complete rollouts at matched transition cost (Figure~\ref{fig:main}(b), Table~\ref{tab:inventory-full}).
Slippery CliffWalking ($H=20$, $\alpha=.1$, 149 stationary kernels) tests layer reuse (Figure~\ref{fig:main}(a)). At 400 queries/kernel, \tis{} lowers MSE by $40.9\%$ vs. learned occupancy, $18.3\%$ vs. learned mean influence, $76.3\%$ vs. complete rollouts at matched transition cost, and $31.3\%$ vs. population occupancy (Table~\ref{tab:cliff-full}). FrozenLake/rainy Taxi add shared-kernel checks (Appendix~\ref{sec:supp-public-gym}). Most gains come from pooling reused kernels; covariance terms matter little numerically here (Appendix~\ref{sec:supp-cliff-full-note}).
\purplerev{To test breadth, we prespecified an 18-case family before simulation: disruption probabilities $\{.01,.04,.12\}$, disruption losses $1$--$3$ units and two fixed ordering policies; all 18 reported (Appendix~\ref{sec:supp-inventory-family}). At $1{,}200$ queries/block, both \tis{} and the anchor have resolved lower MSE vs. learned occupancy/rollouts in all 18 and vs. learned mean in 17 (median MSE/uniform: $.61$ \tis{}, $.69$ anchored, $.93$ learned occupancy, $.84$ learned mean, $2.56$ rollouts).}
At matched RMSE, \tis{}'s query ratios (learned occupancy/rollouts) are $.65$--$.84$/$.32$--$.45$ on CliffWalking and $.55$--$.72$/$.26$--$.30$ on inventory; these retrospective interpolations include pilots (Appendix~\ref{sec:supp-cost-accounting}).

\begin{figure*}[t]
 \centering
 \includegraphics[width=0.7\textwidth]{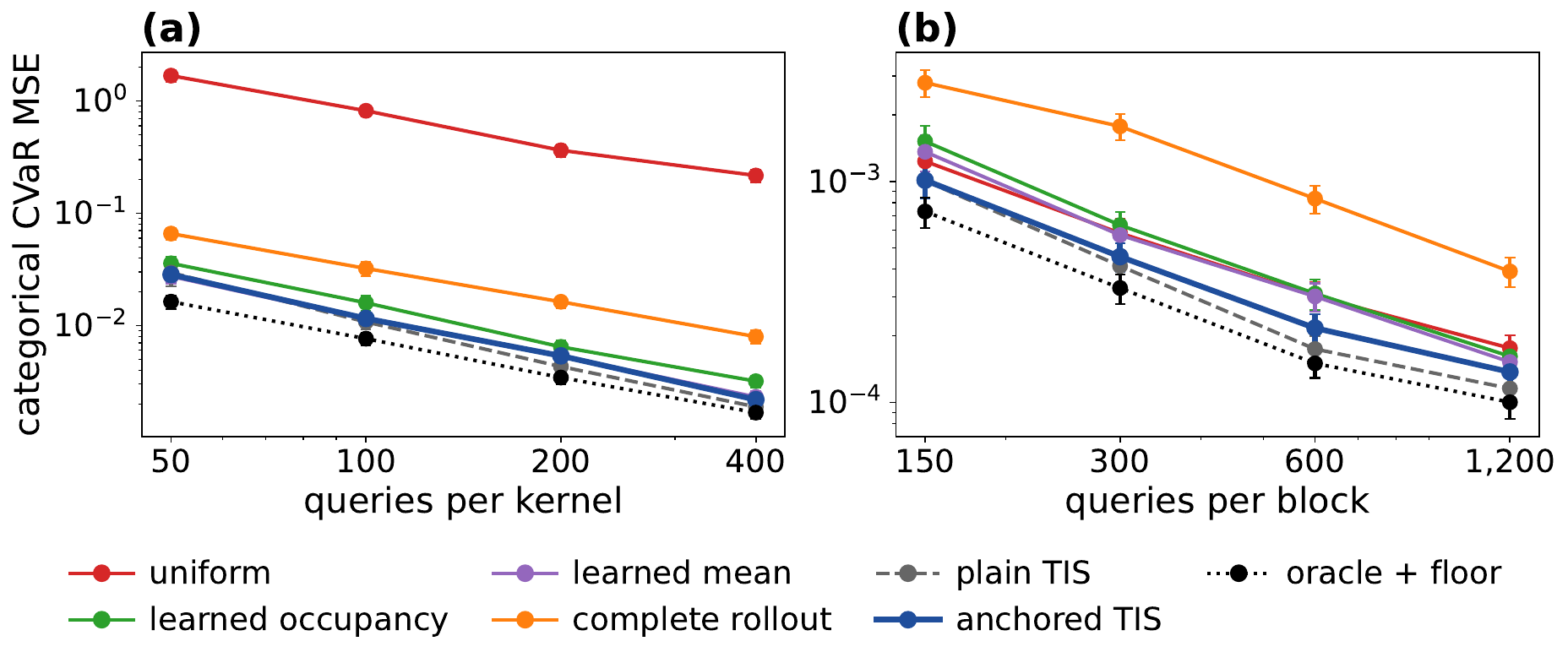}
 \caption{\textbf{At largest budgets, TIS MSE is lower than learned occupancy, learned mean, and rollouts.} CVaR MSE vs. charged queries/kernel in (a) CliffWalking and /block in (b) inventory. Conditional-query methods have direct kernel access; rollouts use trajectories at same charged transition budget. Oracle+floor=population scores; bars=$1.96$ Monte Carlo SEs.}
 \label{fig:main}
\end{figure*}

\noindent\textbf{Q3. Pilot reliability in language-model workflows.}
\label{sec:language_model_workflow_study}
\redrev{Each of 50 fixed MMLU-Pro questions \citep{wang2024mmlupro} is a separate workflow. State $s=(j,c)$ records latest answer $j\in\{1,\ldots,10\}$ and confidence $c\in\{.1,\ldots,.9\}$. Fixed policy solves first, then selects \texttt{reconsider}, \texttt{challenge}, or \texttt{verify} by confidence band. Prompts include question/current pair/prescribed action, but no history or stage index. Root plus $10\times9$ pairs yield 91 directly queryable kernels reused across stages.
Runs use $H\in\{2,4,6\}$ calls. Responses earn normalized Brier utility vs. correct option (Equation~\ref{eq:llm-brier-reward}); for six generators, we estimate each question's worst-$10\%$ CVaR against exact closed-grid targets (Appendix~\ref{sec:llm-protocol}). Queries sample frozen LLM laws; live GPU/API cost is out of scope.}

At $H=6$, 400 queries/kernel, plain \tis{} exceeds uniform MSE for Qwen3-4B \purplerev{($1.38$)} and GLM-4-32B \purplerev{($1.77$)} (Table~\ref{tab:llm-compact}). Occupancy has lower observed MSE than plain \tis{} for all six generators. The anchor mitigates both failures\purplerev{, reaching $.056$--$.093$ of uniform MSE and the lowest MSE for 5/6 generators}; occupancy and complete rollouts remain strong. \purplerev{Kernel reuse matters: fitting recurring-prompt copies separately by step matches shared \tis{} at $H=2$ but has $4$--$10\times$ its MSE at $H=4,6$ (Table~\ref{tab:supp-llm-full}).}
\redrev{The MMLU replay's 90th-percentile realized/oracle variance ratio falls from $136$ (plain \tis{}) to $5.8$ (anchor)} \purplerev{(Appendix Figure~\ref{fig:mechanism}(a)): starved groups drive failures, as Equation~\ref{eq:pilot-payoff-main} predicts.} \purplerev{Using population influences and realized counts, the variance formula predicts observed MSE without fitted constants (median log-ratios $-.001$ for MMLU and $-.006$ for FinQA Phi; retrospective check, Appendix~\ref{sec:pilot-mechanism}).}

\noindent\textbf{Q4. Held-out financial numerical review.}
\label{sec:finqa}
\redrev{FinQA has numerical questions over real financial reports \citep{chen2021finqa}. We compare ordinary review and explicit unit-and-sign audit. Each workflow selects one of eight frozen candidate solutions, then reviews twice ($H=3$) under the MMLU confidence-band policy. Root plus $8\times9$ candidate--confidence pairs yield 73 queryable kernels; review templates alter transition laws.
Terminal utility $u(s)\in[0,1]$ falls with relative numerical error against gold answer; invalid candidates earn zero. With zero root utility, rewards $R=[u(s')-u(s)+1]/2$ sum to $G_H=[H+u(S_H)]/2$, so only final-answer utility matters.}
\purplerev{A 20-question development phase fixed the screen, generators, workflows, seeds, anchored score, and floor before held-out calibration on 50 fresh screened questions from 312 scanned (Appendix~\ref{sec:finqa-protocol}).}
At 100--200 queries/kernel, anchored \tis{} is $.089$--$.100$ of uniform MSE on Qwen3-4B and $.216$--$.327$ on Phi-4-mini (Figure~\ref{fig:finqa-main}(a,b)); plain \tis{} is unresolved vs. uniform in 7/8 cells. At 400/800 queries, the anchor beats occupancy and rollouts in both Phi workflows (MSE ratios $.81$--$.86$ and $.68$--$.88$, respectively); all 8 contrasts resolve. Occupancy and rollouts remain better on Qwen (Appendix~\ref{sec:finqa-protocol})\purplerev{, whose near-deterministic questions fall in the rare-failure regime of Proposition~\ref{prop:rare-failure}}. \purplerev{With six vs. three calls (same frozen kernels), the Phi anchor's MSE is $.29$--$.64$ of rollouts' at every budget in both workflows, all resolved (Table~\ref{tab:longh-rollout}).}

\begingroup\revcolor{purple}
\noindent\textbf{Q5. Is the tail score itself what helps?}
\label{sec:mixture-controls}
Two declared controls test whether anchor gains come from the tail score: each blends the same floored occupancy shares with uniform or learned mean-influence shares and pays the same pilot (Appendix~\ref{sec:supp-mixture-controls}). Anchor MSE is resolved lower vs. uniform blend in 160/166 settings and higher in none (Table~\ref{tab:mixture-summary}), so gains are not a regularization artifact. Vs. mean blend, results follow Proposition~\ref{prop:rare-failure}. FinQA, including calculator-fault variants (Appendix~\ref{sec:supp-toolfault}): median tail--mean distance is $.000$; mean blend is as good or better. MMLU-Pro (more frequent failures; distance $.08$--$.33$): anchor MSE is resolved lower vs. mean blend in 23/24 confident-error cells (this utility makes confident mistakes nearly worthless) and 14/15 Brier cells at distance $\ge.24$, versus 2/12 below $.18$. With pre-simulation predictions, the rule held in all 3 decisive new settings and 6/8 longer-loop settings; the two misses (Qwen3-4B at $H=8,10$, distance $.26$) favored the anchor without resolving (Appendices~\ref{sec:supp-rare-failure} and~\ref{sec:supp-longh}). In practice, use pilot-estimated distance: anchor if pilot median is $\ge.21$, otherwise mean blend; this picked the better or tied design in 141/148 cells (Appendix~\ref{sec:supp-pilot-rule}).

\begin{table}[t]
\centering\revcolor{purple}
\small
\setlength{\tabcolsep}{4pt}
\caption{\textbf{Anchor vs. equally regularized blends: never resolved worse than uniform; mean-blend gains depend on domain.} Counts of resolved ($|z|\ge2$) lower (higher) anchor MSE vs. each blend. Inventory: 18 cases at $1{,}200$ queries/block; FinQA: 2 generators $\times$ 2 workflows $\times$ 4 budgets.}
\label{tab:mixture-summary}
\resizebox{\linewidth}{!}{
\begin{tabular}{lcc}
\toprule
\tableheadrow
Domain & vs.\ occupancy$+$uniform & vs.\ occupancy$+$mean\\
\midrule
Inventory disruption family & 18/18 (0) & 17/18 (0)\\
FinQA review (held-out) & 16/16 (0) & 0/16 (15)\\
MMLU-Pro, Brier utility (6 generators $\times$ 3 budgets) & 18/18 (0) & 11/18 (0)\\
MMLU-Pro high-stakes panel (3 generators $\times$ 3 budgets) & 9/9 (0) & 5/9 (2)\\
MMLU-Pro, confident-error utility (6 generators) & 24/24 (0) & 23/24 (0)\\
FinQA with calculator faults (ordinary review) & 15/18 (0) & 0/18 (12)\\
FinQA with calculator faults (unit check) & 15/18 (0) & 0/18 (10)\\
\midrule
Prospective divergence test (5 new settings $\times$ 3 budgets) & 15/15 (0) & 11/15 (1)\\
MMLU-Pro longer loops ($H=8,10$; 3 generators $\times$ 3 budgets) & 18/18 (0) & 12/18 (0)\\
FinQA longer reviews ($H=6$; 2 generators $\times$ 2 workflows $\times$ 3 budgets) & 12/12 (0) & 0/12 (9)\\
\bottomrule
\end{tabular}
}
\end{table}

\endgroup

\section{Related Work}
\label{sec:related-work}

\textbf{From distributional inference to query design.}
Categorical distributional RL supplies the representation and projection \citep{bellemare2017distributional,rowland2018analysis}. \citet{zhang2025inference} derive return-law and functional limits under specified, possibly nonuniform, sampling laws. \redrev{Quantile-based evaluation also has semiparametric efficiency guarantees \citep{cheng2026quantile}.} Shortfall identities and efficiency theory are established \citep{rockafellar2000optimization,vanderVaart1998asymptotic}; our addition is the computable, learnable Bellman allocation signal.

\textbf{Adaptive allocation with a learned Bellman model.}
Adaptive stratified sampling learns the variances required by Neyman allocation \citep{neyman1934two,etore2010adaptive,carpentier2015adaptive}. \redrev{Theorem~\ref{thm:adaptive} controls learning our Bellman-dependent score and its allocation, including pilot cost.} Small-pilot failures and defensive mixtures motivate the anchor \citep{cai2022smallpilots,hesterberg1995weighted,owen2000safe}; SaVeR targets trajectory design for mean evaluation \citep{mukherjee2024saver}. Risk-sensitive control changes the policy \citep{tamar2015policy,bauerle2024overview}; with generative access, \citet{deng2025iterated} study sample complexity for iterated-CVaR policy optimization. We instead fix the policy and optimize conditional-query allocation for estimating its tail functional. Appendix~\ref{sec:extended-related} compares access models and guarantees.

\section{Discussion}
\label{sec:discussion}

\purplerev{Allocating queries to estimate CVaR is harder than allocating them for a mean: a kernel's value depends on an unknown continuation model and tail cutoff, and one draw of a reused kernel affects several Bellman stages at once.} \redrev{Our contribution is a computable, learnable CVaR allocation signal for reused conditional laws, with oracle first-order MSE under the stated assumptions}\purplerev{, a defensive variant within a factor two of the oracle, and, on exact grids, a condition under which mean influence suffices.}

\redrev{\tis{} is most promising when conditional access is feasible, tail influence differs meaningfully from visitation or mean influence, and pilots estimate that difference reliably enough to repay their cost.} \purplerev{By Equation~\ref{eq:pilot-payoff-main}, the failures reflect either little opportunity ($t=0$ in Table~\ref{tab:controlled-sep}) or opportunity the pilot cannot learn (plain \tis{} in deep loops; FinQA-Qwen, whose rare errors pilots seldom see). In practice, the pilot's tail--mean divergence selects the allocation, at the same computation as mean influence and more than occupancy or rollouts (Table~\ref{tab:replay-timing}).}
\purplerev{The setting arises wherever an evaluator can restart from a chosen state: auditing rare severe errors of language-model review and agent loops before deployment, where each query is a model call and recurring prompts make kernel reuse the norm, or disruption losses of inventory and maintenance policies in simulators. Tool failures and multi-turn safety evaluation are natural next targets.} \redrev{The guarantees require fixed-dimensional categorical models, a positive quantile margin, independent queries, and the stated pilot and floor schedules. Finite-budget MSE guarantees remain open.}

\bibliography{tis_references}

@inproceedings{bellemare2017distributional,
  title = {A Distributional Perspective on Reinforcement Learning},
  author = {Bellemare, Marc G. and Dabney, Will and Munos, R{\'e}mi},
  booktitle = {Proceedings of the 34th International Conference on Machine Learning},
  series = {Proceedings of Machine Learning Research},
  volume = {70},
  pages = {449--458},
  year = {2017},
  publisher = {PMLR}
}

@inproceedings{rowland2018analysis,
  title = {An Analysis of Categorical Distributional Reinforcement Learning},
  author = {Rowland, Mark and Bellemare, Marc G. and Dabney, Will and Munos, R{\'e}mi and Teh, Yee Whye},
  booktitle = {Proceedings of the 21st International Conference on Artificial Intelligence and Statistics},
  series = {Proceedings of Machine Learning Research},
  volume = {84},
  pages = {29--37},
  year = {2018},
  publisher = {PMLR}
}

@inproceedings{thomas2019cvar,
  title = {Concentration Inequalities for Conditional Value at Risk},
  author = {Thomas, Philip S. and Learned-Miller, Erik},
  booktitle = {Proceedings of the 36th International Conference on Machine Learning},
  series = {Proceedings of Machine Learning Research},
  volume = {97},
  pages = {6225--6233},
  year = {2019},
  publisher = {PMLR}
}

@inproceedings{peng2024statistical,
  title = {Statistical Efficiency of Distributional Temporal Difference Learning},
  author = {Peng, Yang and Zhang, Liangyu and Zhang, Zhihua},
  booktitle = {Advances in Neural Information Processing Systems},
  volume = {37},
  pages = {24724--24761},
  year = {2024},
  doi = {10.52202/079017-0779}
}

@article{peng2025linearctd,
  title = {A Finite Sample Analysis of Distributional Temporal-Difference Learning with Linear Function Approximation},
  author = {Peng, Yang and Jin, Kaicheng and Zhang, Liangyu and Zhang, Zhihua},
  journal = {arXiv preprint arXiv:2502.14172},
  year = {2025}
}

@inproceedings{rowland2024nearminimax,
  title = {Near-Minimax-Optimal Distributional Reinforcement Learning with a Generative Model},
  author = {Rowland, Mark and Wenliang, Li Kevin and Munos, R{\'e}mi and Lyle, Clare and Tang, Yunhao and Dabney, Will},
  booktitle = {Advances in Neural Information Processing Systems},
  volume = {37},
  pages = {132774--132823},
  year = {2024},
  doi = {10.52202/079017-4221}
}

@article{zhang2025inference,
  title = {Estimation and Inference in Distributional Reinforcement Learning},
  author = {Zhang, Liangyu and Peng, Yang and Liang, Jiadong and Yang, Wenhao and Zhang, Zhihua},
  journal = {The Annals of Statistics},
  volume = {53},
  number = {5},
  pages = {1987--2011},
  year = {2025},
  doi = {10.1214/25-AOS2527}
}

@article{cheng2026quantile,
  title = {Statistical Efficiency and Inference of Quantile Distributional Reinforcement Learning},
  author = {Cheng, Zijie and Peng, Yang and Zhang, Zhihua},
  journal = {arXiv preprint arXiv:2607.08444},
  year = {2026}
}

@inproceedings{deng2025iterated,
  title = {Near-Optimal Sample Complexity for Iterated {CVaR} Reinforcement Learning with a Generative Model},
  author = {Deng, Zilong and Khan, Simon and Zou, Shaofeng},
  booktitle = {Proceedings of the 28th International Conference on Artificial Intelligence and Statistics},
  series = {Proceedings of Machine Learning Research},
  volume = {258},
  pages = {3907--3915},
  year = {2025},
  publisher = {PMLR}
}

@inproceedings{dai2023neyman,
  title = {{Clip-OGD}: An Experimental Design for Adaptive {Neyman} Allocation in Sequential Experiments},
  author = {Dai, Jessica and Gradu, Paula and Harshaw, Christopher},
  booktitle = {Advances in Neural Information Processing Systems},
  volume = {36},
  pages = {32235--32269},
  year = {2023}
}

@article{carpentier2015adaptive,
  title = {Adaptive Strategy for Stratified Monte Carlo Sampling},
  author = {Carpentier, Alexandra and Munos, R{\'e}mi and Antos, Andr{\'a}s},
  journal = {Journal of Machine Learning Research},
  volume = {16},
  number = {68},
  pages = {2231--2271},
  year = {2015}
}

@inproceedings{chandak2021universal,
  title = {Universal Off-Policy Evaluation},
  author = {Chandak, Yash and Niekum, Scott and Castro da Silva, Bruno and Learned-Miller, Erik and Brunskill, Emma and Thomas, Philip S.},
  booktitle = {Advances in Neural Information Processing Systems},
  volume = {34},
  pages = {27475--27490},
  year = {2021}
}

@inproceedings{wu2023distributional,
  title = {Distributional Offline Policy Evaluation with Predictive Error Guarantees},
  author = {Wu, Runzhe and Uehara, Masatoshi and Sun, Wen},
  booktitle = {Proceedings of the 40th International Conference on Machine Learning},
  series = {Proceedings of Machine Learning Research},
  volume = {202},
  pages = {37685--37712},
  year = {2023},
  publisher = {PMLR}
}

@inproceedings{hong2025bellman,
  title = {Distributional Off-policy Evaluation with Bellman Residual Minimization},
  author = {Hong, Sungee and Qi, Zhengling and Wong, Raymond K. W.},
  booktitle = {Proceedings of the 28th International Conference on Artificial Intelligence and Statistics},
  series = {Proceedings of Machine Learning Research},
  volume = {258},
  pages = {4006--4014},
  year = {2025},
  publisher = {PMLR}
}

@article{bauerle2011avar,
  title = {Markov Decision Processes with Average-Value-at-Risk Criteria},
  author = {B{\"a}uerle, Nicole and Ott, Jonathan},
  journal = {Mathematical Methods of Operations Research},
  volume = {74},
  number = {3},
  pages = {361--379},
  year = {2011},
  doi = {10.1007/s00186-011-0367-0}
}

@article{rockafellar2000optimization,
  title = {Optimization of Conditional Value-at-Risk},
  author = {Rockafellar, R. Tyrrell and Uryasev, Stanislav},
  journal = {The Journal of Risk},
  volume = {2},
  number = {3},
  pages = {21--41},
  year = {2000},
  doi = {10.21314/JOR.2000.038}
}

@inproceedings{tamar2015policy,
  title = {Policy Gradient for Coherent Risk Measures},
  author = {Tamar, Aviv and Chow, Yinlam and Ghavamzadeh, Mohammad and Mannor, Shie},
  booktitle = {Advances in Neural Information Processing Systems},
  volume = {28},
  pages = {1468--1476},
  year = {2015}
}

@article{bauerle2024overview,
  title = {Markov Decision Processes with Risk-Sensitive Criteria: An Overview},
  author = {B{\"a}uerle, Nicole and Ja{\'s}kiewicz, Anna},
  journal = {Mathematical Methods of Operations Research},
  volume = {99},
  number = {1},
  pages = {141--178},
  year = {2024},
  doi = {10.1007/s00186-024-00857-0}
}

@inproceedings{towers2024gymnasium,
  title = {Gymnasium: A Standard Interface for Reinforcement Learning Environments},
  author = {Towers, Mark and Kwiatkowski, Ariel and Terry, Jordan and Balis, John U. and De Cola, Gianluca and Deleu, Tristan and Goul{\~a}o, Manuel and Kallinteris, Andreas and Krimmel, Markus and KG, Arjun and Perez-Vicente, Rodrigo and Pierr{\'e}, Andrea and Schulhoff, Sander and Tai, Jun Jet and Tan, Hannah and Younis, Omar G.},
  booktitle = {Advances in Neural Information Processing Systems, Datasets and Benchmarks Track},
  volume = {38},
  pages = {163114--163129},
  year = {2025},
  doi = {10.52202/085713-4916}
}

@book{vanderVaart1998asymptotic,
  title = {Asymptotic Statistics},
  author = {van der Vaart, Aad W.},
  year = {1998},
  publisher = {Cambridge University Press},
  address = {Cambridge}
}

@inproceedings{mukherjee2022revar,
  title = {{ReVar}: Strengthening Policy Evaluation via Reduced Variance Sampling},
  author = {Mukherjee, Subhojyoti and Hanna, Josiah P. and Nowak, Robert D.},
  booktitle = {Proceedings of the Thirty-Eighth Conference on Uncertainty in Artificial Intelligence},
  series = {Proceedings of Machine Learning Research},
  volume = {180},
  pages = {1413--1422},
  year = {2022},
  publisher = {PMLR}
}

@inproceedings{mukherjee2024saver,
  title = {{SaVeR}: Optimal Data Collection Strategy for Safe Policy Evaluation in Tabular {MDP}},
  author = {Mukherjee, Subhojyoti and Hanna, Josiah P. and Nowak, Robert D.},
  booktitle = {Proceedings of the 41st International Conference on Machine Learning},
  series = {Proceedings of Machine Learning Research},
  volume = {235},
  pages = {36531--36576},
  year = {2024},
  publisher = {PMLR}
}

@article{hoeffding1963probability,
  title = {Probability Inequalities for Sums of Bounded Random Variables},
  author = {Hoeffding, Wassily},
  journal = {Journal of the American Statistical Association},
  volume = {58},
  number = {301},
  pages = {13--30},
  year = {1963},
  doi = {10.1080/01621459.1963.10500830}
}

@article{neyman1934two,
  title = {On the Two Different Aspects of the Representative Method: The Method of Stratified Sampling and the Method of Purposive Selection},
  author = {Neyman, Jerzy},
  journal = {Journal of the Royal Statistical Society},
  volume = {97},
  number = {4},
  pages = {558--606},
  year = {1934},
  doi = {10.1111/j.2397-2335.1934.tb04184.x}
}

@article{etore2010adaptive,
  title = {Adaptive Optimal Allocation in Stratified Sampling Methods},
  author = {{\'E}tor{\'e}, Pierre and Jourdain, Benjamin},
  journal = {Methodology and Computing in Applied Probability},
  volume = {12},
  number = {3},
  pages = {335--360},
  year = {2010},
  doi = {10.1007/s11009-008-9108-0}
}

@article{douglas2026logging,
  title = {Logging Policy Design for Off-Policy Evaluation},
  author = {Douglas, Connor and Persson, Joel and Provost, Foster},
  journal = {arXiv preprint arXiv:2605.15108},
  year = {2026}
}

@inproceedings{wang2024mmlupro,
  title = {{MMLU-Pro}: A More Robust and Challenging Multi-Task Language Understanding Benchmark},
  author = {Wang, Yubo and Ma, Xueguang and Zhang, Ge and Ni, Yuansheng and Chandra, Abhranil and Guo, Shiguang and Ren, Weiming and Arulraj, Aaran and He, Xuan and Jiang, Ziyan and Li, Tianle and Ku, Max and Wang, Kai and Zhuang, Alex and Fan, Rongqi and Yue, Xiang and Chen, Wenhu},
  booktitle = {Advances in Neural Information Processing Systems, Datasets and Benchmarks Track},
  volume = {37},
  pages = {95266--95290},
  year = {2024},
  doi = {10.52202/079017-3018}
}

@misc{qwen3instruct2507,
  title = {{Qwen3-4B-Instruct-2507} Model Card},
  author = {{Qwen Team}},
  year = {2025},
  howpublished = {Hugging Face model repository},
  url = {https://huggingface.co/Qwen/Qwen3-4B-Instruct-2507},
  note = {Accessed Aug. 21, 2026}
}

@misc{granite42,
  title = {{Granite-4.2-8B} Model Card},
  author = {{Granite Team, IBM}},
  year = {2026},
  howpublished = {Hugging Face model repository},
  url = {https://huggingface.co/ibm-granite/granite-4.2-8b},
  note = {Accessed Aug. 26, 2026}
}

@misc{mistralsmall24b,
  title = {{Mistral-Small-24B-Instruct-2501} Model Card},
  author = {{Mistral AI}},
  year = {2025},
  howpublished = {Hugging Face model repository},
  url = {https://huggingface.co/mistralai/Mistral-Small-24B-Instruct-2501},
  note = {Accessed Aug. 31, 2026}
}

@misc{qwen332b,
  title = {{Qwen3-32B} Model Card},
  author = {{Qwen Team}},
  year = {2025},
  howpublished = {Hugging Face model repository},
  url = {https://huggingface.co/Qwen/Qwen3-32B},
  note = {Accessed Aug. 31, 2026}
}

@misc{glm432b,
  title = {{GLM-4-32B-0414} Model Card},
  author = {{Zhipu AI}},
  year = {2025},
  howpublished = {Hugging Face model repository},
  url = {https://huggingface.co/zai-org/GLM-4-32B-0414},
  note = {Accessed Aug. 31, 2026}
}

@article{phi4mini,
  title = {Phi-4-Mini Technical Report: Compact yet Powerful Multimodal Language Models via Mixture-of-{LoRA}s},
  author = {{Microsoft} and Abouelenin, Abdelrahman and Ashfaq, Atabak and Atkinson, Adam and others},
  journal = {arXiv preprint arXiv:2503.01743},
  year = {2025}
}

@article{zhu2024uncertainty,
  title = {Uncertainty Quantification and Exploration for Reinforcement Learning},
  author = {Zhu, Yi and Dong, Jing and Lam, Henry},
  journal = {Operations Research},
  volume = {72},
  number = {4},
  pages = {1689--1709},
  year = {2024},
  doi = {10.1287/opre.2023.2436}
}

@article{peng2026onlineinference,
  title = {Online Inference in Distributional Temporal-Difference Learning},
  author = {Peng, Yang and Zhang, Liangyu},
  journal = {arXiv preprint arXiv:2608.14408},
  year = {2026}
}

@inproceedings{kossen2021active,
  title = {Active Testing: Sample-Efficient Model Evaluation},
  author = {Kossen, Jannik and Farquhar, Sebastian and Gal, Yarin and Rainforth, Tom},
  booktitle = {Proceedings of the 38th International Conference on Machine Learning},
  series = {Proceedings of Machine Learning Research},
  volume = {139},
  pages = {5753--5763},
  year = {2021},
  publisher = {PMLR}
}

@inproceedings{polo2024tinybenchmarks,
  title = {tiny{B}enchmarks: Evaluating {LLM}s with Fewer Examples},
  author = {Maia Polo, Felipe and Weber, Lucas and Choshen, Leshem and Sun, Yuekai and Xu, Gongjun and Yurochkin, Mikhail},
  booktitle = {Proceedings of the 41st International Conference on Machine Learning},
  series = {Proceedings of Machine Learning Research},
  volume = {235},
  pages = {34303--34326},
  year = {2024},
  publisher = {PMLR}
}

@inproceedings{nguyen2018active,
  title = {Active Testing: An Efficient and Robust Framework for Estimating Accuracy},
  author = {Nguyen, Phuc and Ramanan, Deva and Fowlkes, Charless},
  booktitle = {Proceedings of the 35th International Conference on Machine Learning},
  series = {Proceedings of Machine Learning Research},
  volume = {80},
  pages = {3759--3768},
  year = {2018},
  publisher = {PMLR}
}

@inproceedings{li2025activeeval,
  title = {Active Evaluation Acquisition for Efficient {LLM} Benchmarking},
  author = {Li, Yang and Ma, Jie and Ballesteros, Miguel and Benajiba, Yassine and Horwood, Graham},
  booktitle = {Proceedings of the 42nd International Conference on Machine Learning},
  series = {Proceedings of Machine Learning Research},
  volume = {267},
  pages = {35581--35602},
  year = {2025},
  publisher = {PMLR}
}

@article{brier1950,
  title = {Verification of Forecasts Expressed in Terms of Probability},
  author = {Brier, Glenn W.},
  journal = {Monthly Weather Review},
  volume = {78},
  number = {1},
  pages = {1--3},
  year = {1950}
}

@inproceedings{chen2021finqa,
  title = {Fin{QA}: A Dataset of Numerical Reasoning over Financial Data},
  author = {Chen, Zhiyu and Chen, Wenhu and Smiley, Charese and Shah, Sameena and Borova, Iana and Langdon, Dylan and Moussa, Reema and Beane, Matt and Huang, Ting-Hao and Routledge, Bryan and Wang, William Yang},
  booktitle = {Proceedings of the 2021 Conference on Empirical Methods in Natural Language Processing},
  pages = {3697--3711},
  year = {2021},
  publisher = {Association for Computational Linguistics},
  doi = {10.18653/v1/2021.emnlp-main.300}
}

@article{cai2022smallpilots,
  title = {On the Performance of the {Neyman} Allocation with Small Pilots},
  author = {Cai, Yong and Rafi, Ahnaf},
  journal = {arXiv preprint arXiv:2206.04643},
  year = {2022},
  note = {Version 4, revised June 2024}
}

@article{hesterberg1995weighted,
  title = {Weighted Average Importance Sampling and Defensive Mixture Distributions},
  author = {Hesterberg, Tim},
  journal = {Technometrics},
  volume = {37},
  number = {2},
  pages = {185--194},
  year = {1995},
  doi = {10.1080/00401706.1995.10484303}
}

@article{owen2000safe,
  title = {Safe and Effective Importance Sampling},
  author = {Owen, Art and Zhou, Yi},
  journal = {Journal of the American Statistical Association},
  volume = {95},
  number = {449},
  pages = {135--143},
  year = {2000},
  doi = {10.1080/01621459.2000.10473909}
}
\bibliographystyle{iclr2027_conference}

\clearpage
\appendix

\addtocontents{toc}{\protect\setcounter{tocdepth}{2}}

\section*{\textbf{Supplement: TIS for CVaR Policy Evaluation}}

Appendices~\ref{sec:supp-notation}--\ref{sec:supp-learning} develop the allocation theory from the stop-loss Bellman representation to fixed-design efficiency, learned oracle adaptation, and the anchored safeguard. The main results are proved as follows: Theorem~\ref{thm:allocation-clt} in Appendix~\ref{sec:supp-fixed-design}, Theorem~\ref{thm:oracle} in Appendix~\ref{sec:supp-efficiency}, Proposition~\ref{prop:controlled-separation} in Appendix~\ref{sec:supp-controlled-separation}, Theorem~\ref{thm:adaptive} in Appendix~\ref{sec:supp-adaptation}, the anchor's guarantees (Proposition~\ref{prop:anchored} and Corollary~\ref{cor:anchor-asymptotic}) in Appendix~\ref{sec:supp-anchor}, and the pilot-payoff condition (Equation~\ref{eq:pilot-payoff-main}) in Appendix~\ref{sec:pilot-payoff}. Appendix~\ref{sec:supp-evidence} closes the remaining links to the main paper: Appendix~\ref{sec:supp-representation} bounds categorical representation error; the experimental subsections give the protocols and evidence for \textbf{Q1--Q5}; and Appendix~\ref{sec:supp-rare-failure} proves Proposition~\ref{prop:rare-failure} and connects it to the tail--mean diagnostic. Section~\ref{sec:extended-related} positions these results against the closest foundations.

\textbf{Proof dependencies.} Lemma~\ref{lem:projection} justifies the stop-loss Bellman representation. Lemma~\ref{lem:fixed-design-moments} controls its empirical remainder and, together with the positive quantile margin, yields Theorem~\ref{thm:allocation-clt}. That theorem identifies the influence used in the efficiency calculation of Theorem~\ref{thm:oracle}. Lemma~\ref{lem:pilot-concentration} controls pilot estimates of the same influence; combined with Lemma~\ref{lem:fixed-design-moments}, it yields Theorem~\ref{thm:adaptive}. Proposition~\ref{prop:anchored} then transfers the allocation control to the anchored design in Corollary~\ref{cor:anchor-asymptotic}.

\begingroup\small
\let\tocaddvspace\addvspace
\renewcommand{\addvspace}[1]{\tocaddvspace{0.35em}}
\makeatletter
\@starttoc{toc}
\makeatother
\endgroup

\section{Notation and Assumptions}
\label{sec:supp-notation}

A query group identifies a sampled law; a Bellman row identifies one use of it. This section makes that distinction precise and connects the probability, measure, and shortfall notation used in the proofs.

\textbf{Queryable-group experiment.}
$\calG$ is a finite set of conditionally and independently queryable laws $P_g$, $G:=|\calG|$, and the objective is one fixed scalar functional of those laws. Section~\ref{sec:supp-generic-allocation} optimizes the variance once the groupwise influence scales are identified.

\textbf{Categorical Bellman conditions.}
We evaluate one root $(H,s_0)$ and let $\calB$ index the required layer-state rows. A group may feed several rows through known nonnegative mixture coefficients summing to one within each Bellman row. A known component, if present, is represented by a degenerate query law with zero influence. Sub-probability row sums arise in the continuation matrix from termination at nonpositive shifted thresholds. The return law still retains all probability mass. In the stationary model,
$g=(s,a)$, $W_g=(R,S')\sim P_{s,a}$, and the coefficient in row $(h,s)$ is $\pi_h(a\mid s)$. The ordered grid
$\calZ=\{z_0,\ldots,z_{K-1}\}\subset[0,H]$ contains both endpoints and has maximum gap $\Delta$. The proofs keep $H$, $|\mathcal S|$, $|\mathcal A|$, and $K$ fixed, assume bounded rewards, and impose the positive categorical quantile margin in Equation~\ref{eq:margin}.

\textbf{Untied special case.}
For independently queryable layer-specific laws, $g=b=(h,s)$ and $P_b$ is the policy-mixture law obtained by drawing $A\sim\pi_h(\cdot\mid s)$ before the transition. One law then feeds one row. Structurally unreachable groups may be removed in either model only when the declared support and fixed policy certify their irrelevance. The retained rows must be closed under every possible continuation transition in the declared outcome spaces. These fixed outcome spaces also define the nonparametric product model in the efficiency theorem.

\textbf{Equivalent probability and measure forms.}
The measure $\eta_h^*(s)=\sum_k p^*_{h,s,k}\delta_{z_k}$ is another notation for the grid probabilities in Section~\ref{sec:problem}; hats denote the empirical version. The categorical target and estimator are
\begin{equation}
 C_{\alpha,K}=\cvar_\alpha \big(\eta_H^*(s_0)\big),\qquad
 \widehat C_N=\cvar_\alpha \big(\widehat\eta_H(s_0)\big).
 \label{eq:target-estimator}
\end{equation}
For the shift-and-project matrix in Equation~\ref{eq:cat-dp},
$Q(R)_{kj}=\ell_k(R+z_j)$, where $\ell_k(y)$ is the categorical projection weight on $z_k$. Column sums are one, including at clipped endpoints. Thus the vector equation is equivalent to
\begin{equation}
 \eta_0^*(s)=\delta_0,\qquad
 \eta_h^*(s)=\sum_a\pi_h(a\mid s)\,
 \E_{(R,S')\sim P_{s,a}}\!\left[
 \PiC\big((f_R)_\#\eta_{h-1}^*(S')\big)\right],
 \quad f_R(z)=R+z.
 \label{eq:categorical-measure-backup}
\end{equation}
Here $(f_R)_\#\nu$ is the law of $R+Z$ for $Z\sim\nu$. Replacing each expectation by the same group's empirical mean gives the estimator in Equation~\ref{eq:cat-dp}; the measure notation describes the same computation.

\textbf{Stacked-coordinate convention and proof roadmap.}
We index a scalar shortfall (stop-loss) coordinate by $x=(h,s,q)$, write $e_x$ for its standard basis vector, and write $U_{h,s}=U_h(s,\cdot)\in\mathbb R^K$ for the threshold block. A vector such as $U$ stacks these coordinates; a row index $b=(h,s)$ selects one block. Sections~\ref{sec:supp-identities}--\ref{sec:supp-fixed-results} derive the influence and its sampling limits. Section~\ref{sec:supp-learning} controls the extra error from learning the allocation.

For an outcome $W=(R,S')$ and stacked stop-loss vector $U$, define the sampled stop-loss Bellman target at row $b=(h,s)$ by
\begin{equation}
 \big[T_b \big(W;U \big) \big](q)=
 \begin{cases}
 (q-R)_+, & h=1,\\
 \mathcal I_{\calZ} \big[U_{h-1} \big(S',\cdot \big) \big](q-R), & h>1,
 \end{cases}
 \label{eq:sample-target}
\end{equation}
where $\mathcal I_{\calZ}$ is linear interpolation on the threshold grid, extended by zero for nonpositive arguments. \rev{For group $g$, let $\mathcal T_g(W_g;U)$ be its full stacked affine contribution, including all known mixture coefficients. In the shared-kernel model, for $g=(s,a)$ and
$b=(h,s)$,
\[
 \big[\mathcal T_{s,a} \big(W_g;U \big) \big]_b
 =\pi_h(a\mid s)T_b \big(W_g;U \big),
\]
and rows at states other than $s$ are zero. Thus one $W_g$ contributes jointly to every required layer row at state $s$; this is why its layer effects must be summed before their variance is computed.} For a categorical continuation $Z\sim\eta_{h-1}^*(S')$, the stop-loss transform $v\mapsto\E[(v-Z)_+]$ is affine on every grid interval $[z_j,z_{j+1}]$, and its values at the knots are exactly $U_{h-1}^*(S',z_j)$. Hence, for $0<v\le H$, linear interpolation evaluates it exactly:
\[
 \mathcal I_{\calZ}[U_{h-1}^*(S',\cdot)](v)=\E[(v-Z)_+].
\]
Taking $v=q-R$ (with the stated zero extension when $q-R\le0$) gives the shortfall of $R+Z$ at threshold $q$. Lemma~\ref{lem:projection} shows that categorical projection preserves this shortfall for every grid threshold $q$. Therefore
\begin{equation}
 U^*=\sum_{g\in\calG}\E_{P_g}\big[\mathcal T_g \big(W_g;U^* \big) \big]
 =\mathbf b+MU^*.
 \label{eq:mean-target}
\end{equation}
\rev{Set $A:=I-M$. Because $q-R\le q\le H$, the finite-horizon target evaluates the interpolant only at or below the upper grid endpoint. The stated nonpositive extension is sufficient. In \eqref{eq:mean-target}, $\mathbf b$ collects terms independent of continuation values, including the $h=1$ rows. The matrix $M$ collects the known mixture coefficients, transition expectations, and interpolation weights multiplying lower-layer coordinates.}
The matrix $M$ lowers the layer and $M^H=0$. Given $n_g$ observations per group, replace each expectation by its empirical average. This is the stop-loss transform of the shared-kernel empirical categorical estimator. Denote the resulting affine map by $\widehat{\mathbf b}+\widehat M U$; again $\widehat M^H=0$ for every dataset.

\textbf{Explicit shared-kernel influence.}
Let $r_{h,s}$ be the adjoint block for layer--state row $(h,s)$. For $g=(s,a)$, the full-vector formula in Equation~\ref{eq:block-sigma} is
\begin{equation}
 \phi_{s,a}(W)=-\frac1\alpha\sum_{h=1}^H\pi_h(a\mid s)r_{h,s}^\top
 \Big\{T_h \big(W;U^* \big)-\E_{P_{s,a}}T_h \big(W;U^* \big)\Big\}.
 \label{eq:shared-influence-explicit}
\end{equation}
The variance of this sum includes all cross-layer covariances induced by reusing the same observation.

\section{Allocation and Bellman Identities}
\label{sec:supp-identities}
Equation~\ref{eq:oracle-value} follows from classical allocation once the scales are known. The identities below connect those scales to CVaR: projection preserves shortfalls, and Bellman propagation carries each group's error to the root.

\subsection{Minimizing the allocation variance}
\label{sec:supp-generic-allocation}

Once the influence expansion gives $V(w)=\sum_g\sigma_g^2/w_g$, classical Neyman allocation follows from Cauchy--Schwarz
\citep{neyman1934two}:
\[
 \left(\sum_g\sigma_g\right)^2
 =\left(\sum_g\frac{\sigma_g}{\sqrt{w_g}}\sqrt{w_g}\right)^2
 \le \sum_g\frac{\sigma_g^2}{w_g},\qquad \sum_g w_g=1.
\]
If all scales are positive, equality holds only at $w_g=\sigma_g/S_\sigma$, where $S_\sigma=\sum_g\sigma_g$. If $S_\sigma>0$ but some scales vanish, this boundary design gives the infimum over positive designs: the mixtures
$(1-\lambda)\sigma/S_\sigma+\lambda\mathbf1/G$ approach it as $\lambda\downarrow0$. If all scales vanish, every design has zero first-order variance. Sections~\ref{sec:supp-fixed-point} and~\ref{sec:supp-fixed-design} derive the influence expansion and its CLT and MSE limits for the categorical Bellman estimator.

\subsection{Projection and categorical CVaR identities}
\label{sec:supp-projection}

Section~\ref{sec:shared_kernel_tail_influence} relies on projection preserving shortfalls and a stable quantile making CVaR locally affine. For return $X$ with CDF $F_X$, lower-tail CVaR is
\begin{equation}
\cvar_\alpha(X)=\frac1\alpha\int_0^\alpha F_X^{-1}(u)\,du, \ \  \alpha\in(0,1), \ \ \text{with} \ \ F_X^{-1}(u)=\inf\{x:F_X(x)\ge u\}. 
 \label{eq:cvar-definition} 
\end{equation}

\begin{lemma}[Projection identity]
\label{lem:projection}
For every $q\in\calZ$ and every law $\nu$ supported on $[0,\infty)$,
\begin{equation}
 \int(q-z)_+\,d \big(\PiC\nu \big)(z)=\int(q-z)_+\,d\nu(z).
 \label{eq:projection-identity}
\end{equation}
\end{lemma}

\begin{proof}
\rev{For $y\in[z_j,z_{j+1}]$, categorical projection \citep{bellemare2017distributional,rowland2018analysis} sends $\delta_y$ to}
\[
 \frac{z_{j+1}-y}{z_{j+1}-z_j}\delta_{z_j}
 +\frac{y-z_j}{z_{j+1}-z_j}\delta_{z_{j+1}}.
\]
\rev{For a grid atom $q$, the map $z\mapsto(q-z)_+$ is affine on every grid cell. Its expectation is therefore preserved by the barycentric projection. If $y\ge H$, projection clips to $H\ge q$ and both the original and clipped payoffs are zero. Inputs below zero are excluded by the nonnegative-return model. Linearity proves the identity for every input law supported on $[0,\infty)$.}
\end{proof}

\rev{For a categorical law $p$ with CDF $F_k=\sum_{i\le k}p_i$, set $F_{-1}=0$, $k_\alpha=\min\{k:F_k\ge\alpha\}$, and $q_\alpha=z_{k_\alpha}$. The quantile-integral definition of lower-tail CVaR
\citep{rockafellar2000optimization} gives}
\begin{align}
 \cvar_\alpha(p)
 &=\frac{1}{\alpha}\left\{\sum_{i<k_\alpha}p_i z_i
 +(\alpha-F_{k_\alpha-1})z_{k_\alpha}\right\}\\
 &=z_{k_\alpha}-\frac1\alpha\sum_{i<k_\alpha}p_i(z_{k_\alpha}-z_i).
 \label{eq:categorical-cvar}
\end{align}
Thus CVaR is affine in a neighborhood where the VaR index is fixed.

For the population root law, let $F_k^*=\sum_{j\le k}p^*_{H,s_0,j}$, $F_{-1}^*=0$, and $k_\alpha=\min\{k:F_k^*\ge\alpha\}$. The positive quantile margin used in the main text is
\begin{equation}
 m_\alpha=\min\big\{\alpha-F^*_{k_\alpha-1},\ F^*_{k_\alpha}-\alpha\big\}>0.
 \label{eq:margin}
\end{equation}
It places $\alpha$ strictly inside the quantile atom's cumulative-mass interval. Throughout the finite-horizon proof, $C_{\alpha,K}=q_\alpha-\alpha^{-1}U^*_H(s_0,q_\alpha)$ is the population categorical CVaR; $\widehat C_N$ uses the empirical fixed point and its empirical quantile atom.

\textbf{Why a positive margin matters.}
For $H=1$, take $Z\in\{0,1\}$ with $p:=\Prob(Z=0)$. Then
$\cvar_\alpha(Z)=\max\{0,1-p/\alpha\}$. At $p=\alpha$ the margin vanishes, and for an empirical fraction $\widehat p$,
\[
 \sqrt n\,\widehat C_n
 =\max\{0,-\sqrt n(\widehat p-\alpha)/\alpha\}
 \Rightarrow\max\{0,-Z_0/\alpha\},\qquad
 Z_0\sim\mathcal N(0,\alpha(1-\alpha)).
\]
The limit has an atom at zero, so it is non-Gaussian. At $p>\alpha$ the margin is positive, CVaR is locally constant, and every influence is zero. The fixed-design theorem includes this degenerate limit. Oracle adaptation assumes $S_\sigma>0$.

\textbf{A shared-kernel example with nonzero covariance.}
Take one state, one action, $H=2$, rewards $R\sim\operatorname{Bernoulli}(p)$, grid $\{0,1,2\}$, and $p=1/2$. With $\alpha=3/5$, the root VaR is $1$ and the margin is $3/20$. Locally,
\[
 C_{\alpha,K}=1-\frac{(1-p)^2}{\alpha},\qquad
 \phi(R)=\frac{2(1-p)}{\alpha}(R-p)
          =\frac{R-1/2}{\alpha}.
\]
Each of the two layer contributions is $(R-1/2)/(2\alpha)$. Their sum has variance $1/(4\alpha^2)=25/36$, whereas deleting their covariance gives $1/(8\alpha^2)=25/72$. This demonstrates the variance error from treating a shared draw as two independent draws. There is only one group, so the example isolates covariance. The allocation itself is trivial.

\subsection{Exact empirical fixed-point expansion}
\label{sec:supp-fixed-point}
For Theorem~\ref{thm:allocation-clt}, we separate leading sampling error from feedback due to estimating the continuation model. Define the centered contribution of one group observation
\begin{equation}
 \Xi_g(W_g)=\mathcal T_g(W_g;U^*)
 -\E_{P_g}\mathcal T_g(W_g;U^*),
 \qquad \E_{P_g}\Xi_g=0,
\end{equation}
and the combined empirical Bellman error
\begin{equation}
 \widehat\xi=\sum_{g\in\calG}\frac1{n_g}
 \sum_{i=1}^{n_g}\Xi_g(W_{g,i}).
\end{equation}
\rev{By the affine form of $\mathcal T_g$, this same perturbation can be written as
\[
 \widehat\xi=\Big(\widehat{\mathbf b}-\mathbf b \Big)+ \Big(\widehat M-M \Big)U^*.
\]
Thus $\widehat\xi$ is the empirical Bellman error evaluated at the population fixed point. The resolvent below converts it into fixed-point estimation error.} Subtracting the population and empirical fixed-point equations gives the exact identity
\begin{equation}
 \widehat U-U^*= \Big(I-\widehat M \Big)^{-1}\widehat\xi.
 \label{eq:exact-expansion}
\end{equation}
Indeed,
\begin{align*}
 \Big(I-\widehat M \Big) \Big(\widehat U-U^* \Big)
 &=\widehat{\mathbf b}+\widehat M U^*-U^*\\
 &= \Big(\widehat{\mathbf b}-\mathbf b \Big)+ \Big(\widehat M-M \Big)U^*=\widehat\xi.
\end{align*}
Both inverses are finite sums:
\[
 \Big(I-\widehat M \Big)^{-1}=\sum_{t=0}^{H-1}\widehat M^t,
 \qquad \Big(I-M \Big)^{-1}=\sum_{t=0}^{H-1}M^t.
\]
Every row of $M$ and $\widehat M$ is sub-probability, so their induced infinity norms are at most one and both inverse norms are at most $H$.

The resolvent identity gives
\begin{align}
 \widehat U-U^*
 &=A^{-1}\widehat\xi+R_N,\label{eq:first-order}\\
R_N&=A^{-1} \Big(\widehat M-M \Big) \Big(I-\widehat M \Big)^{-1}\widehat\xi.
 \label{eq:remainder}
\end{align}
\rev{In the finite-horizon bounds below, $\|\cdot\|$ denotes the vector infinity norm or its induced matrix infinity norm.} With $n_{\min}=\min_g n_g$, finite dimension and bounded observations imply
\begin{equation}
 \left\|\widehat M-M\right\|=O_p\left(n_{\min}^{-1/2}\right),\quad
 \left\|\widehat\xi\right\|=O_p\left(n_{\min}^{-1/2}\right),\quad
 \left\|R_N\right\|=O_p\left(n_{\min}^{-1}\right).
 \label{eq:remainder-rate}
\end{equation}

\begin{lemma}[Uniform moment control]
\label{lem:fixed-design-moments}
For every fixed integer $p\ge2$, there is a finite constant $C_p$, depending only on the fixed model dimensions, grid, horizon, and $p$, such that
\begin{align}
 \E\left\|\widehat M-M\right\|^p+\E\left\|\widehat\xi\right\|^p
 &\le C_p n_{\min}^{-p/2},\nonumber\\
 \E\left\|R_N\right\|^p&\le C_p n_{\min}^{-p}.
 \label{eq:fixed-design-moments}
\end{align}
In particular, under a stable allocation,
$N\E\|R_N\|^2\to0$.
\end{lemma}

\begin{proof}
\textbf{Idea.} Both empirical coefficient error and Bellman error are bounded sample averages, hence of order $n_{\min}^{-1/2}$. The fixed-point remainder is their product, so it is one order smaller. The deterministic resolvent bound prevents the recursion from amplifying these rates.

Each entry of $\widehat M-M$ and $\widehat\xi$ is a finite sum, over groups, of centered averages of bounded random variables. For a centered average $\bar X_g$ of $n_g$ independent bounded variables, Hoeffding's inequality \citep{hoeffding1963probability} gives $\Prob(|\bar X_g|>t)\le2e^{-cn_gt^2}$. Integrating this tail via $\E|\bar X_g|^p=\int_0^\infty p t^{p-1}\Prob(|\bar X_g|>t)\,dt$ yields $\E|\bar X_g|^p\le C_p n_g^{-p/2}$. The number of groups and matrix entries is fixed, so norm equivalence and a finite-sum inequality give the first line of \eqref{eq:fixed-design-moments}. Nilpotence and the sub-probability row structure give the deterministic bounds $\|A^{-1}\|\le H$ and $\|(I-\widehat M)^{-1}\|\le H$. Hence
\[
 \left\|R_N\right\|\le H^2\left\|\widehat M-M\right\|\,\left\|\widehat\xi\right\|.
\]
Cauchy--Schwarz with the preceding $2p$-moment bounds proves the second line. If $n_{\min}\asymp N$, it yields $N\E\|R_N\|^2=O(N^{-1})$.
\end{proof}

\section{Fixed-Design Limits and Structural Interpretation}
\label{sec:supp-fixed-results}
Theorems~\ref{thm:allocation-clt} and~\ref{thm:oracle} turn the influence calculation into an attainable accuracy benchmark for fixed query shares. The untied factorization explains the allocation ablations; Proposition~\ref{prop:controlled-separation} then shows why visitation and reward moments cannot determine the best tail allocation.

\subsection{Fixed-design CLT and normalized MSE}
\label{sec:supp-fixed-design}

\begin{proof}[Proof of Theorem~\ref{thm:allocation-clt}]
\textbf{Idea.} There are three issues to separate: linearize the empirical Bellman fixed point, show the empirical VaR atom stays on the same categorical cell, and then transfer the groupwise CLT and second moment through that locally affine CVaR readout.

\emph{Step 1: asymptotic linearity on the correct VaR cell.}
Let $E_N$ be the event that the empirical and population categorical VaR indices agree, set $x_0=(H,s_0,q_\alpha)$ and $r^\top=e_{x_0}^\top A^{-1}$, and define
\begin{equation}
 \phi_g(W_g):=-\alpha^{-1}r^\top\Xi_g(W_g),\qquad
 L_N:=\sum_g\frac1{n_g}\sum_{i=1}^{n_g}\phi_g(W_{g,i}).
 \label{eq:influence}
\end{equation}
Then $\E_{P_g}\phi_g=0$ and $\E_{P_g}\phi_g^2=\sigma_g^2$. On $E_N$,
\eqref{eq:categorical-cvar} and \eqref{eq:first-order} give
\begin{equation}
 \widehat C_N-C_{\alpha,K}=L_N+\widetilde R_N,\qquad
 \widetilde R_N:=-\alpha^{-1}e_{x_0}^\top R_N.
 \label{eq:asymptotic-linear}
\end{equation}
Under a stable allocation, $n_{\min}:=\min_g n_g\asymp N$, so Lemma~\ref{lem:fixed-design-moments} gives $\sqrt N\widetilde R_N=o_p(1)$. For each fixed group,
\[
 \frac{\sqrt N}{n_g}\sum_{i=1}^{n_g}\phi_g(W_{g,i})
 =\sqrt{\frac{N}{n_g}}\,\frac1{\sqrt{n_g}}\sum_{i=1}^{n_g}\phi_g(W_{g,i})
 \Rightarrow \mathcal N\!\left(0,\frac{\sigma_g^2}{w_g}\right),
\]
by the ordinary CLT and $n_g/N\to w_g>0$. The groups are independent and their number is fixed, so the vector of group terms converges jointly to independent Gaussian limits. Summing the coordinates gives
\begin{equation}
 \sqrt N L_N\Rightarrow\mathcal N(0,V(w)),\qquad
 V(w):=\sum_g\frac{\sigma_g^2}{w_g}.
\end{equation}

\emph{Step 2: the VaR cell is correct with exponentially high probability.}
Each entry of $\widehat{\mathbf b}-\mathbf b$ and $\widehat M-M$ is a finite sum of averages of bounded variables. Hoeffding's inequality \citep{hoeffding1963probability} and a union bound therefore give constants $c_1,c_2>0$ such that
\begin{equation}
 \Prob\left(\|\widehat{\mathbf b}-\mathbf b\|_\infty +
  \|\widehat M-M\|_\infty>t\right)
 \le c_1e^{-c_2n_{\min}t^2},\qquad t>0.
 \label{eq:operator-concentration}
\end{equation}
For a categorical law with stop-loss vector $U$ and CDF $F$,
\[
 F_{j-1}=\frac{U(z_j)-U(z_{j-1})}{z_j-z_{j-1}}
 \quad(j=1,\ldots,K-1),\qquad F_{K-1}=1.
\]
Because every shortfall coordinate satisfies $0\le U_h^*(s,q)\le q\le H$, we have $\|U^*\|_\infty\le H$. Thus, with $\delta_{\min}:=\min_{j<K-1}(z_{j+1}-z_j)>0$, the exact fixed-point identity and $\|(I-\widehat M)^{-1}\|_\infty\le H$ imply
\begin{align}
 \|\widehat U-U^*\|_\infty
 &\le H\{\|\widehat{\mathbf b}-\mathbf b\|_\infty +
 H\|\widehat M-M\|_\infty\},\nonumber\\
 \|\widehat F-F^*\|_\infty
 &\le\frac{2H}{\delta_{\min}}
 \{\|\widehat{\mathbf b}-\mathbf b\|_\infty +
  H\|\widehat M-M\|_\infty\}.
\end{align}
Combining this bound with \eqref{eq:operator-concentration} gives
\begin{equation}
 \Prob(E_N^c)
 \le\Prob(\|\widehat F-F^*\|_\infty\ge m_\alpha)
 \le c_3e^{-c_4n_{\min}m_\alpha^2}.
 \label{eq:index-concentration}
\end{equation}
Indeed, an error smaller than $m_\alpha$ preserves $\widehat F_{k_\alpha-1}<\alpha<\widehat F_{k_\alpha}$; at the endpoints use $\widehat F_{-1}=0$ and $\widehat F_{K-1}=1$.

\emph{Step 3: normalized MSE and transfer off the good event.} 
For the second-moment claim, centering and independence give
\begin{equation}
 N\E L_N^2=N\sum_g\frac{\sigma_g^2}{n_g}\longrightarrow V(w).
 \label{eq:fixed-leading-mse}
\end{equation}
Lemma~\ref{lem:fixed-design-moments} also gives
$N\E\widetilde R_N^2=O(N^{-1})$ and, by Cauchy--Schwarz,
$N|\E[L_N\widetilde R_N]|=o(1)$. Hence
$N\E(L_N+\widetilde R_N)^2\to V(w)$.
All three variables $\widehat C_N-C_{\alpha,K}$, $L_N$, and $\widetilde R_N$ are uniformly bounded; for the last, use \eqref{eq:remainder} and the deterministic resolvent bounds. Therefore
\begin{equation}
 N\E\!\left[
 \{(\widehat C_N-C_{\alpha,K})^2+
 (L_N+\widetilde R_N)^2\}\ind_{E_N^c}\right]
 \le C N\Prob(E_N^c)\longrightarrow0.
 \label{eq:wrong-cell-mse}
\end{equation}
Equation~\ref{eq:asymptotic-linear} holds on $E_N$, so the last two displays prove the normalized-MSE limit. Since $\Prob(E_N^c)\to0$ and $\sqrt N\widetilde R_N=o_p(1)$, they also transfer the CLT for $L_N$ to $\widehat C_N$.
\end{proof}

\subsection{Semiparametric efficiency}
\label{sec:supp-efficiency}

\begin{proof}[Proof of Theorem~\ref{thm:oracle}]
\textbf{Idea.} Differentiate the target along arbitrary groupwise score directions. The resulting pathwise derivative is represented by $\phi_g$ in each group. Under sampling fraction $w_g$, the product-experiment canonical gradient is therefore $\phi_g/w_g$, whose squared norm is exactly $V(w)$. The asymptotic-linear expansion from Theorem~\ref{thm:allocation-clt}, together with Le Cam's third lemma, then shows that the plug-in estimator is regular under local alternatives and attains this bound.

\emph{Step 1: pathwise derivative.}
Fix the population collection $P=(P_g)_{g\in\calG}$ and let $\mathcal P=\bigotimes_{g\in\calG}\mathcal P_g$, where each $\mathcal P_g$ is the nonparametric model on group $g$'s fixed bounded outcome space. Write $L_0^2(P_g)$ for the square-integrable, mean-zero functions under $P_g$. For a bounded $s_g\in L_0^2(P_g)$ and sufficiently small $|t|$, $dP_{g,t}=(1+t s_g)dP_g$ is a valid submodel with score $s_g$; bounded mean-zero scores are dense in $L_0^2(P_g)$ \citep[Chapters~7 and~25]{vanderVaart1998asymptotic}. It therefore suffices to derive the gradient first for bounded scores; because all Bellman contributions and the resulting $\phi_g$ are bounded, the derivative extends continuously to the $L_0^2(P_g)$ closure.

Write the affine group contribution as $\mathcal T_g(W;U)=a_g(W)+B_g(W)U$. Along simultaneous differentiable-in-quadratic-mean paths with scores $s_g\in L_0^2(P_g)$, differentiating \eqref{eq:mean-target} at $t=0$ gives
\[
 \dot U_s
 =\sum_g\E_{P_g}\!\left[\mathcal T_g(W_g;U^*)s_g(W_g)\right]
   +M\dot U_s.
\]
Because $\E_{P_g}s_g=0$, the expectation in brackets equals
$\E_{P_g}[\Xi_g(W_g)s_g(W_g)]$. Hence
\begin{equation}
 A\dot U_s=v_s,
 \qquad
 v_s:=\sum_g\E_{P_g}[\Xi_g(W_g)s_g(W_g)].
 \label{eq:path-derivative}
\end{equation}
The positive margin fixes the root VaR index on a neighborhood of the population law, so the CVaR readout is locally affine. Therefore
\begin{align}
 \dot C_s
 &=-\frac1\alpha e_{x_0}^\top\dot U_s
 =-\frac1\alpha r^\top v_s\nonumber\\
 &=\sum_g\E_{P_g}[\phi_g(W_g)s_g(W_g)].
 \label{eq:target-derivative}
\end{align}
Thus $C_{\alpha,K}$ is pathwise differentiable at the population law with groupwise derivative representers $\phi_g$.

\emph{Step 2: canonical gradient under the sampling design.}
For the deterministic allocation, applying local asymptotic normality groupwise to $P_{g,t/\sqrt N}^{\otimes n_g}$ and summing the log-likelihood ratios gives a product LAN experiment \citep[Theorem~7.2]{vanderVaart1998asymptotic} with tangent inner product
\[
 \langle s,\widetilde s\rangle_w
 :=\sum_gw_g\E_{P_g}[s_g\widetilde s_g].
\]
By \eqref{eq:target-derivative}, its Riesz representer is
$\psi_{w,g}=\phi_g/w_g$, because
$\langle\psi_w,s\rangle_w=\sum_g\E_{P_g}[\phi_gs_g]=\dot C_s$.
Its squared norm is
\begin{equation}
 \|\psi_w\|_w^2
 =\sum_gw_g\E_{P_g}\!\left[\left(\frac{\phi_g}{w_g}\right)^2\right]
 =\sum_g\frac{\sigma_g^2}{w_g}=V(w).
 \label{eq:efficiency-norm}
\end{equation}

\emph{Step 3: regularity and attainment under local alternatives.}
Theorem~\ref{thm:allocation-clt} gives the baseline asymptotic-linear representation
\[
 \sqrt N\{\widehat C_N-C_{\alpha,K}(P)\}
 =\sum_g\frac{\sqrt N}{n_g}\sum_{i=1}^{n_g}\phi_g(W_{g,i})+o_P(1)
 =\frac1{\sqrt N}\sum_g\sum_{i=1}^{n_g}\psi_{w,g}(W_{g,i})+o_P(1).
\]
For the last relation, write $N/n_g=w_g^{-1}+o(1)$. Since, for each fixed group,
$N^{-1/2}\sum_{i=1}^{n_g}\phi_g(W_{g,i})=O_P(1)$, replacing $N/n_g$ by $1/w_g$ changes the finite sum by only $o_P(1)$. Fix a DQM direction $s=(s_g)$ and a scalar $t$. Write $P_{t/\sqrt N}:=(P_{g,t/\sqrt N})_{g\in\calG}$ for the local collection of group laws and
$P_{N,t}:=\bigotimes_g P_{g,t/\sqrt N}^{\otimes n_g}$ for the corresponding sample law. LAN implies that $P_{N,t}$ is contiguous to the baseline sample law, so the displayed $o_P(1)$ term is also $o_{P_{N,t}}(1)$.

By Step~1,
\[
 \sqrt N\{C_{\alpha,K}(P_{t/\sqrt N})-C_{\alpha,K}(P)\}
 \longrightarrow t\dot C_s.
\]
Under the baseline law, the joint CLT for the influence sum and the LAN central sequence has covariance
$\langle\psi_w,s\rangle_w=\dot C_s$. Le Cam's third lemma \citep[Chapter~6]{vanderVaart1998asymptotic} therefore gives, under $P_{N,t}$,
\[
 \frac1{\sqrt N}\sum_g\sum_{i=1}^{n_g}\psi_{w,g}(W_{g,i})
 \Rightarrow \mathcal N(t\dot C_s,V(w)).
\]
Subtracting the local target shift yields
\begin{equation}
 \sqrt N\{\widehat C_N-C_{\alpha,K}(P_{t/\sqrt N})\}
 \Rightarrow \mathcal N(0,V(w))
 \qquad\text{under }P_{N,t}.
 \label{eq:efficiency-regularity}
\end{equation}
This limit is the same for every fixed $t$ and DQM direction, which is the required regularity. Hence the plug-in estimator attains the canonical-gradient variance $V(w)$.

\emph{Step 4: lower bound.}
The product tangent space is linear and hence a convex cone. The convolution theorem therefore implies that the limit law of any regular estimator is the convolution of $\mathcal N(0,V(w))$ with an independent remainder; in particular, whenever its second moment is finite, its asymptotic variance is at least $V(w)$ \citep[Theorem~25.20]{vanderVaart1998asymptotic}. The local asymptotic minimax theorem gives the corresponding lower bound $V(w)$ for local squared-error risk \citep[Theorem~25.21]{vanderVaart1998asymptotic}. The plug-in estimator has the Gaussian limit in \eqref{eq:efficiency-regularity}, so it attains both bounds. Thus it is semiparametrically efficient for the fixed design.
\end{proof}

Section~\ref{sec:supp-generic-allocation} minimizes this bound, including zero-scale groups.

\subsection{Untied factorization for the allocation ablations}
\label{sec:supp-untied-factorization}

The ablations in Figure~\ref{fig:controlled-mrp} separate propagation from local variability. When one independent law feeds each layer--state row $b=(h,s)$, its centered innovation is $\zeta_b(W)=T_b(W;U^*)-U_b^*$. Thus
\[
 \phi_b(W)=-\alpha^{-1}r_b^\top\zeta_b(W).
\]
The adjoint is nonnegative because
$r^\top=e_{x_0}^\top\sum_{j=0}^{H-1}M^j$ and $M\ge0$.
Set $d_b=\mathbf1^\top r_b$. If $d_b>0$, normalize its threshold weights as $\rho_b=r_b/d_b$ and define
\[
 \tau_b^2=\alpha^{-2}\var_{P_b}(\rho_b^\top\zeta_b(W)),
 \qquad \sigma_b=d_b\tau_b.
\]
If $d_b=0$, then $r_b=0$ and $\sigma_b=0$; set $\tau_b=0$.
This is the factorization used by the reachability-only and local-scale-only ablations. The quantity $d_b$ is adjoint mass: it propagates visitation together with the remaining shortfall threshold, so it need not equal ordinary state occupancy. The factor $\tau_b$ measures variability after averaging over those threshold weights. The oracle uses their product.

\subsection{Controlled separation with fixed visitation, moments, and root law}
\label{sec:supp-controlled-separation}

Proposition~\ref{prop:controlled-separation} isolates tail information missing from visitation and reward moments while preserving the root law. Section~\ref{sec:supp-controlled-separation-runs} tests whether a charged pilot learns this population signal.

\begin{proof}[Proof of Proposition~\ref{prop:controlled-separation}]
Consider one initial state, $H=1$, $G\ge2$ actions with probabilities $1/G$,
and a fixed terminal next state. Each action's reward law is independently queryable. Set $\alpha=1/10$ and use the grid
$\{0,1/3,5/9,2/3,1\}$, which contains every reward below and makes projection exact. Define two reward laws:
\begin{equation}
 P^{\rm A}:\quad \Prob(R=0)=\frac1{10},
 \quad
 \Prob(R=5/9)=\frac9{10};
 \qquad
 P^{\rm B}:
 \quad 
 \Prob(R=1/3)=\Prob(R=2/3)=\frac12.
 \label{eq:separation-laws}
\end{equation}
Both have $\E R=1/2$ and $\E R^2=5/18$, hence
$\var(R)=1/36$. These are properties of the construction; estimation still uses the unrestricted group-law model. Assign $P^{\rm A}$ to group 1 and $P^{\rm B}$ to every other group; denote these separated laws by $P_g^{\rm sep}$.

The root law is $\overline P=G^{-1}P^{\rm A}+(1-G^{-1})P^{\rm B}$. Its mass strictly below $1/3$ is $1/(10G)$ and its cumulative mass at $1/3$ is
$1/2-2/(5G)$. Thus, for every $G\ge2$,
\begin{equation}
 q_\alpha=\frac13,
 \qquad
 m_\alpha=\frac1{10}\left(1-\frac1G\right)>0,\qquad
 C_{\alpha,K}=\frac{G-1}{3G}.
 \label{eq:separation-target}
\end{equation}
Let $L(R)=(1/3-R)_+$. The root is the known uniform mixture of the group laws, so its group influence is
\begin{equation}
 \phi_g(R)=-\frac1{G\alpha}\{L(R)-\E_{P_g}L(R)\}.
 \label{eq:separation-influence}
\end{equation}
Under $P^{\rm A}$, $L$ is $1/3$ with probability $1/10$ and zero otherwise, giving $\var(L)=1/100$. Under $P^{\rm B}$, $L$ is identically zero. Consequently $\sigma_1=1/G$ and $\sigma_g=0$ for $g>1$. The occupancy design is uniform. The group influence for estimating the mean is $(R-1/2)/G$, with standard deviation $1/(6G)$ in every group, so the mean-optimal design is also uniform. Evaluating CVaR variance under either design gives $V=G\sum_g\sigma_g^2=1/G$, whereas $V^*=(\sum_g\sigma_g)^2=1/G^2$.
The tail oracle is a boundary design; its value is approached by positive designs as the floor vanishes.

For the fixed-root-law family, set
\begin{equation}
 P_g(t)=(1-t)\overline P+tP_g^{\rm sep},\qquad 0\le t\le1.
 \label{eq:separation-family}
\end{equation}
Every component being mixed has the same first two reward moments, so each $P_g(t)$ retains mean $1/2$ and variance $1/36$. Moreover, $G^{-1}\sum_g P_g(t)=\overline P$ for every $t$: the full root return law, CVaR, and quantile margin in Equation~\ref{eq:separation-target} stay fixed. Writing $\theta_g(t)=P_g(t)\{R=0\}$ gives
\begin{equation}
 \theta_1(t)=\frac{1+(G-1)t}{10G},\qquad
 \theta_g(t)=\frac{1-t}{10G}\ (g>1),\qquad
 \sigma_g(t)=\frac{10}{3G}\sqrt{\theta_g(t)(1-\theta_g(t))}.
 \label{eq:separation-scales}
\end{equation}
For uniform occupancy, the variance ratio is explicitly
\[
 \frac{V_{\rm occupancy}(t)}{V^*(t)}
 =\frac{G\sum_g\sigma_g(t)^2}{\left(\sum_g\sigma_g(t)\right)^2}.
\]
Cauchy--Schwarz gives the lower bound $1$, while nonnegativity gives $\sum_g\sigma_g(t)^2\le(\sum_g\sigma_g(t))^2$ and hence the upper bound $G$. At $t=0$ all scales are equal and positive, so the ratio is one. At $t=1$ it is $G$, as above. For all intermediate $t$ the ratio is continuous. It therefore attains every value between the two endpoints. Every scale is positive when $t<1$, so ratios arbitrarily close to $G$ also occur with all groups influential. This proves the family without a change in visitation, conditional moments, or root return law.
\end{proof}

\textbf{Rollouts and the ideal occupancy anchor.}
In this one-step example one rollout and one conditional observation each cost one transition query. For complete rollouts, the shortfall indicator has probability $1/(10G)$ under the fixed root law. The empirical CVaR therefore
has leading variance
\begin{equation}
 V_{\rm rollout}
 =\frac{(1/3)^2}{\alpha^2}\frac1{10G}
         \left(1-\frac1{10G}\right)
 =\frac{10G-1}{9G^2},
 \label{eq:separation-rollout}
\end{equation}
which does not change with $t$. At $t=1$, the equal mixture of the population tail oracle and uniform occupancy gives group 1 weight
$(G+1)/(2G)$ and hence $V_{\rm ideal\ anchor}=2/[G(G+1)]$.
For $G=10$, the variance constants for occupancy, rollouts, the tail oracle, and this ideal anchor are respectively $.1$, $.11$, $.01$, and $1/55$. These are leading-variance constants at the separated endpoint, before pilot cost, exploration, and rounding. For intermediate $t$, use Equation~\ref{eq:separation-scales}; the occupancy variance itself changes with $t$ even though the root law is fixed.

\textbf{From population allocation to pilot learning.}
At $t=1$, $m$ independent pilot observations from group 1 miss its zero reward with probability $(9/10)^m$. This is a missed-outcome probability, not the probability of \tis{} failure or of a wrong pilot quantile. Section~\ref{sec:supp-controlled-separation-runs} reports learned performance with all queries charged and the prescribed uniform fallback.

\section{Learned Allocation}
\label{sec:supp-learning}
\begin{algorithm}[htbp]
\small
\caption{Tail-Influence Sampling (\tis{}); pilot cost is included in $N$}
\label{alg:tis}
\begin{algorithmic}[1]
\REQUIRE query groups $\calG$, budget $N$, grid $\calZ$, tail level $\alpha$;
pilot size $m_N\ge2$, floor $0<\lambda_N<1$, $N-Gm_N\ge2G$
\STATE Draw $m_N$ independent pilot samples per group; fit empirical conditional laws.
\STATE Compute the pilot root quantile, shortfalls, and adjoint via Equations~\ref{eq:cat-dp} and~\ref{eq:affine-system}.
\STATE Score each pilot draw and compute $\widehat\sigma_g$ by Equation~\ref{eq:pilot-scales-main}.
\STATE Form $\widehat w$ by Equation~\ref{eq:floored-design}; if all scales vanish, use $\widehat w_g=1/G$.
\STATE Set $n_g=2+\operatorname{LRM}_g(N-Gm_N-2G,\widehat w)$.
\STATE Draw $n_g$ fresh main samples per group; discard the pilot.
\STATE Evaluate Equation~\ref{eq:cat-dp} with the main empirical laws and return the root CVaR $\widehat C_N^{\tis}$.
\end{algorithmic}
\end{algorithm}
Algorithm~\ref{alg:tis} learns allocation scales using a pilot, as in adaptive stratified sampling \citep{etore2010adaptive,carpentier2015adaptive}, but also estimates the Bellman model and quantile. We control these errors to prove oracle adaptation, then establish the anchor's separate safeguard. Only fresh main samples form the final estimate; $G=|\calG|$.

\subsection{Pilot-scale consistency and lower-tail control}
\label{sec:supp-pilot}
We use the centered form of Equation~\ref{eq:pilot-scales-main}. For $m\ge2$ pilot draws per group, set $\widehat x_0=(H,s_0,\widehat q_\alpha^{(0)})$, the root coordinate at the pilot quantile. Define
\begin{align}
 \widehat r^{(0)\top}
 &:=e_{\widehat x_0}^\top(I-\widehat M^{(0)})^{-1},\nonumber\\
 \widehat\Xi_{g,i}^{(0)}
 &=\mathcal T_g(W_{g,i};\widehat U^{(0)})
 -\frac1m\sum_{j=1}^{m}
   \mathcal T_g(W_{g,j};\widehat U^{(0)}),\nonumber\\
 \widehat\phi_{g,i}^{(0)}
 &=-\alpha^{-1}\widehat r^{(0)\top}
   \widehat\Xi_{g,i}^{(0)},\nonumber\\
 \widehat\sigma_g^2
 &=\frac1m\sum_{i=1}^{m}(\widehat\phi_{g,i}^{(0)})^2.
 \label{eq:pilot-scale-estimator}
\end{align}
The four lines give, respectively, weights that propagate local changes to the root shortfall, each draw's deviation from its group's mean Bellman contribution, its estimated CVaR influence, and the empirical variance of these influences.
Linearity gives $\widehat\phi_{g,i}^{(0)}=\widehat d_{g,i}-\bar d_g$, recovering Equation~\ref{eq:pilot-scales-main}. Take $m=m_N$ for Theorem~\ref{thm:adaptive}. Shared stage effects are summed before taking the variance, preserving their covariance.

\begin{lemma}[\textbf{Pilot-scale control}]
\label{lem:pilot-concentration}
Under the fixed-dimensional bounded Bellman model and positive quantile margin, there is a deterministic $B_\sigma<\infty$ such that $0\le\widehat\sigma_g\le B_\sigma$ for every group and pilot dataset. With each pilot formed from the first $m$ observations of an i.i.d.
stream in each group,
\begin{equation}
 \widehat\sigma_g\longrightarrow\sigma_g
 \quad\text{a.s. as }m\to\infty.
 \label{eq:sigma-consistency}
\end{equation}
For each group with $\sigma_g>0$, constants $c_g,C_g>0$ exist such that
\begin{equation}
 \Prob(\widehat\sigma_g<\sigma_g/2)\le C_g e^{-c_gm}.
 \label{eq:pilot-lower-tail}
\end{equation}
\end{lemma}

\begin{proof}
\textbf{Idea.} In fixed dimension, the pilot scale is a continuous function of finitely many bounded empirical moments as long as the VaR cell is correct. Laws of large numbers give consistency, while Hoeffding bounds plus the positive margin control the rare event that a genuinely positive scale is badly underestimated.

\emph{Step 1: finite empirical-moment representation.}
Write the affine group contribution as $\mathcal T_g(W;U)=a_g(W)+B_g(W)U$. Let $Y_g(W)$ collect the finitely many entries of $a_g(W)$ and $B_g(W)$ and their pairwise products. The empirical first moments of $a_g$ and $B_g$ determine $\widehat{\mathbf b}^{(0)}$ and $\widehat M^{(0)}$, hence $\widehat U^{(0)}$ and $\widehat r^{(0)}$. Expanding the empirical variance in \eqref{eq:pilot-scale-estimator} then introduces only empirical averages of pairwise products of entries of $a_g$ and $B_g$; no higher empirical moments are needed. These features are bounded, and all quantities in \eqref{eq:pilot-scale-estimator} are therefore functions of their groupwise empirical means $\widehat\mu$; let $\mu$ denote the corresponding vector of population means. On the correct VaR cell,
\[
 \widehat U^{(0)}
 =\sum_{j=0}^{H-1}(\widehat M^{(0)})^j\widehat{\mathbf b}^{(0)},\qquad
 \widehat r^{(0)\top}
 =e_{\widehat x_0}^\top\sum_{j=0}^{H-1}(\widehat M^{(0)})^j,
\]
so $\widehat\sigma_g^2=F_g(\widehat\mu)$ for a polynomial $F_g$ with $F_g(\mu)=\sigma_g^2$.

\emph{Step 2: consistency.}
The strong law gives $\widehat\mu\to\mu$ almost surely. The positive margin and the CDF bound used in \eqref{eq:index-concentration} make the pilot VaR cell eventually correct almost surely. Continuity of $F_g$ proves \eqref{eq:sigma-consistency}. Bounded features, stop-loss coordinates, and the deterministic resolvent bound also give the uniform constant $B_\sigma$.

\emph{Step 3: lower-tail protection for active groups.}
If $\sigma_g>0$, $F_g$ is Lipschitz on a compact neighborhood of $\mu$.
Choose that neighborhood so $|F_g(\widehat\mu)-\sigma_g^2|\le3\sigma_g^2/4$. Hoeffding's inequality \citep{hoeffding1963probability} and a finite union bound show that leaving this neighborhood has probability at most $Ce^{-cm}$. The same bound holds for a wrong pilot VaR cell by \eqref{eq:index-concentration}. Off these two events, $\widehat\sigma_g^2\ge\sigma_g^2/4$, proving \eqref{eq:pilot-lower-tail}.
\end{proof}
The constants in Equation~\ref{eq:pilot-lower-tail} depend on the fixed population model, including its positive influence scales and quantile margin. The bound supports the asymptotic MSE proof; it does not by itself give a pilot size computable from the data or a bound on the final estimator's finite-sample MSE.

\textbf{Matrix-free influence computation.}
The displayed matrices define the linear operator. A matrix-free implementation avoids forming $A^{-1}$ or the $D\times D$ covariance matrices. Compute the stop losses in increasing layer order, propagate the root adjoint in decreasing layer order, and accumulate the scalar $r^\top\mathcal T_g(W;U)$ for each sample before estimating its variance. In the shared state--action model, a straightforward sample-based implementation uses at most $O(HKN\log K+H|\mathcal S|K)$ arithmetic operations for these passes with binary search on a nonuniform grid, and $O(H|\mathcal S|K+N+G)$ storage. Interpolation indices can be reused. These are upper bounds for the described construction. Measured runtimes depend on the actual grid and implementation and are reported with the experimental artifacts.

\subsection{Finite-pilot design stability}
\label{sec:supp-finite-pilot}
Near an interior oracle, small score errors have a quadratic variance cost. Severe underestimation requires the asymptotic controls in Section~\ref{sec:supp-adaptation}.
\begin{proposition}[Local design stability]
\label{prop:finite-pilot}
If every retained $\sigma_g>0$ and $\epsilon=\max_g|\widehat\sigma_g-\sigma_g|$, then for sufficiently small $\epsilon$ and floor $\lambda$,
\begin{equation}
 0\le V \big(\widehat w \big) - V^*\le C \big(\epsilon^2 +\lambda^2 \big)
 \label{eq:finite-pilot}
\end{equation}
for a finite problem-dependent constant $C$.
\end{proposition}

\begin{proof}
\textbf{Idea.} At an interior oracle, allocation error has a quadratic variance cost. Set $S_\sigma=\sum_g\sigma_g$, $p_g=\sigma_g/S_\sigma$, and $p_{\min}=\min_g p_g>0$. For every positive design $w$,
\begin{equation}
 V(w)-V^*=S_\sigma^2\sum_g\frac{(p_g-w_g)^2}{w_g}.
 \label{eq:design-regret-identity}
\end{equation}
To verify the identity, expand the square and use
$\sum_g p_g=\sum_g w_g=1$. If $G\epsilon\le S_\sigma/2$, normalizing the estimated scales gives
\[
 \left\|\frac{\widehat\sigma}{\sum_g\widehat\sigma_g}-p\right\|_\infty
 \le \frac{2(G+1)\epsilon}{S_\sigma}.
\]
Adding the uniform floor therefore gives
$\|\widehat w-p\|_\infty\le2(G+1)\epsilon/S_\sigma+\lambda$.
For sufficiently small $\epsilon,\lambda$, every denominator
$\widehat w_g$ in Equation~\ref{eq:design-regret-identity} is at least $p_{\min}/2$. Substitution and $(a+b)^2\le2a^2+2b^2$ prove the claim. The lower bound on the weights is local; it gives no protection on a pilot event with severe scale underestimation.
\end{proof}

\subsection{Oracle adaptation}
\label{sec:supp-adaptation}

\begin{proof}[Proof of Theorem~\ref{thm:adaptive}]
\textbf{Idea.} The pilot must make the positive-scale weights converge to the oracle design and make severe underestimation sufficiently rare for second moments. The exploration floor separately guarantees enough main samples in every group to control the nonlinear fixed-point remainder and the VaR cell. Conditional on the pilot, the remaining problem is a deterministic triangular-array CLT.

\emph{Step 1: learned weights, rounding, and a minimum main count.}
Write $S_\sigma:=\sum_g\sigma_g>0$ and $w_g^*:=\sigma_g/S_\sigma$. Let $\widehat C_N^{\tis}$ denote the main-sample estimator produced by Algorithm~\ref{alg:tis}. Define
\[
 \widetilde p_g:=
 \begin{cases}
  \widehat\sigma_g/\sum_j\widehat\sigma_j,
       &\sum_j\widehat\sigma_j>0,\\
  1/G,&\text{otherwise},
 \end{cases}
 \qquad
 \widehat w_g:=(1-\lambda_N)\widetilde p_g+\lambda_N/G.
\]
The second branch also gives $\widehat w_g=1/G$. Pilot consistency and $S_\sigma>0$ imply that its probability tends to zero. For pilots redrawn at each budget, the convergence needed below is in probability:
\begin{equation}
 \widehat w_g\stackrel{p}{\longrightarrow}w_g^*
 \quad\text{for every group, including }w_g^*=0.
\end{equation}

Set $J_N:=N-Gm_N-2G$ and $n_g:=2+\operatorname{LRM}_g(J_N,\widehat w)$. Largest-remainder rounding starts from $\lfloor J_N\widehat w_g\rfloor$ and assigns the leftover calls to the largest fractional remainders, with ties broken in a fixed group order. It satisfies
\[
 \sum_g\operatorname{LRM}_g(J_N,\widehat w)=J_N,\qquad
 |\operatorname{LRM}_g(J_N,\widehat w)-J_N\widehat w_g|<1.
\]
Thus the pilot and main counts sum to $N$. Since $J_N/N\to1$, uniformly in
$g$,
\begin{equation}
 \frac{n_g}{N}-\widehat w_g
 =\left(\frac{J_N}{N}-1\right)\widehat w_g+O(N^{-1})=o(1).
 \label{eq:adaptive-rounded-fractions}
\end{equation}
Moreover, $\widehat w_g\ge\lambda_N/G$ and
$\operatorname{LRM}_g(J_N,\widehat w)\ge J_N\widehat w_g-1$, so, eventually,
\begin{equation}
 n_{\min}\ge\frac{N\lambda_N}{2G}.
 \label{eq:adaptive-min-count}
\end{equation}

\emph{Step 2: conditional CLT for the leading influence term.}
Conditional on the pilot, the main samples are independent and their counts are fixed. Let
\[
 Z_N:=\sum_g\frac1{n_g}\sum_{i=1}^{n_g}\phi_g(W_{g,i})
\]
be the leading influence term. If $\sigma_g=0$, then centering and zero variance imply $\phi_g=0$ almost surely. Let $\mathcal A:=\{g:\sigma_g>0\}$. This set is nonempty because $S_\sigma>0$, and
$w_{\min}^*:=\min_{g\in\mathcal A}w_g^*>0$ because $\mathcal A$ is finite. By \eqref{eq:adaptive-rounded-fractions}, the pilot event
\[
 B_N:=\left\{\min_{g\in\mathcal A}\frac{n_g}{N}\ge\frac{w_{\min}^*}{2}\right\}
\]
satisfies $\Prob(B_N)\to1$. Conditional on a pilot in $B_N$, the summands of $\sqrt N Z_N$ are independent and centered. Because the finite collection of influences is bounded, there is a deterministic $M_3<\infty$ with $\E|\phi_g|^3\le M_3$ for every $g$, and
\[
 \sum_{g\in\mathcal A}\sum_{i=1}^{n_g}
 \E\left[\left|\frac{\sqrt N}{n_g}\phi_g(W_{g,i})\right|^3\middle| \;\textnormal{pilot}\right]
 \le M_3 N^{3/2}\sum_{g\in\mathcal A}\frac1{n_g^2}
 \le \frac{4M_3|\mathcal A|}{(w_{\min}^*)^2\sqrt N}
 \longrightarrow0.
\]
The conditional variance is
\[
 N\sum_{g\in\mathcal A}\frac{\sigma_g^2}{n_g}
 \xrightarrow{p}\sum_{g\in\mathcal A}\frac{\sigma_g^2}{w_g^*}
 =S_\sigma^2=V^*>0.
\]
Consequently, on an event whose pilot probability tends to one, the variance is bounded away from zero; dividing the preceding third-moment bound by its $3/2$ power verifies Lyapunov's condition, and hence Lindeberg's condition.
The Lindeberg--Feller theorem applied conditionally on the pilot \citep[Proposition~2.27]{vanderVaart1998asymptotic} therefore makes the conditional characteristic function of $\sqrt N Z_N$ converge in probability to that of $\mathcal N(0,V^*)$. Characteristic functions are bounded by one, so taking expectations over the pilot yields the unconditional convergence
$\sqrt N Z_N\Rightarrow\mathcal N(0,V^*)$.

\emph{Step 3: nonlinear fixed-point remainder.}
The fixed-point remainder also remains negligible. Conditional on the pilot, Lemma~\ref{lem:fixed-design-moments} gives $\E[\|R_N\|^2\mid\textnormal{pilot}]\le Cn_{\min}^{-2}$ with deterministic $C$. By \eqref{eq:adaptive-min-count},
\begin{equation}
 \sqrt N\|R_N\|
 =O_p\!\left((\sqrt N\lambda_N)^{-1}\right)=o_p(1),\qquad
 N\E\|R_N\|^2
 =O\!\left((N\lambda_N^2)^{-1}\right)=o(1).
\end{equation}
Together with the quantile-cell argument below, this proves the CLT.

\emph{Step 4: uniform integrability of the leading variance.}
For the normalized MSE, independence and centering give
\begin{equation}
 \E[N Z_N^2\mid\textnormal{pilot}]
 =N\sum_{g:\sigma_g>0}\frac{\sigma_g^2}{n_g}.
 \label{eq:adaptive-leading-mse}
\end{equation}
For a \newrev{group with $\sigma_g>0$}, let
$A_{g,N}:=\{\widehat\sigma_g\ge\sigma_g/2\}$. Since
$\sum_j\widehat\sigma_j\le GB_\sigma$ and eventually
$1-\lambda_N\ge1/2$, on $A_{g,N}$,
$\widehat w_g\ge\sigma_g/(4GB_\sigma)$. On $A_{g,N}^c$,
$\widehat w_g\ge\lambda_N/G$. Also
$n_g\ge J_N\widehat w_g$ and eventually $N/J_N\le2$, so
$N/n_g\le2/\widehat w_g$. Lemma~\ref{lem:pilot-concentration} gives
\[
 \E\!\left[\frac{N}{n_g}\ind_{A_{g,N}^c}\right]
 \le\frac{2G C_g}{\lambda_N}e^{-c_gm_N}\longrightarrow0,
\]
because $\log(1/\lambda_N)=o(m_N)$. On $A_{g,N}$, $N/n_g$ is uniformly bounded and converges in probability to $1/w_g^*$. Hence \eqref{eq:adaptive-leading-mse} converges in expectation to $V^*$.

\emph{Step 5: wrong VaR cells and conclusion.}
Finally, conditional on the pilot, \eqref{eq:index-concentration} and \eqref{eq:adaptive-min-count} bound the wrong-VaR-cell contribution to the normalized MSE by $CN\exp(-cN\lambda_Nm_\alpha^2)=o(1)$. Cauchy--Schwarz removes the cross term with the $L^2$-negligible remainder. Thus
\[
\sqrt N(\widehat C_N^{\tis}-C_{\alpha,K})
\Rightarrow\mathcal N(0,V^*),\qquad
N\E[(\widehat C_N^{\tis}-C_{\alpha,K})^2]\to V^*,
\]
which completes the proof.
\end{proof}

\subsection{\redrev{Anchored design: variance safeguard and asymptotic MSE}}
\label{sec:supp-anchor}

For the anchor in Equation~\ref{eq:anchored-score}, \redrev{we define the occupancy scores, prove the component-relative variance safeguard, and establish Corollary~\ref{cor:anchor-asymptotic}.}

Let $\widehat\mu_h(s)$ denote the state occupancy induced from the root by the fixed policy and the pilot transition estimate, with $h$ steps remaining. In the stationary shared-kernel model, define
\begin{equation}
 \widehat o_{s,a}:=\sum_{h:(h,s)\in\calB}\widehat\mu_h(s)\pi_h(a\mid s).
 \label{eq:supp-occupancy-score}
\end{equation}
More generally, when a query group feeds several Bellman rows with known mixture coefficients, $\widehat o_g$ is the sum of the corresponding pilot-model row visitation probabilities times those coefficients. In the untied policy-mixture case $b=(h,s)$, this reduces to $\widehat o_b=\widehat\mu_h(s)$.

Let
\[
 p_g^{\rm inf}=\frac{\widehat\sigma_g}{\sum_j\widehat\sigma_j},\qquad
 p_g^{\rm occ}=\frac{\widehat o_g}{\sum_j\widehat o_j},\qquad
 p_g^{\rm anc}=\tfrac12p_g^{\rm inf}+\tfrac12p_g^{\rm occ},
\]
using the uniform branch for $p^{\rm inf}$ if every estimated influence scale is zero. The occupancy denominator is positive because the retained rows include the root and $H\ge1$. Apply the same exploration floor to each component,
$w^x=(1-\lambda)p^x+\lambda\mathbf 1/G$ for $x\in\{\rm inf,occ,anc\}$.

\begin{proposition}[Component-relative safeguard]
\label{prop:anchored}
For the positive fractional designs before integer rounding,
\begin{equation}
 V(w^{\rm anc})\le2\min\{V(w^{\rm inf}),V(w^{\rm occ})\}.
 \label{eq:anchored-bound}
\end{equation}
\end{proposition}

\begin{proof}
The common floor gives
$w^{\rm anc}=(w^{\rm inf}+w^{\rm occ})/2$, so for every group
\[
 w_g^{\rm anc}\ge\tfrac12w_g^{\rm inf},\qquad
 w_g^{\rm anc}\ge\tfrac12w_g^{\rm occ}.
\]
Because $V(w)=\sum_g\sigma_g^2/w_g$ is decreasing in each coordinate separately,
\[
 V(w^{\rm anc})\le2V(w^{\rm inf}),\qquad
 V(w^{\rm anc})\le2V(w^{\rm occ}),
\]
which proves the claim.
\end{proof}
\begingroup\revcolor{red}
\textbf{Choice of mixture weight.} For occupancy weight $\beta\in(0,1)$, the same argument gives
\[
 V\big((1-\beta)w^{\rm inf}+\beta w^{\rm occ}\big)
 \le\min\left\{\frac{V(w^{\rm inf})}{1-\beta},
                  \frac{V(w^{\rm occ})}{\beta}\right\}.
\]
The worst-case factor in this bound relative to the better component is
$\max\{(1-\beta)^{-1},\beta^{-1}\}$, minimized at $\beta=1/2$; optimal finite-budget weights may differ.

This safeguard compares the two component designs before rounding; it does not bound finite-budget MSE. Corollary~\ref{cor:anchor-asymptotic} identifies the limiting MSE constant of the full anchored estimator, including pilot cost and rounding. Finite-budget behavior is assessed in the workflow experiments (Sections~\ref{sec:llm-protocol}--\ref{sec:pilot-mechanism}).

\begin{corollary}[Asymptotic efficiency cost of anchoring]
\label{cor:anchor-asymptotic}
Under the assumptions and schedules of Theorem~\ref{thm:adaptive}, let $v_g=o_g/\sum_j o_j$ be the population occupancy shares and set
$a_g^*=(w_g^*+v_g)/2$. Using Equation~\ref{eq:anchored-score} in Algorithm~\ref{alg:tis} gives
\begin{equation}
 \begin{gathered}
 \sqrt N\big(\widehat C_N^{\rm anc}-C_{\alpha,K}\big)
 \Rightarrow\mathcal N(0,V_{\rm anc}),
 \qquad V^*\le V_{\rm anc}\le2V^*.
 \\
 N\E\big[(\widehat C_N^{\rm anc}-C_{\alpha,K})^2\big]
 \to V_{\rm anc}=\sum_{g:\sigma_g>0}\frac{\sigma_g^2}{a_g^*}.
 \end{gathered}
 \label{eq:anchor-asymptotic}
\end{equation}
\end{corollary}

\begin{proof}[Proof of Corollary~\ref{cor:anchor-asymptotic}]
\textbf{Idea.} The occupancy shares are consistent, and the anchor retains at least half of every learned influence share. The proof of Theorem~\ref{thm:adaptive} therefore continues to control rare underallocation, with a changed limiting design.

\emph{Step 1: limiting shares and minimum counts.}
In fixed dimension, the pilot transition probabilities converge in probability to their population values. Finite-horizon occupancies are continuous functions of those probabilities, and their sum is positive. Thus the normalized pilot occupancy shares converge to $v$. The common floor vanishes, while Lemma~\ref{lem:pilot-concentration} gives $w^{\rm inf}\to w^*$ in probability. Consequently,
\[
 w_g^{\rm anc}\stackrel{p}{\longrightarrow}a_g^*
 =\tfrac12(w_g^*+v_g).
\]
For every active group $\sigma_g>0$, $a_g^*\ge w_g^*/2>0$. Both floored components have weights at least $\lambda_N/G$, so the anchored counts obey the same lower bound $n_{\min}\ge N\lambda_N/(2G)$ eventually as in Equation~\ref{eq:adaptive-min-count}. Rounding and the vanishing pilot fraction give $n_g^{\rm anc}/N\to a_g^*$ in probability.

\emph{Step 2: the influence term and its second moment.}
Conditional on the pilot, the main observations are independent. The bounded-influence conditional CLT used in Theorem~\ref{thm:adaptive} gives a Gaussian limit with variance $V_{\rm anc}$; zero-scale groups contribute nothing. To justify convergence of second moments, couple the hypothetical plain and anchored allocations to the same pilot. Write $J_N=N-Gm_N-2G$, and let $n_g^{\rm inf}$ be the plain count. Largest-remainder rounding gives
\[
 n_g^{\rm anc}\ge1+J_Nw_g^{\rm anc}
 \ge1+\tfrac12J_Nw_g^{\rm inf}
 \ge\tfrac12(n_g^{\rm inf}-1)
 \ge\tfrac14 n_g^{\rm inf},
\]
where $n_g^{\rm inf}\le3+J_Nw_g^{\rm inf}$ and $n_g^{\rm inf}\ge2$ were used. Hence
\[
 0\le N\sum_g\frac{\sigma_g^2}{n_g^{\rm anc}}
 \le4N\sum_g\frac{\sigma_g^2}{n_g^{\rm inf}}.
\]
To make the moment transfer explicit, set
$Y_N^{\rm inf}:=N\sum_g\sigma_g^2/n_g^{\rm inf}$ and
$Y_N^{\rm anc}:=N\sum_g\sigma_g^2/n_g^{\rm anc}$. Step~4 of Theorem~\ref{thm:adaptive} gives $Y_N^{\rm inf}\to V^*$ in probability and $\E Y_N^{\rm inf}\to V^*$. Since these variables are nonnegative, this implies $Y_N^{\rm inf}\to V^*$ in $L^1$ and hence uniform integrability. The domination $0\le Y_N^{\rm anc}\le4Y_N^{\rm inf}$ therefore makes $\{Y_N^{\rm anc}\}$ uniformly integrable as well. Because $Y_N^{\rm anc}\to V_{\rm anc}$ in probability by Step~1, uniform integrability yields $\E Y_N^{\rm anc}\to V_{\rm anc}$. This proves the normalized-MSE limit for the leading influence term.

\emph{Step 3: remainder, quantile, and efficiency cost.}
The common minimum-count bound gives the same vanishing normalized second moment of the Bellman remainder and the same exponentially small wrong-quantile contribution as in Theorem~\ref{thm:adaptive}. The cross term vanishes by Cauchy--Schwarz. The CLT and MSE limit thus hold for the full CVaR estimator. Finally, $a_g^*\ge w_g^*/2$ on active groups implies
\[
 V_{\rm anc}\le2\sum_{g:\sigma_g>0}\frac{\sigma_g^2}{w_g^*}=2V^*.
\]
Cauchy--Schwarz gives $V_{\rm anc}\ge V^*$ for any probability allocation, including allocations with zero weights only on zero-influence groups.
\end{proof}

\subsection{When learning an allocation repays its pilot}
\label{sec:pilot-payoff}
The variance identity in Equation~\ref{eq:design-regret-identity} separates allocation opportunity from learning error. It remains valid when some influences vanish, provided $S_\sigma=\sum_g\sigma_g>0$ and the evaluated design is positive. Let $n=N-Gm$, $\rho=Gm/N$, retain the oracle shares $w_g^*=\sigma_g/S_\sigma$ from Equation~\ref{eq:oracle-value}, and let $\widetilde w_g=n_g/n$ be the actual main-sample fractions after rounding. Define
\[
 D(w^*\Vert w):=\sum_g\frac{(w_g^*-w_g)^2}{w_g}.
\]
Because both $w^*$ and $w$ sum to one,
\[
 D(w^*\Vert w)=\sum_g\frac{(w_g^*)^2}{w_g}-1.
\]
Conditional on the pilot, the centered leading influence term $L_N=\sum_g n_g^{-1}\sum_i\phi_g(W_{g,i})$ therefore satisfies
\[
 N\E[L_N^2\mid\textnormal{pilot}]
 =N\sum_g\frac{\sigma_g^2}{n_g}
 =\frac{N}{n}V^*\sum_g\frac{(w_g^*)^2}{\widetilde w_g}
 =\frac{V^*}{1-\rho}\{1+D(w^*\Vert\widetilde w)\}.
\]
Averaging over the pilot gives the exact identity
\begin{equation}
 N\E[L_N^2]
 =\frac{V^*}{1-\rho}\left\{1+\E D(w^*\Vert\widetilde w)\right\}.
 \label{eq:pilot-payoff}
\end{equation}
Here the expectation on the right is over pilots. For a deterministic baseline using all $N$ queries without a pilot and positive actual query fractions $v$, write $A_v=V(v)/V^*$. Its leading variance is $V(v)/N$. Learning improves on this benchmark at the level of the influence term exactly when
\begin{equation}
 1+\E D(w^*\Vert\widetilde w)<(1-\rho)A_v.
 \label{eq:pilot-payoff-condition}
\end{equation}
The three quantities have distinct roles: $A_v$ is the available allocation advantage, $D$ penalizes inaccurate shares, especially underallocation, and $\rho$ charges the discarded pilot. For two learned methods with the same pilot size, the common factor $1/(1-\rho)$ cancels, so their comparison depends on their expected allocation penalties. These are identities for the linearized error, not finite-budget guarantees for the nonlinear CVaR estimator. Population influences are needed to evaluate them, so their use in the experiments is diagnostic rather than an operational rule for choosing a method.
\endgroup

\section{Approximation and Experimental Protocols}
\label{sec:supp-evidence}
This section closes two gaps left by the asymptotic allocation theory. Appendix~\ref{sec:supp-representation} controls error from the categorical grid; the remaining subsections give the protocols and evidence supporting \textbf{Q1--Q5}, including the regimes where tail targeting helps, where it does not, and why small pilots can fail. Each primary protocol defines its query and charged budget.

\subsection{Representation error}
\label{sec:supp-representation}
A fixed grid introduces error even with exact conditional laws. The following bound justifies the separation of approximation and sampling error at the end of Section~\ref{sec:problem}. Let $Z_h^*(s)\sim\eta_h^*(s)$ denote the projected $h$-step return and $G_h(s)$ its true-return counterpart. For every state $s$,
\begin{equation}
 |\cvar_\alpha(\eta_H^*(s))-\cvar_\alpha(G_H(s))|\le H\Delta.
 \label{eq:representation-cvar-bound}
\end{equation}

\begin{proof}
Realize each categorical projection as randomized rounding to adjacent grid atoms \citep{rowland2018analysis}: for $y\in[z_j,z_{j+1}]$, set $\widetilde y=z_j$ with probability $(z_{j+1}-y)/(z_{j+1}-z_j)$ and $\widetilde y=z_{j+1}$ otherwise; for $y\ge H$, set $\widetilde y=H$. Then $\mathcal L(\widetilde y)=\Pi_C\delta_y$ and $|\widetilde y-y|\le\Delta$ whenever $y\le H$ (inputs below zero do not arise here). Couple the projected and true return recursions using the same actions, rewards, next states, and these rounding variables. Conditional on a coupled next state, use the inductive coupling for the two continuation returns. If their difference is at most $(h-1)\Delta$, adding the same reward preserves that difference and adjacent-grid rounding adds at most $\Delta$. If the projected pre-rounding value exceeds $H$, clipping it to $H$ cannot increase its distance from the true $h$-step return, which lies in $[0,h]\subseteq[0,H]$. Starting from equal zero-step returns, induction yields
\begin{equation}
 |Z_h^*(s)-G_h(s)|\le h\Delta\qquad\text{almost surely}.
\end{equation}
If $|X-Y|\le c$ almost surely, then $F_X(x-c)\le F_Y(x)\le F_X(x+c)$ for every $x$, which implies $|F_X^{-1}(u)-F_Y^{-1}(u)|\le c$ for $u\in(0,1)$. Integrating over $u\in(0,\alpha)$ and dividing by $\alpha$ gives Equation~\ref{eq:representation-cvar-bound}: the integration interval's length cancels the factor $1/\alpha$.
\end{proof}

\subsection{Experimental Evidence and Protocols}
\label{sec:supp-experiments}
The protocols below support \textbf{Q1}--\textbf{Q3} in Section~\ref{sec:experiments}; Section~\ref{sec:finqa-protocol} gives the held-out FinQA protocol for \textbf{Q4}. MSE is the average squared error over independent replications, with Monte Carlo SE equal to the sample standard deviation of squared errors divided by the square root of the replication count; bars show $1.96$ SEs. Rollouts use $\lfloor N/H\rfloor$ full trajectories (leaving fewer than $H$ transition slots unused), whereas conditional allocations exhaust $N$. Bold marks sample-only point minima, including displayed ties, without a superiority claim; population references are shaded. Uncertainty is conditional on the fixed tasks or panels.

\textbf{Method key and common estimator.} All conditional-query methods use the categorical CVaR plug-in. Fixed-score designs normalize scores, add the floor in Equation~\ref{eq:floored-design}, and round to exhaust the budget; learned scores pay for and discard a pilot, whereas population references use exact laws without a pilot. \textbf{Uniform} assigns equal shares, and \textbf{learned occupancy} (\textbf{occup.}/\textbf{Occ.}) uses pilot-model visit counts $\widehat o_g$ from Equation~\ref{eq:supp-occupancy-score}. \textbf{Learned mean} uses the pilot-estimated standard deviation of the mean-return influence; for a shared kernel,
\[
 \psi_{s,a}(W)=\sum_{h=1}^H\mu_h(s)\pi_h(a\mid s)
 \{R+v_{h-1}(S')-\E_{P_{s,a}}[R+v_{h-1}(S')]\},
\]
where $\mu_h(s)$ is visitation with $h$ steps left and $v_h(s)=\E[G_h(s)]$; an untied block uses $\mu_h(s)\operatorname{sd}_{P_b}(R+v_{h-1}(S'))$. \textbf{Occupancy (population)} and \textbf{mean influence (population)} use the corresponding exact scores, and \textbf{complete rollout} (\textbf{rollout}/\textbf{Roll.}) takes empirical CVaR of full fixed-policy returns \citep{thomas2019cvar}.

\textbf{Plain/shared \tis{}} uses $\widehat\sigma_g$; \textbf{anchored \tis{}} (\textbf{anchored}/\textbf{Anch.}) averages the tail and occupancy designs. \textbf{Oracle+floor} uses exact $\sigma_g$ and is a population allocation reference, not a finite-budget MSE lower bound. \textbf{Reachability only} and \textbf{local scale only} replace $\sigma_b=d_b\tau_b$ by exact $d_b$ or $\tau_b$ (Section~\ref{sec:supp-untied-factorization}), separating threshold-dependent adjoint mass from local shortfall variability; $d_b$ need not equal occupancy. \textbf{\tis{} no-cov} removes cross-layer covariance from the allocation score while retaining pooled estimation, with \textbf{oracle no-cov} its population analogue; \textbf{untied \tis{}} fits separate layer kernels under the same total budget. \textbf{MC-UCB (frozen)} sequentially allocates using uncertainty in pilot-frozen influence scores (Section~\ref{sec:supp-cliffwalking}), and \textbf{pilot answer entropy} uses Shannon entropy of pilot answer frequencies, summed over confidence labels.

\subsubsection{Controlled stochastic Markov reward process}
\label{sec:supp-controlled}
This auxiliary multi-step check supports the mechanism behind \textbf{Q1--Q2}: in the untied setting, the oracle score combines propagation to the root and local shortfall variability. The $H=5$ MRP has five states per layer, rewards in $\{0,1/2,1\}$, grid spacing $\Delta=1/2$, and stochastic rewards and transitions in every block, giving $G=21$ independently sampled layer--state blocks. At $\alpha=.1$, exact enumeration gives CVaR $.9086716$, margin $.0106854$, and $V_{\rm unif}/V^*=12.98416/7.53437=1.72332$.

We use total budgets $12{,}500$--$100{,}000$, the common $m_N=\lceil4N^{2/3}/G\rceil$ discarded pilot and $\lambda_N=N^{-1/4}$ floor, and 400 replications. At $N=100{,}000$, \tis{} spends 8,631 pilot queries and reaches $.615$ of uniform MSE; anchored \tis{}, learned mean, and learned occupancy are close, while rollouts are worse than uniform (Figure~\ref{fig:controlled-mrp}).

\begin{figure}[t]
\centering
\includegraphics[width=0.62\textwidth]{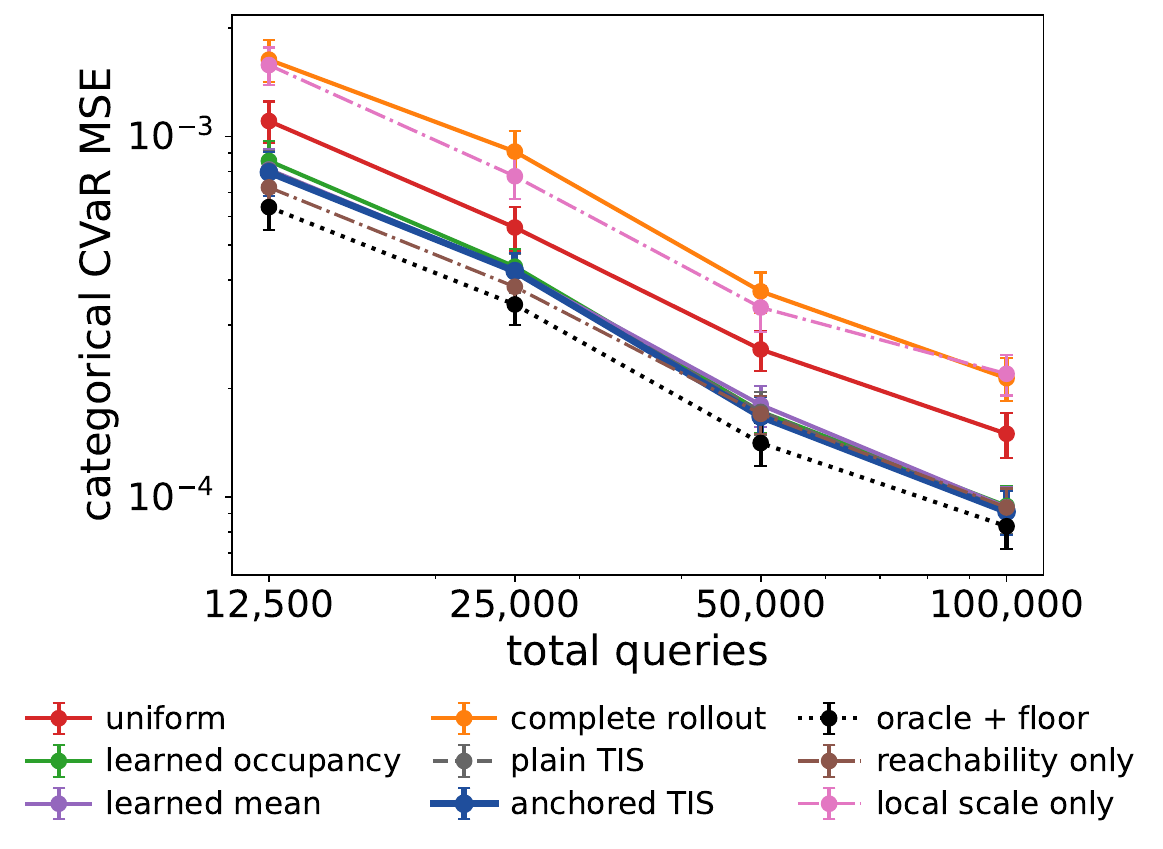}
\caption{Controlled $H=5$ MRP (21 untied blocks): CVaR MSE versus total queries for learned methods and population references; bars are $1.96$ Monte Carlo SEs.}
\label{fig:controlled-mrp}
\end{figure}

Across 50 perturbed stochastic instances (Dirichlet seeds 100--149), the median $V_{\rm unif}/V^*$ at $\alpha=.05,.1,.2$ is $1.95344,1.92311,1.82727$; ratios span $1.34558$--$3.56479$ over all instance--risk pairs, with positive margins throughout. This is a population-level breadth check.

\subsubsection{Learned-budget runs for the controlled separation}
\label{sec:supp-controlled-separation-runs}

This experiment directly supports \textbf{Q1--Q2}: it asks whether a charged pilot learns the population separation of Equation~\ref{eq:separation-family}. We use $G=10$, $\alpha=.1$, and $t\in\{0,\frac12,1\}$ under a protocol frozen before simulation. Analytically, CVaR is $.3$, the margin is $.09$, occupancy-to-oracle variance ratios are $1$, $1.29971$, and $10$, and conditional moments and the root law agree across $t$ to $7\times10^{-17}$.

Each cell uses $1{,}000$ replications at 25, 50, 100, 200, and 400 queries per kernel; pilot size, floor, rounding, and fallback follow the common protocol, and rollouts use the same transition budget on an independent stream. Because the one-step policy fixes occupancy at $1/G$, uniform is also the no-pilot occupancy design; \textbf{occupancy + pilot} discards the common pilot before using that same allocation, isolating pilot cost. At 100/400 queries per kernel the pilot consumes 40/101 draws per group ($40\%/25.25\%$); learned mean, \tis{}, and the anchor pay the same cost, whereas uniform, rollouts, and oracle+floor do not. Rollouts randomize actions while uniform fixes equal counts. At $t=0$, each group's zero-reward probability is $.01$, so the pilots miss it with probability $.669/.362$ at 100/400 queries, explaining why learning can hurt when uniform is optimal. At $t=1$, the median \tis{} share of the informative kernel rises from $.62$ to $.88$ across budgets (oracle share $1$); its 10th percentile is $.10$--$.75$ at 25--50 queries, exposing the underallocation that anchoring mitigates. Table~\ref{tab:controlled-sep-uncertainty} reports absolute MSEs and Monte Carlo SEs for the main-table cells.

\begin{table}[t]
\centering
\small
\setlength{\tabcolsep}{3pt}
\caption{Absolute MSE $\times10^5$ (Monte Carlo SE) for the six main controlled-separation settings; $1{,}000$ replications per entry.}
\label{tab:controlled-sep-uncertainty}
\begin{tabular}{lrrrrrr}
\toprule
\tableheadrow
 & \multicolumn{2}{c}{$t=0$} & \multicolumn{2}{c}{$t=\frac12$} & \multicolumn{2}{c}{$t=1$}\\
\tableheadrow
Queries/kernel & 100 & 400 & 100 & 400 & 100 & 400\\
\midrule
Uniform & $\mathbf{11.07}\,(0.48)$ & $2.98\,(0.13)$ & $11.44\,(0.50)$ & $\mathbf{2.69}\,(0.12)$ & $9.79\,(0.44)$ & $2.46\,(0.12)$\\
Occupancy + pilot & $17.68\,(0.79)$ & $3.89\,(0.17)$ & $19.12\,(0.86)$ & $3.73\,(0.17)$ & $16.51\,(0.76)$ & $3.22\,(0.16)$\\
Learned mean & $18.05\,(0.79)$ & $3.82\,(0.16)$ & $19.53\,(0.92)$ & $3.74\,(0.16)$ & $18.32\,(1.07)$ & $3.30\,(0.17)$\\
Complete rollout & $11.16\,(0.56)$ & $\mathbf{2.84}\,(0.12)$ & $\mathbf{11.27}\,(0.50)$ & $2.87\,(0.12)$ & $10.84\,(0.46)$ & $2.90\,(0.13)$\\
\primaryrow \tis{} & $59.96\,(3.19)$ & $12.82\,(0.66)$ & $41.45\,(2.68)$ & $8.47\,(0.44)$ & $\mathbf{2.17}\,(0.11)$ & $\mathbf{0.40}\,(0.02)$\\
\primaryrow anchored \tis{} & $22.55\,(0.98)$ & $4.52\,(0.20)$ & $19.02\,(0.97)$ & $3.14\,(0.14)$ & $3.65\,(0.17)$ & $0.66\,(0.03)$\\
\referencerow Oracle + floor & $11.07\,(0.48)$ & $2.98\,(0.13)$ & $8.62\,(0.42)$ & $2.08\,(0.09)$ & $1.26\,(0.06)$ & $0.30\,(0.01)$\\
\bottomrule
\end{tabular}
\end{table}

\subsubsection{Seasonal base-stock inventory evaluation}
\label{sec:supp-inventory}
This stage-dependent benchmark supports \textbf{Q2}; each $(h,s)$ is a separate query group (the untied case of Section~\ref{sec:supp-notation}). Inventory is $\{0,\ldots,6\}$, a fixed seasonal policy orders toward four or five units over $H=8$, demand is truncated Poisson with seasonally varying mean, and a disruption with probability $.06$ ($.08$ at capacity) removes one extra unit and lowers the reward category. Sales, ordering, holding, and lost-demand terms determine profit, quantized to $\{0,1/2,1\}$ on grid $\Delta=1/2$.

From zero initial inventory at $\alpha=.1$, backward structural reachability leaves $B=41$ blocks. Exact enumeration gives categorical CVaR $1.1512371$, margin $0.0436202$, and
\[
 V_{\rm unif}=7.54133,\qquad V^*=4.13541,\qquad V_{\rm unif}/V^*=1.82360.
\]
Budgets are 150, 300, 600, and 1,200 queries per retained block (6,150--49,200 total calls); pilot, floor, rounding, and sample splitting follow the controlled-MRP protocol. We use 300 replications.

\begin{table}[t]
\centering
\small
\caption{Inventory MSE $\times10^3$ over 300 replications versus charged queries per retained block.}
\label{tab:inventory-full}
\begin{tabular}{lrrrr}
\toprule
\tableheadrow
Method & 150 & 300 & 600 & 1,200\\
\midrule
Uniform & 1.232 & 0.583 & 0.303 & 0.176\\
\referencerow Reachability only & 1.091 & 0.659 & 0.297 & 0.177\\
\referencerow Local scale only & 1.908 & 1.049 & 0.531 & 0.285\\
\referencerow Oracle + floor & 0.731 & 0.329 & 0.150 & 0.100\\
Learned occupancy & 1.517 & 0.634 & 0.311 & 0.161\\
Learned mean & 1.361 & 0.571 & 0.302 & 0.152\\
Complete rollout & 2.794 & 1.774 & 0.836 & 0.391\\
\primaryrow anchored \textsc{TIS} & 1.012 & 0.456 & 0.217 & 0.137\\
\primaryrow \textbf{\textsc{TIS}} & \best{1.007} & \best{0.414} & \best{0.174} & \best{0.116}\\
\bottomrule
\end{tabular}
\end{table}

Pilot fractions are $22.0\%,17.3\%,13.8\%,10.9\%$. \tis{} has the lowest observed learned-method MSE at every budget (Table~\ref{tab:inventory-full}); at 1,200 queries per block it is $15.4\%$ above oracle+floor. Rollouts exceed uniform throughout: with $H=8$, 49,200 transitions yield only 6,150 returns, or 615 returns' worth of mass in the worst decile.

At 600 queries per block, 300 independent pilots per setting isolate allocation learning through $V(\widehat w)/V^*$; this excludes the main-sample cost of larger pilots, for which the leading cost-inclusive MSE is $V(\widehat w)/(N-Gm_N)$. Table~\ref{tab:pilot-sensitivity} gives medians and 90th percentiles; smaller exploration exponents correspond to larger floors.

\begin{table}[t]
\centering
\small
\caption{Inventory pilot sensitivity: median (90th percentile) $V(\widehat w)/V^*$ over 300 pilots.}
\label{tab:pilot-sensitivity}
\begin{tabular}{rrrr}
\toprule
\tableheadrow
Pilot fraction & $\lambda=N^{-1/6}$ & $N^{-1/4}$ & $N^{-1/3}$\\
\midrule
$3.7\%$ & 1.124 (1.184) & \best{1.111 (1.184)} & 1.123 (1.212)\\
$7.2\%$ & 1.082 (1.108) & 1.057 (1.087) & \best{1.054 (1.088)}\\
$14.2\%$ & 1.064 (1.073) & 1.034 (1.043) & \best{1.028 (1.038)}\\
$28.3\%$ & 1.055 (1.060) & 1.023 (1.028) & \best{1.015 (1.020)}\\
\bottomrule
\end{tabular}
\end{table}

\begingroup
\subsubsection{Slippery CliffWalking with stationary state--action kernels}
\label{sec:supp-cliffwalking}
Figure~\ref{fig:main}(a) provides the shared-kernel \textbf{Q2} benchmark: stationary state--action laws are reused across Bellman stages. We use Gymnasium slippery CliffWalking-v1 \citep{towers2024gymnasium}, a $4\times12$ grid with start 36, goal 47, cliff cells 37--46, and four actions; each action realizes its intended or either perpendicular direction with probability $1/3$, boundary moves stay put, cliffs reset to start, and the fixed-horizon wrapper makes the goal absorbing.

We set $H=20$ and map raw rewards $-100,-1,0$ to $0,.99,1$. With zero reward after termination, the normalized return is exactly $20+G_{\rm episodic}/100$, so the transform preserves lower-tail ordering and CVaR. The grid contains every sum of 20 elements of $\{0,.99,1\}$ (231 atoms), hence the Bellman recursion is exact. The primary policy is $\epsilon=.05$-soft around a route crossing row 2 immediately above the cliff; a prespecified safe diagnostic policy crosses row 1, and off-route states first steer toward the chosen corridor. Structural reachability leaves 149 state--action groups, including the absorbing goal. For the primary policy at $\alpha=.1$, CVaR is $11.8709354$, the margin is $0.0139085$, and $V_{\rm unif}/V^*=84.6973$; across both policies and $\alpha\in\{.05,.1,.2\}$, exact ratios span $77.5748$--$116.3908$. Their minimum route lengths are 13 and 15, so $H=20$ allows completion and slippery deviations with a tractable exact grid. All reported methods evaluate the primary policy; the safe policy is a separate policy--risk check.

Population occupancy and mean-influence references use exact scores with the same $N^{-1/4}$ floor; learned counterparts use the pilot fit. The discarded pilot is $m_N=\max\{8,\lceil4N^{2/3}/149\rceil\}$ per group. For MC-UCB \citep{carpentier2015adaptive}, whose original guarantee concerns fixed-stratum weighted means rather than our Bellman estimator, we freeze the pilot influence scores and treat fresh main draws as bounded score arms. After two draws per group it selects the largest
\[
 B_{g,t}=\frac{1/G}{T_{g,t-1}}
 \left(\widehat s_{g,t-1}+\frac{2\beta}{\sqrt{T_{g,t-1}}}\right),
\]
where $T_{g,t-1}$ is the main-sample count and $\widehat s_{g,t-1}$ the score standard deviation. We use the published bounded-arm choices $\delta=n^{-9/2}$ and $\beta=c\sqrt{\log(2/\delta)}$, with $n$ the main budget and $c$ the largest frozen-arm range. Positive affine rescaling leaves the rule unchanged; computing $c$ uses the public three-outcome support but not its probabilities, which is extra information relative to unknown-support simulators. MC-UCB and \tis{} share the charged pilot, exhaust the same main budget, use the same categorical plug-in, and select no hyperparameter from outcomes. We use 500 independent replications. \label{sec:supp-cliff-full-note}\purplerev{The no-covariance ablations in Table~\ref{tab:cliff-full} show that most gains here come from pooling reused kernels; the theoretically required cross-layer covariance has only a small numerical effect.}

\begin{table}[t]
\centering
\small
\caption{CliffWalking MSE $\times10^3$ over 500 replications versus charged queries per reachable kernel.}
\label{tab:cliff-full}
\begin{tabular}{lrrrr}
\toprule
\tableheadrow
Method & 50 & 100 & 200 & 400\\
\midrule
Uniform & 1683.126 & 818.420 & 363.617 & 216.877\\
\referencerow Occupancy (population) & 26.269 & 12.474 & 5.727 & 2.741\\
\referencerow Mean influence (population) & 19.182 & 9.731 & 4.371 & 2.164\\
Learned mean & \best{27.413} & 11.367 & 5.349 & 2.307\\
Learned occupancy & 35.792 & 15.924 & 6.450 & 3.189\\
Complete rollout & 65.883 & 32.221 & 16.257 & 7.954\\
MC-UCB (frozen) & 1863.426 & 840.730 & 351.955 & 171.316\\
\referencerow Oracle no-cov & 16.203 & 7.462 & 3.416 & 1.655\\
\referencerow Oracle + floor & 16.288 & 7.631 & 3.434 & 1.674\\
\textsc{TIS} no-cov & 29.979 & 10.821 & \best{4.242} & \best{1.875}\\
\primaryrow anchored \textsc{TIS} & 28.427 & 11.551 & 5.362 & 2.182\\
\primaryrow \textbf{\textsc{TIS}} & 29.560 & \best{10.755} & 4.305 & 1.884\\
\bottomrule
\end{tabular}
\end{table}

Pilot fractions are $22.0\%,17.0\%,13.0\%,10.25\%$. At 50 queries per kernel, \tis{} is resolved worse than population mean influence and oracle+floor, while its differences from population occupancy and learned mean are unresolved; it beats population occupancy from 100 queries and population mean at 400. At 400, MSE reductions are $31.3\%$, $12.9\%$, and $18.3\%$ versus population occupancy, population mean, and learned mean. Deleting covariance from the learned score yields no resolved difference at any budget. Rollouts beat uniform but trail the occupancy and influence designs, except MC-UCB. MC-UCB's mean realized first-order variance ratios are $96.69,85.80,76.18,67.77$ versus uniform's $84.70$; under its published significance schedule, the confidence bonus dominates at these pull counts, keeping allocation near uniform and yielding only modest improvement at larger budgets.

\subsubsection{Additional public stationary Gymnasium environments}
\label{sec:supp-public-gym}
These breadth checks test whether the shared-kernel conclusions extend beyond CliffWalking and expose a regime in which a smoother mean score can be preferable. We use stationary \texttt{FrozenLake-v1} and rainy \texttt{Taxi-v4} \citep{towers2024gymnasium}, with fixed policies, reused state kernels, the CliffWalking budgets and estimator, and prespecified screening for at least two stochastic reachable kernels, nonzero oracle influence variance, a positive categorical margin, and an exact finite grid.

\textbf{FrozenLake.} The slippery $8\times8$ task leaves 50 reachable nonterminal kernels under a fixed success-maximizing policy. Its failure probability $.13704$ makes the first-order tail signal zero at $\alpha=.05,.1$, so $\alpha=.2$ is the smallest prespecified passing level; there CVaR is $.3147769$, the margin is $.0629554$, and $V_{\rm unif}/V^*=3.92589$. \textbf{Rainy Taxi.} The fixed shortest-route policy leaves 37 reachable nonterminal kernels. The positive affine reward transform $(r+1)/21$ gives the exact grid $\{0,\ldots,220\}/21$ and preserves CVaR; at $\alpha=.1$, transformed CVaR is $9.0541409$, the margin is $.0121392$, and $V_{\rm unif}/V^*=2.33938$.

Both tasks use 300 replications at 50--400 queries per reachable kernel. At 400 queries, \tis{} MSE is $.00791\,(.00079)$ on FrozenLake and $.00020\,(.00002)$ on Taxi, versus uniform $.01464\,(.00119)$ and $.00033\,(.00003)$. Learned mean is better on FrozenLake ($.00577\,(.00051)$) and tied at displayed precision on Taxi ($.00020\,(.00002)$); differences from population occupancy and covariance-deleted \tis{} are unresolved, while oracle+floor is best on both. At smaller budgets, pilot cost can make \tis{} worse than uniform or population references. These checks reinforce the structural limit rather than a universal win: when tail and mean signals align, the smoother mean score can be as good as or better than tail targeting.

\subsubsection{Structured language-model workflow evaluation}
\label{sec:llm-protocol}
\begin{table}[htbp]
\centering
\small
\setlength{\tabcolsep}{5pt}
\caption{\textbf{MMLU-Pro: anchoring mitigates plain TIS's observed failures.} Panel MSE/uniform MSE at $H=6$, $\alpha=.1$, 400 queries/kernel.}
\label{tab:llm-compact}
\begin{tabular}{lrrrrr}
\toprule
\tableheadrow
Generator & \tis{} & occup. & rollout & \cellcolor{tableprimary}\textbf{anchored} & \cellcolor{tablereference}\shortstack{oracle\\+floor}\\
\midrule
Qwen3-4B & 1.38 & .090 & .074 & \cellcolor{tableprimary}\best{.057} & \cellcolor{tablereference}.058\\
Phi-4-mini & .49 & .113 & .072 & \cellcolor{tableprimary}\best{.066} & \cellcolor{tablereference}.025\\
Granite-4.2-8B & .15 & .126 & .107 & \cellcolor{tableprimary}\best{.070} & \cellcolor{tablereference}.034\\
Mistral-24B & .21 & .138 & .127 & \cellcolor{tableprimary}\best{.093} & \cellcolor{tablereference}.055\\
Qwen3-32B & .41 & .079 & \best{.055} & \cellcolor{tableprimary}.056 & \cellcolor{tablereference}.028\\
GLM-4-32B & 1.77 & .087 & .069 & \cellcolor{tableprimary}\best{.065} & \cellcolor{tablereference}.034\\
\bottomrule
\end{tabular}
\end{table}
These frozen-law experiments support \textbf{Q3}: exact targets diagnose plain-\tis{} pilot failures and the effect of anchoring. We use 50 stratified ten-option MMLU-Pro questions \citep{wang2024mmlupro}. Qwen3-4B-Instruct-2507 is primary \citep{qwen3instruct2507}; Phi-4-mini-instruct and Granite-4.2-8B are cross-family checks \citep{phi4mini,granite42}; follow-up generators ($\dagger$) are Mistral-Small-24B-Instruct-2501, Qwen3-32B (thinking disabled), and GLM-4-32B-0414 \citep{mistralsmall24b,qwen332b,glm432b}. All use the same panel, prompts, decoder, policy, budgets, and methods; floor, utility, panel, and policy sensitivities are post hoc.

Each question is a separate finite-horizon Markov reward process. From $S_0=\varnothing$, state $S_t=(j_t,c_t)$ records the latest answer $j\in\{1,\ldots,10\}$ and confidence $c\in\{.1,\ldots,.9\}$. The root action is \texttt{solve}; reviews use \texttt{reconsider} for $c\le.3$, \texttt{challenge} for $.4\le c\le.6$, and \texttt{verify} for $c\ge.7$. The controller tracks $h=H-t$, but prompts omit history and stage index, so a fixed state and action have the same next-response law at every stage. Thus the pair is Markov, the dynamics are stationary, and any retained prompt can be queried directly rather than reached by rollout.

\begin{table}[htbp]
\centering\small
\caption{MMLU-Pro and FinQA workflow models; both use the confidence-band review policy.}
\label{tab:workflow-models}
\begin{tabular}{@{}>{\raggedright\arraybackslash}p{.18\linewidth}>{\raggedright\arraybackslash}p{.38\linewidth}>{\raggedright\arraybackslash}p{.38\linewidth}@{}}
\toprule
\tableheadrow
Component & MMLU-Pro & FinQA\\
\midrule
State after a call & Answer $j$ and confidence $c$ ($10\times9$ possibilities) & Candidate $i$ and confidence $c$ ($8\times9$ possibilities)\\[3pt]
Initial action & \texttt{solve} & \texttt{select}\\[3pt]
Number of calls & $H=2,4,6$, including the initial solve & $H=3$: selection and two reviews\\[3pt]
Stage index & $h=H-t$ calls remaining; stop at $h=0$ & $h=3,2,1,0$; stop at $h=0$
\\[3pt]
Query groups & $1+90=91$ per question & $1+72=73$ per question and workflow\\[3pt]
Reward on $s\to s'$ & $r_q(s')$: normalized Brier utility & $[u(s')-u(s)+1]/2$: shifted utility change\\[3pt]
Total return & $\sum_{t=1}^{H}r_q(S_t)$ & $[H+u(S_H)]/2$, with $u(S_0)=0$\\
\bottomrule
\end{tabular}
\end{table}
Response probabilities factor as restricted answer softmax times conditional confidence softmax, with each label one token. A query to $g=(q,s,a)$ draws $S'$ and reward $r_q(S')$:
\begin{equation}
 P_{q,s,a}(r,s')=
 P_{\rm LLM}\!\left(s'\mid\operatorname{prompt}(q,s,a)\right)
 \ind\{r=r_q(s')\}.
 \label{eq:llm-query-kernel}
\end{equation}
For example, $(1,.2)$ selects \texttt{reconsider}; response $(3,.8)$ earns $r_q(3,.8)$ and, if a call remains, selects \texttt{verify}. One fitted law and sample count are shared across all uses of $g$, but its return effect depends on calls remaining; \tis{} therefore sums these stage effects before computing the group scale. Allocation is learned separately for every question, generator, and $H$. With one action per state there are 91 groups per question (4,550 per generator).

\textbf{Reward and exact target.} For response $y=(j,c)$, define the ten-class forecast and normalized Brier utility \citep{brier1950}
\begin{equation}
 p_k(y)=
 \begin{cases}
 c,&k=j,\\
 (1-c)/9,&k\ne j,
 \end{cases}
 \qquad
 r_q(y)=1-\frac12\sum_{k=1}^{10}
 \left(p_k(y)-\ind\{k=j_q^*\}\right)^2 ,
 \label{eq:llm-brier-reward}
\end{equation}
where $j_q^*$ is the published correct option; no learned judge is used. The return $G_{H,q}=\sum_{t=1}^{H}r_q(S_t)$ includes the initial answer and every review, measuring cumulative response utility (FinQA instead uses terminal severity). Dividing by $H$ rescales CVaR and MSE but leaves within-horizon ratios and allocations unchanged.

The grid is the union of all attainable partial-return supports for $h=0,\ldots,H$ plus endpoints $0,H$: 121, 617, and 983 atoms at $H=2,4,6$. Closure makes both the categorical recursion and stop-loss interpolation exact; each question--horizon pair has its own target, margin, and scales. We use $\alpha=.1$ primarily and $.2$ as a sensitivity check. Partial returns are essential: two deterministic $.5$ rewards have true sum $1$, but backing them up on $\{0,1,2\}$ yields $\tfrac14\delta_0+\tfrac12\delta_1+\tfrac14\delta_2$, preserving the mean while driving CVaR$_{.1}$ to $0$; including $.5$ restores exactness. Across all 1,800 question--horizon--risk cells, full-horizon-only grids have median/90th-percentile absolute CVaR gaps $.02/.35$, whereas closed-grid gaps are below $10^{-14}$. All reported runs use the closed grids; grid smoothing partly masks the primary generator's depth failure.

\textbf{Ground truth and validation.} Enumerating each prompt's 90 probabilities and applying dynamic programming gives exact targets, influences, and population references; sample-only methods receive fresh $(R,S')$ draws. Enumeration is practical here only because the output alphabet is restricted, so the study measures logical-query efficiency under frozen laws while remaining relevant to longer outputs, restricted APIs, and stochastic tools. All kernels pass the designated direct-generation audit. Methods otherwise follow Section~\ref{sec:supp-experiments}; MC-UCB uses its sequential index.

\textbf{Budgets and uncertainty.} Each generator--horizon cell has $J=300$ replications at $b=50,100,200,400$ queries per group, so $N=91b=4{,}550,9{,}100,18{,}200,36{,}400$ draws per question. The charged pilot $m_N=\max\{8,\lceil4N^{2/3}/91\rceil\}$ uses $13,20,31,49$ draws per group and $\lambda_N=N^{-1/4}$; the remaining fresh draws follow the learned shares and rounding rule, and the pilot is excluded from the final estimate. One logical query is a constrained two-token draw; ground-truth enumeration is outside this budget. Fixed-allocation methods share outcome streams, sequential methods use separate streams, and replications are independent. For squared-error difference $D_{q,\ell}$ between methods A and B on question $q$ and replication $\ell$,
\begin{equation}
 D_\ell=\frac1Q\sum_{q=1}^Q D_{q,\ell},\qquad
 \widehat\Delta=\frac1J\sum_{\ell=1}^J D_\ell,\qquad
 \widehat{\operatorname{se}}(\widehat\Delta)=\frac{s_D}{\sqrt J},
 \label{eq:llm-panel-se}
\end{equation}
where $s_D^2=(J-1)^{-1}\sum_{\ell}(D_\ell-\widehat\Delta)^2$ and $z=\widehat\Delta/\widehat{\operatorname{se}}(\widehat\Delta)$. This fixed-panel calculation allows arbitrary within-replication dependence among questions; uncertainty is simulation error conditional on the panel, not population generalization.

\textbf{Checks and main pattern.} All 150 question--horizon margins are positive at each risk level, and shared and population-untied targets agree within $10^{-10}$. Median population $V_{\rm unif}/V^*$ at $\alpha=.1$ is $46,52,37,27,54,42$ in table order. Table~\ref{tab:supp-llm-full} shows occupancy below plain \tis{} in all 18 settings and the anchor below both; rollouts have the lowest sample-only point MSE for all generators at $H=2$, while the anchor does so for five at $H=6$. MC-UCB gives no consistent gain over uniform ($.956$--$1.48$), population occupancy/mean are strong ($.024$--$.103/.016$--$.074$), and answer entropy is poor ($2.09$--$7.74$). Pooling matters more than covariance correction: deeper untied evaluation splits data across $1+90(H-1)$ pilot groups, whereas deleting covariance changes little (population no-cov is within $.003$ of oracle+floor in uniform-MSE units). Tables~\ref{tab:supp-llm-budget}--\ref{tab:supp-llm-alpha} give budget and risk sensitivity.

\begin{table}[t]
\centering
\small
\setlength{\tabcolsep}{3pt}
\caption{MMLU-Pro, $\alpha=.1$, 400 queries/shared kernel: MSE/uniform MSE. $\dagger$ marks follow-up generators; Table~\ref{tab:supp-llm-absolute} gives absolute \tis{} MSE.}
\label{tab:supp-llm-extension}
\label{tab:supp-llm-full}
\begin{tabular}{lrrrrrrrrrr}
\toprule
\tableheadrow
Generator & $H$ & \shortstack{plain\\\tis{}} & \cellcolor{tableprimary}\shortstack{anchored\\\tis{}} & \shortstack{learned\\mean} & \shortstack{learned\\occup.} & rollout & \shortstack{\tis{}\\no cov.} & untied & MC-UCB & \cellcolor{tablereference}\shortstack{oracle\\+floor}\\
\midrule
Qwen3-4B & 2 & .719 & \cellcolor{tableprimary}.030 & .307 & .036 & \best{.028} & .719 & .724 & 1.04 & \cellcolor{tablereference}.192\\
 & 4 & 1.02 & \cellcolor{tableprimary}.059 & .707 & .069 & \best{.057} & 1.02 & 4.28 & 1.08 & \cellcolor{tablereference}.058\\
 & 6 & 1.38 & \cellcolor{tableprimary}\best{.057} & 1.00 & .090 & .074 & 1.35 & 7.67 & 1.06 & \cellcolor{tablereference}.058\\
\midrule
Phi-4-mini & 2 & .181 & \cellcolor{tableprimary}.027 & .106 & .032 & \best{.023} & .181 & .183 & 1.03 & \cellcolor{tablereference}.015\\
 & 4 & .326 & \cellcolor{tableprimary}\best{.047} & .264 & .074 & .048 & .327 & 1.33 & 1.11 & \cellcolor{tablereference}.020\\
 & 6 & .487 & \cellcolor{tableprimary}\best{.066} & .425 & .113 & .072 & .493 & 3.00 & 1.48 & \cellcolor{tablereference}.025\\
\midrule
Granite-4.2-8B & 2 & .172 & \cellcolor{tableprimary}.033 & .038 & .042 & \best{.029} & .172 & .160 & 1.06 & \cellcolor{tablereference}.021\\
 & 4 & .127 & \cellcolor{tableprimary}\best{.051} & .065 & .079 & .062 & .127 & .834 & .981 & \cellcolor{tablereference}.027\\
 & 6 & .154 & \cellcolor{tableprimary}\best{.070} & .089 & .126 & .107 & .164 & 1.56 & 1.09 & \cellcolor{tablereference}.034\\
\midrule
Mistral-24B$^\dagger$ & 2 & .140 & \cellcolor{tableprimary}.038 & .041 & .043 & \best{.026} & .140 & .123 & .967 & \cellcolor{tablereference}.020\\
 & 4 & .148 & \cellcolor{tableprimary}\best{.067} & .100 & .089 & .075 & .144 & .763 & 1.07 & \cellcolor{tablereference}.038\\
 & 6 & .214 & \cellcolor{tableprimary}\best{.093} & .159 & .138 & .127 & .221 & 1.90 & 1.13 & \cellcolor{tablereference}.055\\
\midrule
Qwen3-32B$^\dagger$ & 2 & .321 & \cellcolor{tableprimary}.028 & .199 & .030 & \best{.023} & .321 & .308 & .956 & \cellcolor{tablereference}.108\\
 & 4 & .410 & \cellcolor{tableprimary}\best{.044} & .373 & .058 & \best{.044} & .409 & 2.43 & 1.03 & \cellcolor{tablereference}.036\\
 & 6 & .411 & \cellcolor{tableprimary}.056 & .440 & .079 & \best{.055} & .412 & 3.71 & 1.08 & \cellcolor{tablereference}.028\\
\midrule
GLM-4-32B$^\dagger$ & 2 & .438 & \cellcolor{tableprimary}.035 & .426 & .039 & \best{.025} & .438 & .425 & 1.06 & \cellcolor{tablereference}.027\\
 & 4 & 1.08 & \cellcolor{tableprimary}.062 & .790 & .079 & \best{.055} & 1.08 & 4.24 & 1.28 & \cellcolor{tablereference}.042\\
 & 6 & 1.77 & \cellcolor{tableprimary}\best{.065} & 1.42 & .087 & .069 & 1.80 & 6.95 & 1.29 & \cellcolor{tablereference}.034\\
\bottomrule
\end{tabular}
\end{table}

\textbf{Depth failures and anchoring.} At $H=6$, plain \tis{} exceeds uniform MSE for Qwen3-4B ($1.38$) and GLM-4-32B ($1.77$), and more budget does not reliably remove the failures (Table~\ref{tab:supp-llm-budget}); at $H=4$ the full sweep reaches $1.10$ and $1.21$. At $\alpha=.2$, only GLM $H=6$ remains above uniform ($1.52$), while its $H=4$ difference is unresolved. For GLM's primary $H=4,6$ cells, bias is at most $3\%$ of MSE but realized population-scale variance is $1.36$ and $2.66$ times uniform: rare pilots starve influential groups, making $\sigma_g^2/w_g$ large (Section~\ref{sec:pilot-mechanism}), outside Proposition~\ref{prop:finite-pilot}'s local regime.

A retrospective floor sweep supports the underallocation diagnosis: at $H=6$, $\alpha=.1$, 400 queries, increasing the floor from $.072$ to $.40$ lowers Qwen and GLM \tis{}/uniform MSE from $1.38/1.77$ to $.47/.55$, with still larger floors reducing them further. Because the sweep followed the failures, it is diagnostic rather than a tuning result; held-out FinQA below tests the subsequently frozen anchor. Workflow selection on this MMLU panel is nearly saturated, so the informative outcome is estimation error rather than final pairwise choice.

Three post-hoc variations point to the same pilot-reliability mechanism. Under an asymmetric confidence-weighted utility, Qwen plain-\tis{}/uniform ratios at $H=2,4,6$ are $.39,1.02,1.06$, versus anchor $.031,.053,.056$; on a disjoint high-stakes panel they are $.96,.42,.72$ versus anchor $.039,.052,.052$. A cautious policy again gives plain \tis{} $1.04$ of uniform at $H=6$ while occupancy is $.13$. These checks support the diagnosis only; none was used to choose the frozen anchor.

\begin{table}[t]
\centering\small
\caption{Absolute \tis{} panel MSE (Monte Carlo SE), $\alpha=.1$, 400 queries/kernel; 300 replications.}
\label{tab:supp-llm-absolute}
\begin{tabular}{lrrr}
\toprule
\tableheadrow
Generator & $H=2$ & $H=4$ & $H=6$\\
\midrule
Qwen3-4B & 2.2e-04 (1e-05) & 9.7e-04 (6e-05) & 2.7e-03 (2e-04)\\
Phi-4-mini & 2.8e-04 (1e-05) & 1.7e-03 (1e-04) & 5.1e-03 (5e-04)\\
Granite-4.2-8B & 9.8e-05 (5e-06) & 1.3e-04 (1e-05) & 2.4e-04 (1e-05)\\
Mistral-24B$^\dagger$ & 1.2e-04 (2e-05) & 3.5e-04 (3e-05) & 9.6e-04 (1e-04)\\
Qwen3-32B$^\dagger$ & 5.7e-04 (4e-05) & 2.9e-03 (3e-04) & 7.5e-03 (8e-04)\\
GLM-4-32B$^\dagger$ & 4.1e-04 (3e-05) & 4.2e-03 (2e-04) & 1.9e-02 (9e-04)\\
\bottomrule
\end{tabular}
\end{table}

\begin{table}[t]
\centering\small
\caption{Budget sensitivity at $H=6$, $\alpha=.1$: shared/untied \tis{} MSE relative to uniform at equal total budget.}
\label{tab:supp-llm-budget}
\begin{tabular}{lrrrrrrrr}
\toprule
\tableheadrow
 & \multicolumn{4}{c}{shared \tis{}} & \multicolumn{4}{c}{untied \tis{}}\\
\tableheadrow
Generator & 50 & 100 & 200 & 400 & 50 & 100 & 200 & 400\\
\midrule
Qwen3-4B & .819 & .851 & \best{1.22} & \best{1.38} & 8.59 & 4.53 & 5.44 & 7.67\\
Phi-4-mini & .448 & .499 & .469 & .487 & 4.29 & 2.02 & 2.55 & 3.00\\
Granite-4.2-8B & .469 & .386 & .248 & .154 & 5.49 & 1.64 & 2.06 & 1.56\\
Mistral-24B$^\dagger$ & .497 & .462 & .391 & .214 & 5.81 & 1.58 & 2.08 & 1.90\\
Qwen3-32B$^\dagger$ & .681 & .606 & .497 & .411 & 5.45 & 2.71 & 2.91 & 3.71\\
GLM-4-32B$^\dagger$ & \best{1.02} & \best{1.44} & \best{1.76} & \best{1.77} & 6.39 & 2.65 & 4.81 & 6.95\\
\bottomrule
\end{tabular}
\end{table}

\begin{table}[t]
\centering\small
\caption{Risk sensitivity at $\alpha=.2$, 400 queries/kernel: \tis{} MSE/uniform MSE with uniform-minus-\tis{} $z$ (positive favors \tis{}); GLM $H=4$ is unresolved.}
\label{tab:supp-llm-alpha}
\begin{tabular}{lrrr}
\toprule
\tableheadrow
Generator & $H=2$ & $H=4$ & $H=6$\\
\midrule
Qwen3-4B & .195 (18.8) & .290 (14.4) & .613 (6.2)\\
Phi-4-mini & .115 (37.4) & .219 (25.0) & .521 (8.0)\\
Granite-4.2-8B & .120 (43.8) & .084 (42.6) & .088 (37.1)\\
Mistral-24B$^\dagger$ & .068 (49.5) & .175 (31.5) & .289 (16.9)\\
Qwen3-32B$^\dagger$ & .170 (33.5) & .232 (27.5) & .368 (16.2)\\
GLM-4-32B$^\dagger$ & .345 (20.6) & .929 (1.1) & \best{1.52} (-6.2)\\
\bottomrule
\end{tabular}
\end{table}

\endgroup

\subsubsection{Cost-to-accuracy analysis and computation accounting}
\label{sec:supp-cost-accounting}
A lower MSE at one budget need not imply fewer queries at every target accuracy, so we convert error curves to query cost and report computation separately. These comparisons are retrospective (no accuracy tolerance was declared before the runs), and charged query counts include pilots and rollout transitions.

\textbf{Multi-target relative cost.} At 21 RMSE targets in the common attained range, we use log--log first-crossing interpolation without extrapolation; out-of-grid crossings retain budget bounds. Pointwise 5th/95th-percentile envelopes come from empirical bootstrap when replication rows are available (CliffWalking \tis{}/learned mean, inventory \tis{}/uniform, all FinQA pairs) and Gaussian sensitivity otherwise; they are not simultaneous confidence bands.

\emph{CliffWalking.} \tis{} requires $.65$--$.84$ of learned-occupancy queries, $.32$--$.45$ of rollout queries, and $.84$--$1.05$ of learned-mean queries. Against learned mean, empirical envelopes are below one at the 12 strictest targets and above one at none; uniform's best tested RMSE exceeds \tis{}'s worst, censoring that comparison in \tis{}'s favor. \emph{Inventory.} Ratios are $.55$--$.72$ versus learned occupancy, $.58$--$.79$ versus learned mean, $.26$--$.30$ versus rollouts, and $.50$--$.83$ versus uniform, with envelopes below one at all 21 targets; only the uniform comparison uses the empirical bootstrap. \emph{FinQA.} On the 25--800-query grid, the anchor uses $.10$--$.17$ of uniform's charged queries on Qwen and $.23$--$.52$ on Phi, with empirical envelopes favoring it at all 21 targets. For Phi ordinary review, envelopes favor the anchor at 14 targets versus occupancy and the 10 strictest versus rollouts, while the baselines win at 4 and 8 targets; for unit-check review these counts are 12 and 6 for the anchor, and 0 and 12 for the baselines. Other contrasts are unresolved, and on Qwen both strong baselines require fewer queries at every target.

A single target can be misleading: uniform's largest-budget RMSE cannot rank sample-only CliffWalking designs because all reach it at the smallest tested budget; inventory shows savings across the studied range, but its half-budget crossing remains unresolved. \textbf{Computation is separate.} Table~\ref{tab:replay-timing} times complete sampling-and-estimation pipelines on an Intel i9-13900H with one BLAS/OMP thread (ten replications after two warm-ups). Uniform has no pilot; occupancy needs only the forward-visitation solve; \tis{}, the anchor, learned mean, and the occupancy--mean blend include the influence solve, with the latter three matching \tis{} within $4\%$, so the tail-specific score adds no computation over the strongest mean-based baseline. FinQA timing covers all 50 Phi ordinary-review questions (11--197 grid atoms). These frozen-law fixed-budget times are not live-call costs or time-to-equal-accuracy.

\begin{table}[htbp]
\centering\small
\caption{CPU milliseconds per replication for the full sampling-and-estimation pipeline at fixed query budgets; FinQA is the 50-question Phi ordinary-review panel. These are frozen-law replay times, not live-call costs.}
\label{tab:replay-timing}
\setlength{\tabcolsep}{4pt}
\begin{tabular}{lrrrrrrrr}
\toprule
\tableheadrow
\purplerev{Task} & \purplerev{Budget} & \purplerev{Uniform} & \purplerev{Occ.} & \purplerev{Mean} & \purplerev{\tis{}} & \purplerev{Anch.} & \purplerev{Mean blend} & \purplerev{Rollout}\\
\midrule
\purplerev{CliffWalking} & \purplerev{400} & \purplerev{110} & \purplerev{122} & \purplerev{394} & \purplerev{394} & \purplerev{396} & \purplerev{390} & \purplerev{\best{32}}\\
\purplerev{Inventory} & \purplerev{1,200} & \purplerev{\best{5.0}} & \purplerev{6.5} & \purplerev{10.7} & \purplerev{10.9} & \purplerev{10.8} & \purplerev{11.2} & \purplerev{5.2}\\
\purplerev{FinQA panel (50 q.)} & \purplerev{200} & \purplerev{658} & \purplerev{823} & \purplerev{2{,}031} & \purplerev{2{,}015} & \purplerev{2{,}039} & \purplerev{2{,}040} & \purplerev{\best{298}}\\
\bottomrule
\end{tabular}
\end{table}

\subsection{FinQA terminal-risk protocol and full results}
\label{sec:finqa-protocol}
\begin{figure}[htbp]
\centering
\includegraphics[width=\textwidth]{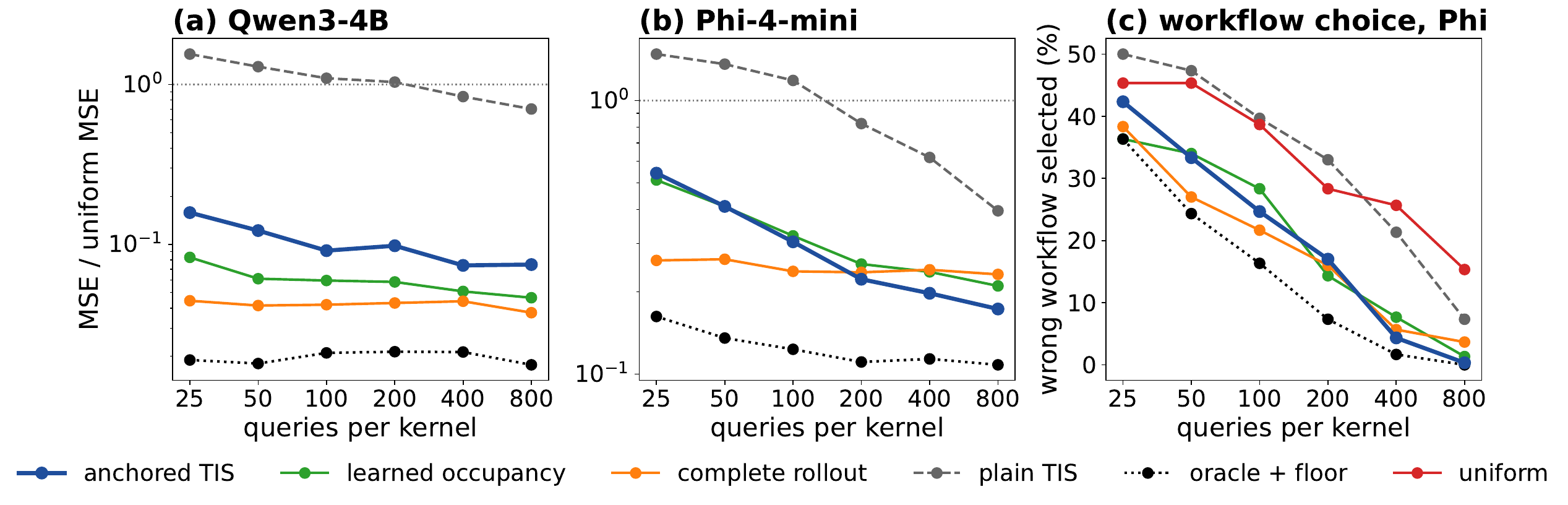}
\caption{\textbf{FinQA: the anchor beats occupancy and rollouts on Phi at 400/800 queries; Qwen favors the baselines.} (a,b) Mean panel-MSE/uniform-MSE ratio across workflows; (c) Phi wrong-selection rate. Table~\ref{tab:finqa-supp-est} gives per-workflow results.}
\label{fig:finqa-main}
\end{figure}
This held-out study supports \textbf{Q4}: does the frozen anchor improve estimation of numerical failure severity and workflow selection (Figure~\ref{fig:finqa-main})?

\textbf{Source and screening.} FinQA development items \citep{chen2021finqa} require finite nonzero executable gold answers and fit the context budget. Each eight-candidate bank is generated before gold-based scoring and admitted only if severity-utility spread is at least $.25$; this screen was fixed after 16 of the first 20 banks had constant utility, so the study targets material severity disagreement. Development admitted 20/170 items; after protocol freeze, a disjoint held-out scan admitted 50/312.

\textbf{Workflow and kernels.} Candidate banks are generated at temperature $.8$ and frozen, with parsing failures kept as invalid candidates. The root action is \texttt{select}; reviews use the MMLU confidence bands (Table~\ref{tab:workflow-models}). Prompts contain the financial context, question, full bank, current candidate--confidence pair, and review instruction, but no history or stage index. Two constrained output tokens determine the next state. The root plus 72 pairs gives 73 query groups per question, generator, and workflow, shared across stages. Ordinary and unit-and-sign workflows share the root prompt and policy but use different review templates, hence different transition laws. Runs have one selection and two reviews ($H=3$), with $\alpha=.1$.

\textbf{Severity utility.} For gold $y^*$ and candidate $y$, let $e_0=\min(|y-y^*|,|y/100-y^*|,|100y-y^*|)$, $t=.005|y^*|$, and $e=\max(0,e_0-t)$. Terminal utility is $u=1-e/(e+|y^*|)\in[0,1]$, with $u=0$ for unparseable candidates. Binary correctness would collapse magnitude: for $R\in\{0,1\}$ with failure probability $p_{\rm fail}$, lower CVaR is $\max\{\alpha-p_{\rm fail},0\}/\alpha$. A post-hoc unit-aware rerun changes screening for 3 development and 5 held-out questions; at 200 queries anchor/uniform MSE is $.096/.100$ on Qwen and $.226/.186$ on Phi for ordinary/unit-check review. The anchor remains resolved against uniform, trails occupancy/rollouts on Qwen, and beats occupancy on Phi; Phi's decision gap widens from $6.7\times10^{-4}$ to $5.6\times10^{-3}$. The frozen metric defines the primary results.

\textbf{Reward and exact target.} Let $u(s)$ be candidate severity utility, independent of confidence, and $u(\varnothing)=0$. For sampled next state $S'$,
\begin{equation}
 R(s,S')=\frac{u(S')-u(s)+1}{2}\in[0,1],\qquad
 G_H=\sum_{t=0}^{H-1}R(S_t,S_{t+1})
     =\frac{H+u(S_H)}{2}.
 \label{eq:finqa-reward-return}
\end{equation}
Intermediate utilities telescope, so a temporary improvement that is later reversed does not improve the return. The $h$-step return-to-go support is $\{(u_j-u_i+h)/2\}$; the grid contains these values and the endpoints, at most $H(C+1)C+2$ atoms ($C=8$; at most 197 observed). The categorical recursion is exact and matches exact-law CVaR to machine precision, with $C_{\rm ret}=[H+\cvar_\alpha(u(S_H))]/2$. All reported CVaRs, gaps, regrets, and MSEs use this scale; converting back by $2\widehat C_{\rm ret}-H$ multiplies MSE by four and doubles gaps without changing within-setting ratios or rankings.

\textbf{Freeze and audits.} Development used 20 screened questions and 100 replications to check parsing (142/160 candidates valid), the $.25$ spread screen, positive margins, oracle headroom (Qwen oracle+floor $.006$--$.013$ of uniform MSE at 100--200 queries), and replication noise; the latter informed held-out size without guaranteeing power. Before calibration, the protocol fixed 50 questions, 300 replications, generators, workflows, anchor, floor, and budgets 25--800 queries per kernel. All four held-out calibrations (3,650 kernels each) pass provenance and generation audits. All 50 margins are positive, but total influence is zero for 3 Qwen questions per workflow and 13/14 Phi questions (ordinary/unit-check), so every allocation has zero leading variance there although higher-order error can remain.

Table~\ref{tab:finqa-supp-est} gives the complete frozen-metric results through 800 queries; the post-hoc unit-aware sensitivity is summarized above. At 100--200 queries all eight anchor/uniform contrasts are resolved ($z=8.1$--$30.3$), while plain \tis{} is unresolved in seven. At 400/800 queries the anchor beats both strong baselines in both Phi workflows: MSE is $.81$--$.86$ of occupancy ($z=-8.4$ to $-4.8$) and $.68$--$.88$ of rollouts ($z=-8.4$ to $-3.1$), using $z$ for anchor minus baseline. On Qwen, occupancy and rollouts remain better. All contrasts use shared conditional-query streams, a separate rollout stream, and 300 replications.

\begin{table}[t]
\centering\scriptsize
\caption{FinQA frozen-metric estimation: uniform panel MSE, anchor-to-baseline MSE ratios, and paired uniform-minus-anchor $z$ over 300 replications.}
\begin{tabular}{llrrrrrr}
\toprule
\tableheadrow
Gen. & Workflow & Budget & Unif.\ MSE & \cellcolor{tableprimary}\textbf{Anch./unif.} & Anch./occ. & Anch./roll. & $z$\\
\midrule
qwen & ordinary & 25 & $6.2\,e{-5}$ & \cellcolor{tableprimary}0.152 & 1.84 & 3.52 & +5.3\\
qwen & ordinary & 50 & $2.8\,e{-5}$ & \cellcolor{tableprimary}0.127 & 2.10 & 2.81 & +6.5\\
qwen & ordinary & 100 & $1.6\,e{-5}$ & \cellcolor{tableprimary}0.089 & 1.47 & 2.07 & +8.1\\
qwen & ordinary & 200 & $6.7\,e{-6}$ & \cellcolor{tableprimary}0.097 & 1.77 & 2.30 & +9.9\\
qwen & ordinary & 400 & $3.3\,e{-6}$ & \cellcolor{tableprimary}0.075 & 1.41 & 1.62 & +10.0\\
qwen & ordinary & 800 & $2.0\,e{-6}$ & \cellcolor{tableprimary}0.077 & 1.66 & 2.11 & +10.2\\
qwen & unit check & 25 & $1.1\,e{-4}$ & \cellcolor{tableprimary}0.165 & 1.97 & 3.59 & +7.7\\
qwen & unit check & 50 & $5.7\,e{-5}$ & \cellcolor{tableprimary}0.117 & 1.90 & 3.10 & +10.2\\
qwen & unit check & 100 & $2.8\,e{-5}$ & \cellcolor{tableprimary}0.094 & 1.60 & 2.28 & +11.5\\
qwen & unit check & 200 & $1.2\,e{-5}$ & \cellcolor{tableprimary}0.100 & 1.61 & 2.26 & +14.3\\
qwen & unit check & 400 & $6.9\,e{-6}$ & \cellcolor{tableprimary}0.073 & 1.50 & 1.74 & +15.6\\
qwen & unit check & 800 & $3.6\,e{-6}$ & \cellcolor{tableprimary}0.073 & 1.56 & 1.89 & +15.5\\
phi & ordinary & 25 & $3.9\,e{-4}$ & \cellcolor{tableprimary}0.540 & 1.10 & 1.96 & +16.7\\
phi & ordinary & 50 & $2.0\,e{-4}$ & \cellcolor{tableprimary}0.431 & 1.03 & 1.49 & +21.4\\
phi & ordinary & 100 & $1.2\,e{-4}$ & \cellcolor{tableprimary}0.282 & 0.93 & 1.10 & +27.2\\
phi & ordinary & 200 & $5.9\,e{-5}$ & \cellcolor{tableprimary}0.216 & \best{0.88} & \best{0.84} & +30.3\\
phi & ordinary & 400 & $3.1\,e{-5}$ & \cellcolor{tableprimary}0.198 & \best{0.86} & \best{0.77} & +31.5\\
phi & ordinary & 800 & $1.5\,e{-5}$ & \cellcolor{tableprimary}0.184 & \best{0.82}& \best{0.68} & +27.1\\
phi & unit check & 25 & $4.5\,e{-4}$ & \cellcolor{tableprimary}0.546 & 1.02 & 2.23 & +15.7\\
phi & unit check & 50 & $2.5\,e{-4}$ & \cellcolor{tableprimary}0.390 & 0.97 & 1.65 & +21.0\\
phi & unit check & 100 & $1.3\,e{-4}$ & \cellcolor{tableprimary}0.327 & 0.97 & 1.50 & +23.3\\
phi & unit check & 200 & $7.3\,e{-5}$ & \cellcolor{tableprimary}0.229 & 0.88 & 1.07 & +26.9\\
phi & unit check & 400 & $3.5\,e{-5}$ & \cellcolor{tableprimary}0.197 & \best{0.81} & \best{0.88} & +28.9\\
phi & unit check & 800 & $1.9\,e{-5}$ & \cellcolor{tableprimary}0.162 & \best{0.83} & \best{0.84} & +26.3\\
\bottomrule
\end{tabular}
\label{tab:finqa-supp-est}
\end{table}

\begin{table}[t]
\centering\scriptsize
\caption{FinQA frozen-metric workflow selection at 25--200 queries/kernel: exact panel CVaRs, their gap, and wrong-selection percentages under shared workflow streams. Zero means no errors in 300 replications.}
\begin{tabular}{lrrrrrrrrrr}
\toprule
\tableheadrow
Gen. & Budget & $C_{\mathrm{ord}}$ & $C_{\mathrm{unit}}$ & Gap & Unif. & \tis{} & Occ. & Roll. & \cellcolor{tableprimary}\textbf{Anch.} & \cellcolor{tablereference}\shortstack{Oracle\\+floor}\\
\midrule
qwen & 25 & 1.87304 & 1.87262 & +0.00041 & \best{0.0} & 2.0 & \best{0.0} & \best{0.0} & \cellcolor{tableprimary}0.7 & \cellcolor{tablereference}0.0\\
qwen & 50 & 1.87304 & 1.87262 & +0.00041 & \best{0.0} & 0.3 & \best{0.0} & \best{0.0} & \cellcolor{tableprimary}\best{0.0}& \cellcolor{tablereference}0.0\\
qwen & 100 & 1.87304 & 1.87262 & +0.00041 & \best{0.0} & 0.3 & \best{0.0} & \best{0.0} & \cellcolor{tableprimary}\best{0.0}& \cellcolor{tablereference}0.0\\
qwen & 200 & 1.87304 & 1.87262 & +0.00041 & \best{0.0} & 1.0 & \best{0.0} & \best{0.0} & \cellcolor{tableprimary}\best{0.0}& \cellcolor{tablereference}0.0\\
phi & 25 & 1.76764 & 1.76697 & +0.00067 & 45.3 & 50.0 & \best{36.3} & 38.3 & \cellcolor{tableprimary}42.3 & \cellcolor{tablereference}36.3\\
phi & 50 & 1.76764 & 1.76697 & +0.00067 & 45.3 & 47.3 & 34.0 & \best{27.0} & \cellcolor{tableprimary}33.3 & \cellcolor{tablereference}24.3\\
phi & 100 & 1.76764 & 1.76697 & +0.00067 & 38.7 & 39.7 & 28.3 & \best{21.7} & \cellcolor{tableprimary}24.7 & \cellcolor{tablereference}16.3\\
phi & 200 & 1.76764 & 1.76697 & +0.00067 & 28.3 & 33.0 & \best{14.3} & 16.0 & \cellcolor{tableprimary}17.0 & \cellcolor{tablereference}7.3\\
\bottomrule
\end{tabular}\label{tab:finqa-supp-dec}
\end{table}

\textbf{Workflow decision and coupling sensitivity.} Table~\ref{tab:finqa-supp-dec} reports 25--200-query selection rates. Phi's ordinary/unit-check return-CVaR gap is only $6.7\times10^{-4}$; at 100/200 queries, discordant-pair tests resolve the anchor over uniform ($z=4.2,3.5$) but not over occupancy or rollouts ($|z|\le1.7$), and the budget-200 regret reduction is $7.6\times10^{-5}$. Workflow templates share common random numbers, which reduce comparison noise without changing either workflow's marginal MSE. A retrospective re-pairing check shows that Qwen's exceptionally low paired selection errors partly reflect this cancellation; because re-pairing is not fresh simulation, the frozen-protocol rates remain primary.

A separately committed skeptical-template follow-up on Phi had a much larger workflow gap and made the strong methods essentially error-free at small budgets, so it did not distinguish the anchor from occupancy or rollouts. This reinforces why the near-tie frozen study, rather than the easier follow-up, is the informative workflow-selection test.

\begingroup\revcolor{purple}
\subsection{Inventory disruption family}
\label{sec:supp-inventory-family}

This prespecified breadth family supports \textbf{Q2} by varying three axes of the seasonal inventory benchmark: base disruption probability $\{.01,.04,.12\}$ (plus $.02$ at capacity), disruption loss of $1,2,$ or $3$ remaining units, and two fixed policies (the original seasonal base-stock targets or those targets plus one unit, capped at capacity). Rewards, $\{0,.5,1\}$ quantization, demand laws, $H=8$, and the grid are unchanged, producing 18 cases with 41 or 48 retained blocks. Before simulation, exact enumeration confirmed distinct kernel laws in all 18 cases, positive categorical margins (including two near $.001$), CVaR $.675$--$1.260$, and $V_{\rm unif}/V^*$ from $1.49$ to $2.28$. Budgets are 150, 600, and 1,200 queries per block with 300 replications; methods, pilot, floor, and rounding match Section~\ref{sec:supp-inventory}. Absolute RMSE tolerances $.02,.01,.005$ were declared with the family.

\begin{table}[htbp]
\centering\revcolor{purple}\small
\caption{Inventory family: resolved lower-MSE cases (of 18) at 1,200 queries/block, using paired $|z|\ge2$; none resolve in the comparator's favor.}
\label{tab:inventory-family}
\begin{tabular}{lccccc}
\toprule
\tableheadrow
Method & occupancy & learned mean & occ.$+$uniform & occ.$+$mean & rollouts\\
\midrule
\tis{} & 18 & 17 & 18 & 18 & 18\\
anchored \tis{} & 18 & 17 & 18 & 17 & 18\\
\bottomrule
\end{tabular}
\end{table}

Across the 18 cases, median MSE/uniform at 150/600/1,200 queries per block is $.74/.62/.61$ for \tis{}, $.82/.69/.69$ for the anchor, $1.05/.92/.93$ for learned occupancy, $.98/.86/.84$ for learned mean, $1.11/.99/.95$ for the occupancy--uniform blend, $1.03/.87/.86$ for the occupancy--mean blend, $2.72/2.45/2.56$ for rollouts, and $.56/.51/.53$ for oracle+floor. At RMSE $.02$, \tis{} first reaches the target by 150 queries in 2 cases and by 600 in 15; uniform, occupancy, and the uniform blend require 1,200 in 6. At $.01$, \tis{} reaches 11 cases, uniform/occupancy 8, and rollouts none; at $.005$, only \tis{} reaches the target (2 cases). Plain \tis{} leads the anchor throughout, consistent with anchoring acting as insurance rather than a gain when pilots are reliable.

\subsection{Blending controls}
\label{sec:supp-mixture-controls}

To isolate \textbf{Q5}, that is whether the tail score itself drives the anchor, we compare two equally regularized controls. With floored component shares $w=(1-\lambda)\hat p+\lambda u$, they use $\tfrac12w_{\rm occ}+\tfrac12u$ and $\tfrac12w_{\rm occ}+\tfrac12w_{\rm mean}$, with no second floor; hence all three blends differ only in the component mixed with occupancy. The uniform blend equals occupancy with floor $(1+\lambda)/2$ and coincides with the anchor when the tail pilot falls back to uniform. All blends pay the same pilot and share coupled query streams. On held-out FinQA (2 generators $\times$ 2 workflows $\times$ budgets 100--800; 300 replications), the anchor is resolved better than the uniform blend in all 16 cells (MSE ratios $.62$--$.97$, decreasing with budget), while the mean blend has lower point MSE in all 16 and is resolved better in 15 (anchor/mean ratios $1.17$--$1.54$ on Qwen, $1.03$--$1.10$ on Phi). Learned mean alone can reach $1.4\times$ uniform, but occupancy--mean blending is competitive. Under MMLU-Pro confident-error utility (6 generators, $H=6$, $\alpha=.1$, 100--800 queries, 300 replications), the anchor beats the uniform blend and learned occupancy in all 24 cells and the mean blend in 23 (ratios $.76$--$.97$; Phi-4-mini at 100 is unresolved). Rollouts are better for five generators at 100 queries, but the anchor is better for all six at 800 (ratios $.55$--$.89$); plain \tis{} has $1.4$--$26\times$ the mean blend's MSE. Under Brier utility (100--400 queries), the anchor beats the uniform blend and occupancy in all 18 cells, the mean blend in 11, is worse in none, and is unresolved in 7 (all Phi-4-mini and Qwen3-32B budgets, plus Qwen3-4B at 100). On the high-stakes panel (Qwen3-4B, Phi-4-mini, Qwen3-32B), it beats the uniform blend and occupancy in all 9 cells and the mean blend in 5; the mean blend wins 2 (Phi-4-mini at 200/400). Brier comparisons with rollouts again cross over with budget: rollouts lead at small budgets, while the anchor leads four of six generators at 400.

\subsection{Rare failures and the tail--mean coincidence}
\label{sec:supp-rare-failure}

\begin{proof}[Proof of Proposition~\ref{prop:rare-failure}]
Let $F^{-1}$ be the quantile function of $X=G_H(s_0)$. By assumption, $X\le x^\star$ almost surely and
$\Pr(X<x^\star)=\pi<\alpha$. Hence $F^{-1}(u)=x^\star$ for every $u\in(\pi,1)$, and boundedness gives
\begin{align*}
 \alpha\,\cvar_\alpha(X)
 &=\int_0^\alpha F^{-1}(u)\,du\\
 &=\int_0^1F^{-1}(u)\,du-\int_\alpha^1F^{-1}(u)\,du\\
 &=\E[X]-(1-\alpha)x^\star.
\end{align*}
This proves the stated identity.

By exactness, the grid contains every attainable partial return generated by the fixed declared outcome spaces. Hence changing only the kernel laws changes probabilities but not the structural maximum $x^\star$, and the categorical root law continues to equal the true return law. In a finite horizon, the root return law depends continuously in total variation on the finite collection of group laws; therefore $\Pr(X<x^\star)$ remains below $\alpha$ throughout a sufficiently small neighborhood because the baseline gap $\alpha-\pi$ is strictly positive. On this neighborhood,
\[
 C_{\alpha,K}
 =\cvar_\alpha(X)
 =\frac1\alpha\E[X]-\frac{1-\alpha}{\alpha}x^\star.
\]

Let $J(P):=\E_P[X]$. Write $v_h(s)=\E[G_h(s)]$ and let $\mu_h(s)$ be the probability, under the fixed policy and population laws, of visiting state $s$ with $h$ steps remaining. Differentiating the ordinary mean Bellman recursion shows that, for a shared stationary group $g=(s,a)$, a centered groupwise influence function for $J$ is
\[
 \psi_{s,a}(W)
 =\sum_{h=1}^H\mu_h(s)\pi_h(a\mid s)
 \Big\{R+v_{h-1}(S')-\E_{P_{s,a}}[R+v_{h-1}(S')]\Big\},
\]
with the analogous single-row expression in the untied model. This is exactly the ordinary mean-return influence used by the learned-mean design in Appendix~\ref{sec:supp-experiments}. Consequently, for every DQM direction $s=(s_g)$,
$\dot J_s=\sum_g\E_{P_g}[\psi_gs_g]$. Differentiating the affine identity above and using Equation~\ref{eq:target-derivative} gives
\[
 \sum_g\E_{P_g}[\phi_gs_g]
 =\dot C_s
 =\frac1\alpha\dot J_s
 =\frac1\alpha\sum_g\E_{P_g}[\psi_gs_g].
\]
Both $\phi_g$ and $\psi_g$ are centered. Because the product tangent space contains an arbitrary $L_0^2(P_g)$ direction in each group, equality for every score direction implies
$\phi_g=\psi_g/\alpha$ in $L_0^2(P_g)$ for each group. Consequently the tail and mean influence standard deviations satisfy
$\sigma_g=\operatorname{sd}(\psi_g)/\alpha$. If these scales are not all zero, normalizing them gives identical Neyman shares; if they are all zero, both objectives have zero first-order variance for every allocation. This proves the allocation claim.
\end{proof}

This diagnostic supports \textbf{Q5} by testing when tail-specific allocation should differ from mean allocation. All experiments use closed grids exact on their declared supports. Normalize each nonzero tail- or mean-influence scale vector to sum one, representing an all-zero vector by uniform; then $\tfrac12\sum_g|p_g-m_g|=0$ whenever Proposition~\ref{prop:rare-failure} applies. The median distance is $.000$ on held-out FinQA for both generators and workflows, with exact zero on $45$--$51\%$ of questions. On 25-question MMLU-Pro samples at $H=6$, $\alpha=.1$, Brier-utility medians are $.24$ (Qwen3-4B), $.09$ (Phi-4-mini), $.29$ (Granite-4.2-8B), $.30$ (Mistral-24B), $.16$ (Qwen3-32B), and $.27$ (GLM-4-32B); the high-stakes panel gives $.33,.08,.18$ for Qwen3-4B, Phi-4-mini, Qwen3-32B, and confident-error utility gives $.32,.18,.32$ for Qwen3-4B, Phi-4-mini, GLM-4-32B. Granite, Mistral, Qwen3-32B, and the high-stakes Phi/Qwen3-32B distances were computed after the blending runs.

\purplerev{\textbf{Prospective divergence test.} We therefore fixed a rule before simulating five new settings: median distance $\ge.24$ predicts that the anchor has resolved lower MSE than the mean blend in at least two of three budgets (100/200/400 queries) and higher MSE in none; distance $\le.18$ predicts at most one resolved win; intermediate values make no prediction. The five settings were high-stakes confident-error Qwen3-4B ($.318$, win), Phi-4-mini ($.169$, no advantage), Qwen3-32B ($.216$, none), and cautious-policy Qwen3-4B with Brier ($.195$, none) or confident-error ($.269$, win). All use closed grids, $H=6$, $\alpha=.1$, and 300 replications. Against the mean blend, MSE ratios (paired $z$) at 100/200/400 are $.88/.84/.83$ ($-6.2/-8.5/-9.7$) for high-stakes Qwen3-4B, $1.03/1.02/.99$ ($+2.6/+1.4/-0.6$) for high-stakes Phi, and $.96/.95/.92$ ($-3.0/-4.0/-5.3$) for cautious confident-error Qwen, so all three predictions hold. The two unpredicted settings give $.97/.94/.91$ ($-1.3/-4.3/-6.0$) for high-stakes Qwen3-32B and $.96/.92/.88$ ($-3.3/-5.5/-8.7$) for cautious Brier Qwen. In all 15 cells the anchor also beats learned occupancy and the uniform blend; at 400 queries it beats complete rollouts in four settings (ratios $.74$--$.85$) and ties the fifth ($.99$).}

The confident-error utility scores a correct response with confidence $c$ as $(1+c)/2$ and a wrong response as $(1-c)^2/2$, making confident errors nearly worthless; it preserves the Brier per-response ordering but has a much heavier lower tail.

\subsection{FinQA with calculator faults}
\label{sec:supp-toolfault}

This robustness check extends \textbf{Q5} to exogenous tool errors while preserving the FinQA terminal-severity objective. Each review prompt includes an automated calculator report for the current candidate; independently after each model call, the report is correct with probability $1-p$ and multiplied by 100 with probability $p$. The root, 72 correct-report states, and 72 faulted-report states define 145 queryable prompt laws per question; an outcome is the model response together with the fault indicator, so one query still costs one model call. The laws were calibrated exactly for both generators on the 50 held-out questions, and the same design was declared for ordinary and unit-check review.

The perturbation is informative because influence and visitation react differently. For Qwen, faulted states carry median tail/mean/occupancy shares $5.8/6.2/0.7\%$ at $p=.01$ and $38/42/6.7\%$ at $p=.10$; for Phi the corresponding shares are $1.0/0.9/0.7\%$ and $9.8/9.0/6.7\%$; 12/50 Phi questions have no first-order tail signal. Thus tail and mean influence remain close even when occupancy can be very different. Across 100--400 queries per kernel, anchor/uniform MSE is $.034$--$.050$ for Qwen and $.081$--$.144$ for Phi in ordinary review, with similar $.041$--$.050$ and $.069$--$.144$ ranges under unit check. Yet rollouts are usually strongest, the mean blend beats the anchor throughout Qwen, and the anchor trails occupancy on Qwen while beating it on Phi. Plain \tis{} can have $7$--$29\times$ the mean blend's MSE. Qwen's oracle remains only $.007$--$.010$ of uniform MSE, locating the gap in pilot learning rather than the population influence signal; several Phi questions instead have near-zero margins. Median tail--mean distance is $.000$ for Qwen and $.03$--$.06$ for Phi, well below the no-advantage regime identified in Appendix~\ref{sec:supp-rare-failure}. The tool-fault study therefore supports the diagnostic's negative prediction: a tail score can be highly informative in population yet unnecessary relative to a smoother mean score when their normalized allocations nearly coincide.

\endgroup

\begingroup\revcolor{purple}
\subsection{Longer review loops}
\label{sec:supp-longh}
This extension supports the main-text longer-loop claim and re-tests \textbf{Q4--Q5}: because prompts omit stage index, the frozen kernels define longer loops without new model calls. Before simulation we declared MMLU-Pro confident-error runs at $H=8,10$ for Qwen3-4B, Phi-4-mini, and GLM-4-32B, plus held-out FinQA runs at $H=6$ for both generators and workflows (100--400 queries per kernel, 300 replications), using the $H=6$ MMLU and $H=3$ FinQA blending runs as references and recording the divergence values and three predictions.

\emph{Divergence rule.} Appendix~\ref{sec:supp-rare-failure}'s rule holds in 6/8 settings it decides: GLM (distance $.36/.39$) beats the mean blend at every budget at both horizons; all four FinQA settings (distance $.000$--$.098$) show no advantage; Qwen3-4B (distance $.26$ at both horizons) has anchor/mean ratios $.92$--$.98$ but resolves only at $H=8$, 400 queries, so its two win predictions fail. Phi-4-mini has distance $.19$ (no prediction) and wins 5/6 cells.

\emph{Kernel reuse.} Table~\ref{tab:longh-rollout} shows that longer shared-kernel loops increasingly favor the anchor for MMLU and FinQA Phi but not near-deterministic FinQA Qwen; the prediction holds in 5/7 settings. For Phi at $H=6$, anchor MSE is $.29$--$.64$ of rollout MSE at every tested budget in both workflows, i.e. $1.6$--$3.4\times$ lower and $2.4$--$3.4\times$ lower at 400 queries, which are the reductions quoted in the abstract.
\begin{table}[htbp]
\centering\revcolor{purple}
\small
\setlength{\tabcolsep}{3.5pt}
\caption{\textbf{Longer review loops favor the anchor over rollouts.} Anchored-\tis{}/rollout MSE ratio (paired $z$; negative favors the anchor), 300 replications. The Phi $H=6$, 400-query ratios $.29/.41$ give the $3.4\times/2.4\times$ reductions in the abstract.}
\label{tab:longh-rollout}
\begin{tabular}{lcccccc}
\toprule
\tableheadrow
 & \multicolumn{3}{c}{FinQA $H=3$} & \multicolumn{3}{c}{FinQA $H=6$}\\
\tableheadrow
Generator, workflow & 100 & 200 & 400 & 100 & 200 & 400\\
\midrule
Phi-4-mini, ordinary & $1.10$ ($+1.9$) & $.84$ ($-4.4$) & $.77$ ($-5.2$) & $.36$ ($-22.1$) & $.30$ ($-22.5$) & $.29$ ($-22.4$)\\
Phi-4-mini, unit check & $1.50$ ($+10.0$) & $1.07$ ($+1.8$) & $.88$ ($-3.1$) & $.64$ ($-10.4$) & $.47$ ($-17.3$) & $.41$ ($-19.7$)\\
Qwen3-4B, ordinary & $2.07$ ($+6.1$) & $2.30$ ($+5.9$) & $1.62$ ($+3.4$) & $2.64$ ($+6.5$) & $2.03$ ($+5.1$) & $2.02$ ($+5.3$)\\
Qwen3-4B, unit check & $2.28$ ($+9.5$) & $2.26$ ($+8.6$) & $1.74$ ($+5.4$) & $2.73$ ($+8.8$) & $2.09$ ($+7.2$) & $2.13$ ($+7.3$)\\
\midrule
\tableheadrow
MMLU-Pro, 400 queries & \multicolumn{2}{c}{$H=6$} & \multicolumn{2}{c}{$H=8$} & \multicolumn{2}{c}{$H=10$}\\
\midrule
Qwen3-4B & \multicolumn{2}{c}{$.80$ ($-6.6$)} & \multicolumn{2}{c}{$.65$ ($-7.5$)} & \multicolumn{2}{c}{$.67$ ($-5.5$)}\\
Phi-4-mini & \multicolumn{2}{c}{$.94$ ($-1.7$)} & \multicolumn{2}{c}{$.86$ ($-3.5$)} & \multicolumn{2}{c}{$.84$ ($-4.0$)}\\
GLM-4-32B & \multicolumn{2}{c}{$.90$ ($-2.8$)} & \multicolumn{2}{c}{$.95$ ($-1.1$)} & \multicolumn{2}{c}{$.89$ ($-2.8$)}\\
\bottomrule
\end{tabular}
\end{table}

\emph{Pilot risk.} From $H=6$ to $10$, the plain-\tis{}/anchor MSE ratio at 400 queries changes $7.4\to7.6$ for Phi-4-mini, $30.1\to32.0$ for GLM, and $19.0\to14.4$ for Qwen3-4B, so this prediction holds in 2/3 settings. Across all 18 MMLU-Pro cells, the anchor attains $.07$--$.19$ of uniform MSE and is resolved better than learned occupancy and the uniform blend in every cell.

\subsection{Selecting the allocation from the pilot}
\label{sec:supp-pilot-rule}
This section operationalizes \textbf{Q5}: the population tail--mean distance is unavailable, but every learned design already draws a uniform pilot. Before evaluation we declared: for each setting and budget, compute the replication-0 pilot distance between normalized tail and mean influence shares for each question, take the median, use the anchor if it is at least $.21$ (the midpoint of the population thresholds), and otherwise use the occupancy--mean blend. A choice is wrong only if the selected design has resolved higher MSE ($|z|\ge2$) than the alternative.

Across 46 blending settings (148 setting--budget cells), the rule is wrong in 7; selected-design MSE relative to the better design has median $1.00$. Pilot and population medians have Spearman correlation $.86$, and replication-0 decisions agree with replications 1--4 in $91\%$ of cases. Three errors are Phi-4-mini confident-error cells with pilot medians $.197$--$.206$ just below threshold, costing $3$--$6\%$ MSE. Four are FinQA Qwen3-4B $H=6$ cells where near-determinism makes a small pilot overstate distance (pilot $.27$--$.57$, population $.000$), giving the anchor $1.4$--$2.3\times$ mean-blend MSE. Thus the statistic is reliable except when failures are too rare for the pilot to observe, precisely the regime in which Proposition~\ref{prop:rare-failure} removes the need for tail-specific allocation.
\endgroup
\subsection{Mechanism: what the pilot sees and what anchoring changes}
\label{sec:pilot-mechanism}
\begin{figure}[htbp]
\centering
\includegraphics[width=\textwidth]{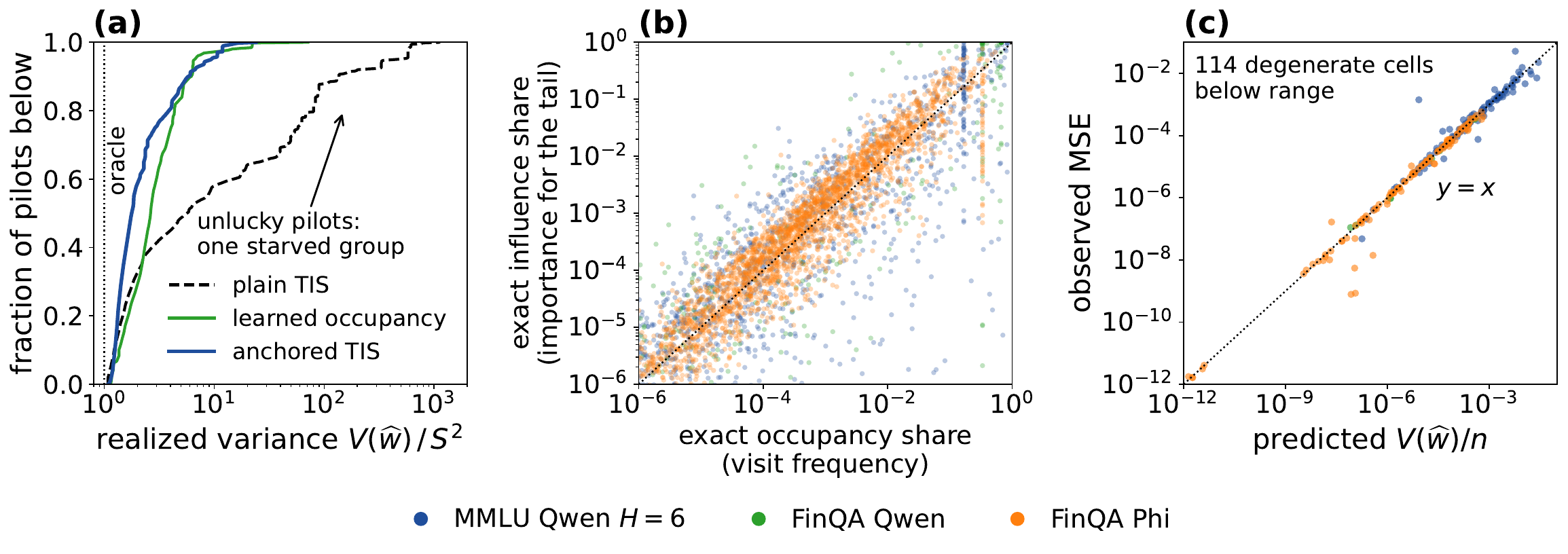}
\caption{\textbf{Anchoring reduces the variance penalty from poor pilots.} (a) Realized/oracle leading variance; (b) occupancy versus tail-influence shares; (c) observed MSE versus the no-fit prediction $V(w)/n$. Rollouts are excluded.}
\label{fig:mechanism}
\end{figure}

Figure~\ref{fig:mechanism} explains \textbf{Q3}'s failure mode at the query-share level. Equation~\ref{eq:design-regret-identity} penalizes underallocation by dividing squared allocation error by assigned weight. We replay the recorded pilots, floors, and rounding and use population scales to diagnose $V(\widehat w)=\sum_g\sigma_g^2/\widehat w_g$ for MMLU Qwen ($H=6$, 400 queries), FinQA Qwen ordinary review (200), and FinQA Phi ordinary review (200). All questions enter; normalization by $S^2=(\sum_g\sigma_g)^2$ excludes their 1, 3, and 13 zero-influence questions. Rollouts are excluded because they estimate under a different sampling functional.

Median correlations between population influence and occupancy shares are $.74,.85,.62$ in the three settings, and the bottom-occupancy half of kernels carries essentially no tail influence (Figure~\ref{fig:mechanism}(b)), unlike Proposition~\ref{prop:controlled-separation}'s equal-visitation construction. For MMLU pilots, median/90th-percentile/maximum $V(\widehat w)/S^2$ is $5.6/136/1{,}137$ for plain \tis{}, $1.8/5.8/24$ for the anchor, and $2.6/6.2/71$ for occupancy; a single starved group can dominate variance (90th-percentile top-contribution share $1.00$). The anchor multiplies the most-starved group's weight by median factors $7.5,20,4.9$ across the three settings (MMLU 90th percentile $99$). On near-deterministic FinQA Qwen, pilots often see point masses, estimated scales vanish, and \tis{} approaches uniform: median $V/S^2$ is $72.4$ with 73 groups versus $3.3$ for the anchor. These are leading-variance diagnostics, not finite-sample MSE guarantees.

Using actual main-sample counts, pilot-averaged $V(\widehat w)/n$ predicts MSE without fitted constants in the examined nondegenerate cells: median $\log_{10}$(observed/predicted) is $-.001$ over 98 MMLU question--method pairs and $-.006$ over 135 Phi pairs, with 10th--90th percentiles within $\pm.19$ (Figure~\ref{fig:mechanism}(c)). For near-deterministic Qwen questions, 300 replications can miss rare deviations carrying much of the variance, so observed MSE can be orders of magnitude lower; the plot flags 114 below-range FinQA pairs (96 Qwen, 18 Phi). Agreement elsewhere is a retrospective check of the variance formula, not a finite-budget guarantee.


\section{Additional Related Work}\label{sec:extended-related}
Table~\ref{tab:closest-comparison} summarizes the closest foundations; the connections below extend Section~\ref{sec:related-work}.

\begin{table}[htbp]
\centering
\small
\caption{Established foundations and the additional results for the conditional-query CVaR problem. The allocation rule is Neyman allocation; the work here identifies and learns the Bellman influence scales it needs.}
\label{tab:closest-comparison}
\begin{tabular}{@{}p{0.34\linewidth}p{0.62\linewidth}@{}}
\toprule
\tableheadrow
\textbf{Established starting point} & \textbf{Additional result here}\\
\midrule
\redrev{Return-law and functional inference under specified sampling \citep{zhang2025inference}; quantile-based efficiency \citep{cheng2026quantile}.}
& An explicit categorical-CVaR influence for each conditional kernel, including its joint reuse across stages. Theorems~\ref{thm:allocation-clt} and~\ref{thm:oracle} give the variance as a function of query shares and its fixed-design efficiency interpretation.\\[4pt]
Neyman allocation and learning unknown stratum variances \citep{neyman1934two,etore2010adaptive,carpentier2015adaptive}.
& The influence function depends on the unknown model and quantile. Theorem~\ref{thm:adaptive} controls learning those quantities, quantile
errors, and Bellman remainders to obtain oracle variance and normalized MSE, including pilot cost and individual zero-influence groups.\\[4pt]
Trajectory collection for mean policy evaluation
\citep{mukherjee2022revar,mukherjee2024saver}.
& Independent conditional queries for a fixed policy's tail functional. Proposition~\ref{prop:controlled-separation} shows why visitation and mean-optimal allocation can miss the relevant signal. Comparisons with
complete rollouts are empirical; the product-model efficiency theorem does not cover their different sampling experiment.\\
\bottomrule
\end{tabular}
\end{table}

The distinction from adjacent work is mainly the design variable. Distributional and quantile-based RL develop inference or efficiency under specified data laws \citep{zhang2025inference,cheng2026quantile,peng2024statistical,peng2026onlineinference}, while generative-access and offline analyses study estimation error under fixed access models \citep{rowland2024nearminimax,peng2025linearctd,chandak2021universal,wu2023distributional,hong2025bellman}. Stratified and adaptive experimental design learn Neyman allocations when the within-stratum target is already defined \citep{carpentier2015adaptive,dai2023neyman}; our score itself depends on an unknown Bellman continuation model and CVaR cutoff, which is why Theorem~\ref{thm:adaptive} must control learning the influence function as well as its allocation. Risk-sensitive control and logging-policy design change the policy or trajectory distribution \citep{bauerle2011avar,zhu2024uncertainty,douglas2026logging}; here the policy and conditional laws stay fixed and only the independent conditional-query counts change. Active testing allocates effort across benchmark items \citep{nguyen2018active,kossen2021active,polo2024tinybenchmarks,li2025activeeval}; our experiments allocate within an item's stochastic workflow, and combining the two levels is a natural extension.

\end{document}